\documentclass[twoside]{article}

\usepackage[accepted]{aistats2025}
\usepackage{microtype}
\usepackage{graphicx}
\usepackage{subfigure}
\usepackage{booktabs}
\usepackage{mypackage}
\usepackage{xcolor}
\usepackage{pythonhighlight}
\usepackage{hyperref}
\usepackage{algorithm}
\usepackage{algorithmic}
\usepackage{natbib}

\usepackage{amsmath}
\usepackage{amssymb}
\usepackage{mathtools}
\usepackage{amsthm}
\usepackage[capitalize,noabbrev]{cleveref}

\theoremstyle{plain}
\newtheorem{theorem}{Theorem}[section]
\newtheorem{proposition}[theorem]{Proposition}
\newtheorem{lemma}[theorem]{Lemma}
\newtheorem{corollary}[theorem]{Corollary}
\theoremstyle{definition}
\newtheorem{definition}[theorem]{Definition}

\theoremstyle{remark}

\usepackage[textsize=tiny]{todonotes}

\definecolor{zaffre}{RGB}{0, 20, 168}

\begin{document}

%

%

\twocolumn[

\aistatstitle{Tighter Confidence Bounds for Sequential Kernel Regression}

\aistatsauthor{Hamish Flynn \And David Reeb}

\aistatsaddress{Universitat Pompeu Fabra\\ Barcelona, Spain \And  Bosch Center for Artificial Intelligence\\ Renningen, Germany} ]

\begin{abstract}
Confidence bounds are an essential tool for rigorously quantifying the uncertainty of predictions. They are a core component in many sequential learning and decision-making algorithms, with tighter confidence bounds giving rise to algorithms with better empirical performance and better performance guarantees. In this work, we use martingale tail inequalities to establish new confidence bounds for sequential kernel regression. Our confidence bounds can be computed by solving a conic program, although this bare version quickly becomes impractical, because the number of variables grows with the sample size. However, we show that the dual of this conic program allows us to efficiently compute tight confidence bounds. We prove that our new confidence bounds are always tighter than existing ones in this setting. We apply our confidence bounds to kernel bandit problems, and we find that when our confidence bounds replace existing ones, the KernelUCB (GP-UCB) algorithm has better empirical performance, a matching worst-case performance guarantee and comparable computational cost.
\end{abstract}

\section{INTRODUCTION}

A confidence sequence for an unknown function $f^{\star}$ is a sequence of sets $\mathcal{F}_1, \mathcal{F}_2, \dots$, where each set $\mathcal{F}_t$ contains all functions that could plausibly be $f^{\star}$, given all data available at time step $t$. For any point $x$, the corresponding upper confidence bound, $\max_{f \in \mathcal{F}_t}\{f(x)\}$, is thus the largest value that $f^{\star}(x)$ could plausibly take, given all data available at time step $t$.

The uncertainty estimates provided by confidence bounds are exceptionally useful in sequential learning and decision-making problems, such as bandits \citep{lattimore2020bandit} and reinforcement learning \citep{sutton2018reinforcement}, where they can be used to design exploration strategies \citep{auer2002using}, establish stopping criteria \citep{jamieson2014lil} and guarantee safety \citep{sui2015safe, berkenkamp2017safe}. In general, tighter confidence bounds give rise to sequential learning and decision-making algorithms with better empirical performance and better performance guarantees. As many modern algorithms use uncertainty estimates for flexible nonparametric function approximators, such as kernel methods, it is vital to establish tight confidence bounds for infinite-dimensional function classes.

In this work, we develop a new confidence sequence for sequential kernel regression. The corresponding upper confidence bound at any point $x$ is the solution of an infinite-dimensional convex program in the kernel Hilbert space. This can be reformulated using the representer theorem, yielding a finite-dimensional convex program, and our Kernel CMM-UCB method. However, computing the solution becomes expensive, because the number of variables grows with the sample size. One solution is to use an analytic upper bound on the exact confidence bound, as we do in our Kernel AMM-UCB method. Unfortunately, this analytic confidence bound can be noticeably looser. A better solution is to compute confidence bounds by solving the Lagrangian dual of the original convex program, which can be done cheaply. This results in our Kernel DMM-UCB method, which achieves the best of both worlds: cheap and tight confidence bounds (see Fig.\ \ref{fig:ucbs}).

\begin{figure*}
\centering
\includegraphics[width=0.91\textwidth]{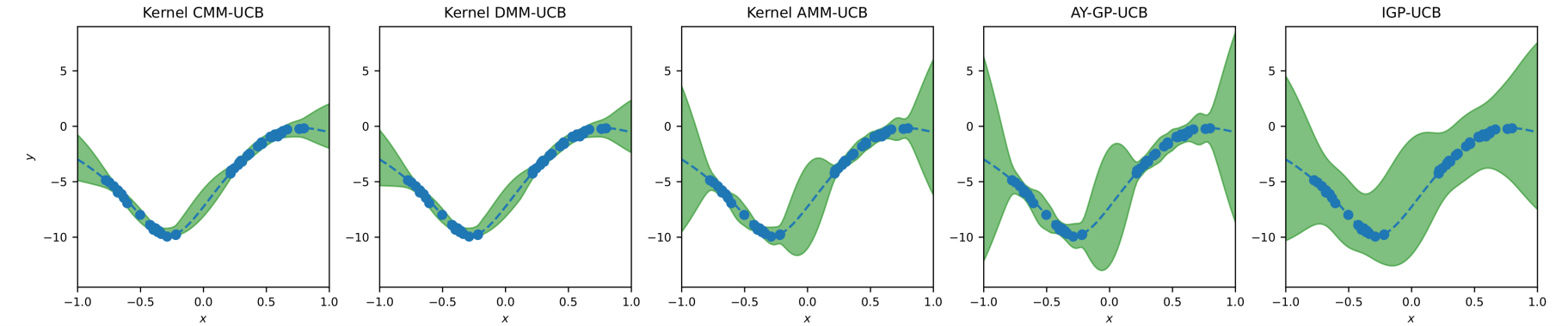}
\vspace{-2mm}
\caption{The upper and lower confidence bounds used by our Kernel CMM-UCB method (left), the confidence bounds from \citep{abbasi2012online} (AY-GP-UCB) (middle-right), and the confidence bounds used by Improved GP-UCB (IGP-UCB) \citep{chowdhury2017kernelized} (right) for a Mat\'{e}rn kernel test function with smoothness $\nu = 3/2$ and lengthscale $\ell = 0.5$. Our Kernel CMM-UCB method -- and its relaxed versions DMM-UCB and AMM-UCB -- produces confidence bounds that are visibly closer to the ground-truth function (dashed line) than those of AY-GP-UCB and IGP-UCB (cf.\ also Sec.\ \ref{sec:tighter_conf}).}
\label{fig:ucbs}
\end{figure*}

Our new confidence bounds can be used in kernel bandit algorithms, such as KernelUCB/GP-UCB. We show that using our confidence bounds preserves the worst-case performance guarantees of KernelUCB. In doing so, we refine the standard KernelUCB regret analysis by using new bounds on the number of large elliptical potentials. With these new tools, we show that in the noiseless setting, KernelUCB with our confidence bounds enjoys a $\mathrm{polylog}(T)$ cumulative regret bound when the RBF kernel is used, and a sub-$\sqrt{T}$ regret bound when the Mat\'{e}rn kernel with smoothness $\nu \geq d/\sqrt{2}$ is used. Finally, we demonstrate that the empirical performance of KernelUCB can be vastly improved by using our new confidence bounds.

\section{PROBLEM STATEMENT}
\label{sec:problem}

We consider a sequential kernel regression problem with covariates $x_t \in \mathcal{X}$ and responses $y_t \in \mathbb{R}$, where
\begin{equation*}
y_t = f^{\star}(x_t) + \epsilon_t.
\end{equation*}
$f^{\star} \in \mathcal{H}$ is an unknown function in a reproducing kernel Hilbert space (RKHS) $\mathcal{H}$ and $\epsilon_1, \epsilon_2, \dots$ are noise variables. We allow the covariates to be generated in an arbitrary sequential fashion, meaning each $x_t$ can depend on $x_1, y_1, \dots, x_{t-1}, y_{t-1}$. We assume that the RKHS norm of $f^{\star}$ is bounded by $B$, i.e.\ $\norm{f^{\star}}_{\mathcal{H}} \leq B$. For most of the paper, we assume that, conditioned on $x_1, y_1, \dots, x_{t-1}, y_{t-1}, x_t$, each $\epsilon_t$ is a $\sigma$-sub-Gaussian random variable. In Sec. \ref{sec:noiseless}, we consider the noiseless setting, where $\epsilon_t \equiv 0$, and we observe $y_t = f^{\star}(x_t)$.

Our first aim is to construct a \emph{confidence sequence} for the function $f^{\star}$, which we now define formally. For any level $\delta \in (0, 1]$, a $(1 - \delta)$-confidence sequence is a sequence $\mathcal{F}_1, \mathcal{F}_2, \dots$ of subsets of $\mathcal{H}$, such that each $\mathcal{F}_t$ can be calculated using the data $x_1, y_1, \dots, x_t, y_t$ and the sequence satisfies the coverage condition
\begin{equation*}
\mathbb{P}_{\substack{x_1, x_2, \dots \\ y_1, y_2, \dots}}\left[\,\forall t \geq 1:\, f^{\star} \in \mathcal{F}_t\,\right]\, \geq \,1 - \delta.
\end{equation*}
In words, this coverage condition says that with high probability (over the random draw of $x_1, y_1, x_2, y_2, \dots$), $f^{\star}$ lies in $\mathcal{F}_t$ for all times $t \geq 1$ simultaneously. We call each set in a confidence sequence a confidence set. Our second aim is to construct \emph{confidence bounds} for the value of $f^{\star}(x)$ at every $x \in \mathcal{X}$. For any level $\delta \in (0, 1]$, an anytime-valid $(1 - \delta)$-upper confidence bound (UCB) is a sequence $u_1, u_2, \dots$ of functions on $\mathcal{X}$, such that each $u_t$ can be calculated using the data $x_1, y_1, \dots, x_t, y_t$, and the sequence satisfies
\begin{equation*}
\mathbb{P}_{\substack{x_1, x_2, \dots \\ y_1, y_2, \dots}}\left[\,\forall t \geq 1, \forall x \in \mathcal{X}: f^{\star}(x) \leq u_t(x)\,\right]\, \geq \,1 - \delta.
\end{equation*}
This condition says that with high probability, $u_t(x)$ is an upper bound on $f^{\star}(x)$ for all $t \geq 1$ and all $x \in \mathcal{X}$ simultaneously. The definition of an anytime-valid lower confidence bound (LCB) is analogous. If we already have a confidence sequence, we can use it to construct upper and lower confidence bounds. In particular, if $\mathcal{F}_1, \mathcal{F}_2, \dots$ is a $(1-\delta)$-confidence sequence, then with probability at least $1 - \delta$, for all $t \geq 1$ and all $x \in \mathcal{X}$, we have\vspace{-1mm}
\begin{equation}\label{eqn:confidence_bounds}
\min_{f \in \mathcal{F}_t}\left\{f(x)\right\}\, \leq\, f^{\star}(x) \,\leq\, \max_{f \in \mathcal{F}_t}\left\{f(x)\right\}.
\end{equation}
Therefore, as long as we can solve the maximisation/minimisation problems in (\ref{eqn:confidence_bounds}), we can compute anytime-valid confidence bounds from a confidence sequence. We use $\mathrm{UCB}_{\mathcal{F}_t}(x)$ to denote the RHS of (\ref{eqn:confidence_bounds}), i.e.\ $\mathrm{UCB}_{\mathcal{F}_t}(x) := \max_{f \in \mathcal{F}_t}\left\{f(x)\right\}$, and $\mathrm{LCB}_{\mathcal{F}_t}(x)$ to denote the LHS of (\ref{eqn:confidence_bounds}).

\paragraph{Notation.} For any $n \in \mathbb{N}$, $[n] := \{1, \dots, n\}$. We use matrix notation to denote the inner product $f_1^{\top}f_2 := \inner{f_1}{f_2}_{\mathcal{H}}$ and outer product $f_1f_2^{\top} := f_1\inner{f_2}{\cdot}_{\mathcal{H}}$ between two functions $f_1, f_2 \in \mathcal{H}$. We denote the kernel function by $k(\cdot,\cdot)$. For a sequence of observations $x_1, \dots, x_t$, we define $\bs{k}_t(x) := [k(x, x_1), \dots, k(x, x_t)]^{\top}$ and $\Phi_t := [k(x_1, \cdot), \dots, k(x_t, \cdot)]^{\top}$. We let $\bs{K}_t := \{k(x_i, x_j)\}_{i \in [t], j \in [t]} = \Phi_t\Phi_t^{\top}$ denote the kernel matrix and $\bs{V}_t := \sum_{s=1}^{t}k(\cdot, x_s)k(\cdot, x_s)^{\top} = \Phi_t^{\top}\Phi_t$ the empirical covariance operator. For $\alpha \geq 0$, we define
\begin{align}
\mu_{\alpha, t}(x) \,&:=\, \bs{k}_t(x)^{\top}(\bs{K}_t + \alpha\bs{I})^{-1}\bs{y}_t,\label{eqn:kernel_ridge}\\
\rho_{\alpha, t}^2(x) \,&:=\, k(x,x) - \bs{k}_t(x)^{\top}(\bs{K}_t + \alpha\bs{I})^{-1}\bs{k}_t(x).\label{eqn:kernel_ridge_var}
\end{align}

$\mu_{\alpha, t}(x)$ is the kernel ridge estimate (evaluated at $x$) or, equivalently, the predictive mean of a standard Gaussian process (GP) posterior with noise $\alpha$. $\rho_{\alpha, t}^2(x)$ is the predictive variance of a standard GP posterior.

\section{RELATED WORK}
Several confidence sequences/bounds have been proposed for the sequential kernel regression problem that we consider. The confidence bounds in \citep{srinivas2010gaussian} are perhaps the most well-known. The two most widely used were developed by \citet{abbasi2011improved, abbasi2012online} and \citet{chowdhury2017kernelized}, using concentration inequalities for self-normalised processes \citep{pena2004self, pena2009self}. Subsequently, \citet{durand2018streaming} derived empirical Bernstein-type versions of these confidence bounds, and \citet{whitehouse2023sublinear} showed that for certain kernels, these confidence bounds can be tightened by regularising in proportion to the smoothness of the kernel. Confidence sequences for this setting have also been derived using online-to-confidence-set conversions \citep{abbasi2012onlinetoconc, abbasi2012online}.

With applications Bayesian optimisation and kernel bandits in mind, \citet{neiswanger2021uncertainty} and \citet{emmenegger2023likelihood} developed confidence sequences using martingale tail bounds for sequential likelihood ratios. However, cumulative regret bounds are not established in \citep{neiswanger2021uncertainty}, and the regret bound in \citep{emmenegger2023likelihood} is worse than the regret bounds in \citep{abbasi2012online} and in this paper.

In the setting where $\epsilon_1, \epsilon_2, \dots$ are bounded deterministic errors, \citet{scharnhorst2022robust} derived confidence bounds using convex programming and Lagrangian duality. However, these confidence bounds are not comparable to ours. In our setting, the random errors $\epsilon_1, \epsilon_2, \dots$ average out to 0, and our confidence bounds contract around $f^{\star}$. This is not the case for the confidence bounds in \citep{scharnhorst2022robust}.

The work most closely related to ours is \citep{flynn2023improved}, in which an improved confidence sequence for linear bandits is developed. While our Eq. (\ref{eqn:radius_def}) is a simple kernel generalisation of their Eq.\ (5), deriving confidence bounds is more involved in our kernel setting. In particular, $\max_{f \in \mathcal{F}_t}\{f(x)\}$ is now an infinite-dimensional optimisation problem, which makes our DMM-UCB method vital for balancing tightness and efficiency. Moreover, even in the simpler linear setting, the regret analysis in \citep{flynn2023improved} fails to provide a sub-$\sqrt{T}$ rate as $\sigma$ goes to zero.


\section{CONFIDENCE BOUNDS FOR KERNEL REGRESSION}
\label{sec:confidence_bounds}

In this section, we describe our new confidence sequences and bounds for sequential kernel regression.

\subsection{Martingale Mixture Tail Bounds}\label{sec:mm_tail_bounds}

We use a martingale tail bound from \citep{flynn2023improved} to derive a data-dependent constraint for $f^{\star}$. We begin by describing the general setting in which it holds. There is a stochastic process $(Z_t(g)| t \in \mathbb{N}, g \in \mathbb{R})$, indexed by a time $t$ and a real number $g$, which is adapted to a filtration $(\mathcal{D}_t|t \in \mathbb{N})$. In this paper, $(\mathcal{D}_t|t \in \mathbb{N})$ is any filtration such that $y_t$ and $x_{t+1}$ are both $\mathcal{D}_t$-measurable, e.g.\ $\mathcal{D}_t = \sigma(x_1, y_1, \dots, x_{t}, y_{t}, x_{t+1})$. In addition, there is a sequence of predictable ``guesses'' $(g_t\in\mathbb{R}|t \in \mathbb{N})$, and a sequence of predictable random variables $(\lambda_t\in\mathbb{R}|t \in \mathbb{N})$. We use the shorthand $\bs{g}_t = [g_1, \dots, g_t]^{\top}$ and $\bs{\lambda}_t = [\lambda_1, \dots, \lambda_t]^{\top}$. We define\vspace{-1mm}
\begin{align}
M_t(\bs{g}_t, \bs{\lambda}_t) &:= \mathrm{exp}\left(\sum_{s=1}^{t}\lambda_s Z_s(g_s) - \psi_s(g_s, \lambda_s)\right),
\label{eqn:def_Mt}\\
\psi_t(g_t, \lambda_t) &:= \ln\left(\mathbb{E}\left[\mathrm{exp}(\lambda_t Z_t(g_t)) | \mathcal{D}_{t-1}\right]\right),\label{eqn:def_psi_t}
\end{align}

\citet{flynn2023improved} show that $M_t(\bs{g}_t, \bs{\lambda}_t)$ is a non-negative martingale, which means Ville's inequality \citep{ville1939etude} can be applied, giving $M_t(\bs{g}_t, \bs{\lambda}_t) \leq 1/\delta$ with high probability. To obtain the tightest possible inequality, we employ the method of mixtures. That is, we place a distribution $P_t$ on each vector of guesses $\bs{g}_t$. \citet{flynn2023improved} show that if the sequence of distributions $(P_t| t \in \mathbb{N})$ satisfies: (a) $P_t$ is $\mathcal{D}_{t-1}$-measurable; (b) $\int P_t(\bs{g}_t)dg_t=P_{t-1}(\bs{g}_{t-1})$ for all $t$ and $\bs{g}_{t-1}$, then the mixture $\mathbb{E}_{\bs{g}_t \sim P_t}\left[M_t(\bs{g}_t, \bs{\lambda}_t)\right]$ is also a non-negative martingale, and Ville's inequality can still be used.


\begin{theorem}[Theorem 5.1 of \citep{flynn2023improved}]\label{thm:adaptive_tail_bound}
For any $\delta \in (0, 1)$, any sequence of distributions $(P_t| t \in \mathbb{N})$ satisfying (a) and (b), and any sequence of predictable random variables $(\lambda_t| t \in \mathbb{N})$, it holds with probability at least $1 - \delta$ that
\begin{equation}
\forall t \geq 1: \quad \mathrm{ln}\left(\mathop{\mathbb{E}}_{\bs{g}_t \sim P_t}\left[M_t(\bs{g}_t, \bs{\lambda}_t)\right]\right) \,\leq\, \mathrm{ln}(1/\delta).\label{eqn:adaptive_tail_bound}
\end{equation}
\end{theorem}
To turn the inequality in Thm. \ref{thm:adaptive_tail_bound} into a constraint for $f^{\star}$, all that remains is to specify the process $Z_t(g)$, the random variables $\lambda_t$ and the distributions $P_t$. Intuitively, each $g_t$ is a guess for the value $y_t$, given the previous data $x_1, y_1, \dots, x_{t-1}, y_{t-1}, x_t$. We should therefore choose each $P_t$ such that (a) and (b) are satisfied and $P_t$ assigns high probability to the vector $[y_1, \dots, y_t]^{\top}$. The choices of $Z_t(g)$ and $\lambda_t$ determine the shape of the resulting confidence sequence. We choose $Z_t(g) = (g - f^{\star}(x_t))\epsilon_t$, which leads to a convex quadratic constraint for $f^{\star}$. As $Z_t(g)$ is linear in the noise variable $\epsilon_t$, $\psi_t(g, \lambda_t)$ can be upper bounded using the $\sigma$-sub-Gaussian property of $\epsilon_t$ (cf. Eq. (\ref{eqn:subgaussian})).
\begin{align*}
\psi_t(g_t, \lambda_t) \,\leq\, \tfrac{1}{2}\sigma^2\lambda_t^2(g_t-f^{\star}(x_t))^2.
\end{align*}

Using this upper bound, choosing $\lambda_t \equiv 1/\sigma^2$, and rearranging Eq.\ (\ref{eqn:adaptive_tail_bound}) (see App. \ref{app:pre_radius_derivation}), we obtain an upper bound on the squared Euclidean distance between the vector of function values $\bs{f}_{t}^{\star} := [f^{\star}(x_1), \dots, f^{\star}(x_t)]^{\top}$ and the response vector $\bs{y}_{t} := [y_1, \dots, y_t]^{\top}$.
\begin{equation}\label{eqn:pre_radius_def}
\begin{split}
&\Vert\bs{f}_{t}^{*}-\bs{y}_{t}\Vert_2^2 ~\leq ~ 2\sigma^2\ln(1/\delta)\\
&\quad- 2\sigma^2\ln\left(\mathop{\mathbb{E}}_{\bs{g}_{t} \sim P_t}\left[\mathrm{exp}\left(-\tfrac{1}{2\sigma^2}\Vert\bs{g}_{t}-\bs{y}_{t}\Vert_2^2\right)\right]\right).
\end{split}
\end{equation}
For the mixture distributions, we choose $P_t = \mathcal{N}(\bs{0}, c\bs{K}_t)$, resembling a zero-mean Gaussian process with covariance scaled by $c > 0$. There are several reasons for this choice of $P_t$. First, for an appropriate covariance scale $c > 0$, the confidence bounds obtained from (\ref{eqn:pre_radius_def}) with $P_t = \mathcal{N}(\bs{0}, c\bs{K}_t)$ are always tighter than those in \citep{abbasi2012online} and \citep{chowdhury2017kernelized} (see Sec. \ref{sec:tighter_conf}). Second, any Gaussian $P_t$ yields a convenient closed-form expression for the expected value in (\ref{eqn:pre_radius_def}) (see App. \ref{app:radius_derivation}). Substituting this expression into (\ref{eqn:pre_radius_def}), we obtain
\begin{align}
\Vert\bs{f}_{t}^{\star}-\bs{y}_{t}\Vert_2^2 ~ \leq &  ~\,\bs{y}_{t}^{\top}\left(\bs{I}+\tfrac{c}{\sigma^2}\bs{K}_{t}\right)^{-1}\bs{y}_{t}\label{eqn:radius_def}\\
&+ \sigma^2\ln\det\left(\bs{I}+\tfrac{c}{\sigma^2}\bs{K}_{t}\right)+2\sigma^2\ln\frac{1}{\delta}\nonumber\\
&=: R_{t}^2.\nonumber
\end{align}

This inequality states that with probability at least $1-\delta$, at every step $t$, the vector of ground-truth function values $\bs{f}_{t}^{\star}$ lies within a sphere of radius $R_{t}$ around the observed response vector $\bs{y}_t$.

\subsection{Confidence Sequences}
We now describe how the inequality in (\ref{eqn:radius_def}) can be used to construct confidence sequences for $f^{\star}$. Let $\bs{f}_{t} = [f(x_1), \dots, f(x_t)]^{\top}$ denote the values of an arbitrary function $f \in \mathcal{H}$ at the points $x_1, \dots, x_t$. Since, with probability at least $1 - \delta$, (\ref{eqn:radius_def}) holds for all $t \geq 1$, the sets of functions that (at each $t$) satisfy both (\ref{eqn:radius_def}) and $\norm{f}_{\mathcal{H}} \leq B$ form a $1 - \delta$ confidence sequence for $f^{\star}$.

\begin{lemma}\label{lem:rkhs_conf_set}
For any $c > 0$ and $\delta \in (0, 1]$, it holds with probability at least $1 - \delta$ that for all $t \geq 1$ simultaneously, the function $f^{\star}$ lies in the set
\begin{align*}
\mathcal{F}_t = \{f \in \mathcal{H}\,:\, &\Vert\bs{f}_{t} - \bs{y}_{t}\Vert_2 \leq R_{t}, ~\norm{f}_{\mathcal{H}} \leq B\}.
\end{align*}
\end{lemma}

The two constraints that define $\mathcal{F}_t$ are both ellipsoids in $\mathcal{H}$, meaning that $\mathcal{F}_t$ is the intersection of two ellipsoids. In contrast, most existing confidence sets for functions in RKHS's (e.g.\ \citep{abbasi2012online, chowdhury2017kernelized}) are single ellipsoids centred at the kernel ridge estimate $\mu_{\alpha, t}$ (see Eq.\ (\ref{eqn:kernel_ridge})), for some regularisation parameter $\alpha > 0$. It turns out that by taking a weighted sum of the constraints in $\mathcal{F}_t$, we can obtain a single-ellipsoid confidence set $\widetilde{\mathcal{F}}_t \supseteq \mathcal{F}_t$, which is also centred at the kernel ridge estimate.

\begin{corollary}\label{cor:rkhs_ridge_conf_set}
For any $c > 0$, any $\alpha > 0$ and any $\delta \in (0, 1]$, it holds with probability at least $1 - \delta$ that for all $t \geq 1$ simultaneously, $f^{\star}$ lies in the set
\begin{align*}
&\widetilde{\mathcal{F}}_t = \{f \in \mathcal{H}\,: \, \|(\bs{V}_t + \alpha\bs{I}_{\mathcal{H}})^{1/2}(f - \mu_{\alpha, t})\|_{\mathcal{H}} \leq \widetilde{R}_{\alpha, t}\},
\end{align*}
where
\begin{align*}
\widetilde{R}_{\alpha, t}^2 := R_{t}^2 + \alpha B^2 - \bs{y}_t^{\top}\bs{y}_t + \bs{y}_t^{\top}\Phi_t(\bs{V}_t + \alpha\bs{I}_{\mathcal{H}})^{-1}\Phi_t^{\top}\bs{y}_t.
\end{align*}
\end{corollary}
The proof (in App. \ref{app:ridge_set}) involves rearranging
\begin{equation*}
\Vert\bs{f}_{t} - \bs{y}_{t}\Vert_2^2 + \alpha\norm{f}_{\mathcal{H}}^2 ~\leq~ R_{t}^2 + \alpha B^2,
\end{equation*}
into the constraint that defines $\widetilde{\mathcal{F}}_t$. This confidence sequence can be more easily compared to those in \citep{abbasi2012online, chowdhury2017kernelized}.


\subsection{Implicit Confidence Bounds}\label{sec:implicit_conf}
We now turn our attention to the confidence bounds that can be obtained from our confidence sequences, and how we can compute them. We start by considering the exact upper confidence bound $\mathrm{UCB}_{\mathcal{F}_t}(x) = \max_{f \in \mathcal{F}_t}\left\{f(x)\right\}$, which can be written as
\begin{equation}\label{eqn:ucb_opt}
\max_{f \in \mathcal{H}}\; f(x) ~~\text{s.t.} ~~\Vert\bs{f}_t - \bs{y}_t\Vert_2 \leq R_{t}, ~~\norm{f}_{\mathcal{H}} \leq B.
\end{equation}
The optimisation variable in (\ref{eqn:ucb_opt}) is a function in the possibly infinite-dimensional RKHS $\mathcal{H}$. However, as (\ref{eqn:ucb_opt}) contains a constraint on the RKHS norm of $f$, the (generalised) representer theorem \citep{kimeldorf1971some, scholkopf2001generalized} can be applied. This means that the optimising function $f$ in  (\ref{eqn:ucb_opt}) can be expressed as a finite linear combination of the form $\bs{k}_{t+1}(x)^{\top}\bs{w}$, where $\bs{k}_{t+1}(x) := [k(x, x_1), \dots, k(x, x_t), k(x, x)]^{\top}$. When we substitute the functional form $f(\cdot) = \bs{k}_{t+1}(\cdot)^{\top}\bs{w}$ into the objective and both constraints, we obtain the finite-dimensional convex program in Thm. \ref{thm:ucb_representer}. (see App. \ref{app:representer_proof}).

\begin{theorem}
$\mathrm{UCB}_{\mathcal{F}_t}(x)$ equals the solution of the convex program
\begin{align}
\max_{\bs{w} \in \mathbb{R}^{t+1}} &~\bs{k}_{t+1}(x)^{\top}\bs{w}\label{eqn:ucb_cone_prog}\\
\text{s.t.} \quad &\Vert\bs{K}_{t, t+1}\bs{w} - \bs{y}_t\Vert_2 \leq R_{t},\nonumber\\
&\Vert\bs{L}_{t+1}\bs{w}\Vert_2 \leq B\nonumber.
\end{align}
where $\bs{K}_{t, t+1}:=(\bs{K}_t,\bs{k}_{t}(x))\in\mathbb{R}^{t\times(t+1)}$ denotes kernel matrix with added last column $\bs{k}_{t}(x)$, 
and $\bs{L}_{t+1}$ is any matrix satisfying $\bs{L}_{t+1}^{\top}\bs{L}_{t+1} = \bs{K}_{t+1}$ (e.g.\ the right Cholesky factor of $\bs{K}_{t+1}$).
\label{thm:ucb_representer}
\end{theorem}


Eq.\ (\ref{eqn:ucb_cone_prog}) is a $(t+1)$-dimensional second-order cone program. Using interior point methods, the solution can be computed to $\epsilon$ accuracy in $\mathcal{O}(t^3\mathrm{polylog}(1/\epsilon))$ steps (see Sec. 10.2 of \citealp{nemirovski2004interior}). To obtain cheaper confidence bounds, we consider the Lagrangian dual problem associated with (\ref{eqn:ucb_opt}). We derive the dual in App.\ \ref{app:dual_ucb_proof} and give a short sketch here. The primal in (\ref{eqn:ucb_opt}) can be written as $\max_{f \in \mathcal{H}}\min_{\eta_1, \eta_2 > 0}L(f, \eta_1, \eta_2)$, where
\begin{equation*}
L(f, \eta_1, \eta_2) = f(x) + \eta_1(R_t^2 - \|\bs{f}_t - \bs{y}_t\|_2^2) + \eta_2(B^2 - \|f\|_{\mathcal{H}}^2),
\end{equation*}

is the Lagrangian and $\eta_1, \eta_2$ are the Lagrange multipliers. The dual problem $\min_{\eta_1, \eta_2 > 0}\max_{f \in \mathcal{H}}L(f, \eta_1, \eta_2)$ is obtained by swapping the $\min$ and the $\max$. Since the Lagrangian is concave and quadratic in $f$, the inner maximisation in the dual can be solved analytically, which gives the following expression for the dual.

\begin{theorem}
The dual of (\ref{eqn:ucb_opt}) is
\begin{equation}
\min_{\substack{\eta_1 > 0\\ \eta_2 > 0}}\left\{\mu_{\eta_2/\eta_1, t}(x) - \frac{\rho_{\eta_2/\eta_1, t}^2(x)}{4\eta_2} + \eta_1\widetilde{R}_{\eta_2/\eta_1, t}^2\right\}.\label{eqn:dual_ucb}
\end{equation}
\label{thm:dual_ucb}
\end{theorem}\vspace{-3mm}
See (\ref{eqn:kernel_ridge}) and (\ref{eqn:kernel_ridge_var}) for the definitions of $\mu_{\eta_2/\eta_1, t}(x)$ and $\rho_{\eta_2/\eta_1, t}^2(x)$. By weak duality, the solution of (\ref{eqn:dual_ucb}) is always an upper bound on the solution of the primal problem in (\ref{eqn:ucb_opt}). Whenever there exists an $f$ that lies in the interior of the ellipsoids that define $\mathcal{F}_t$, as is usually the case, then strong duality holds and the solutions of (\ref{eqn:ucb_opt}) and (\ref{eqn:dual_ucb}) are actually equal. In App. \ref{app:dual_ucb_proof}, we show that (\ref{eqn:dual_ucb}) can be further reduced to a one-dimensional optimisation problem. Using the substitution $\eta_2 = \alpha\eta_1$ (for $\alpha > 0$), the solution of (\ref{eqn:dual_ucb}) is equal to
\begin{equation}\label{eqn:dual_ucb_alpha}
\min_{\alpha > 0}\left\{\mu_{\alpha, t}(x) + \tfrac{1}{\sqrt{\alpha}}\widetilde{R}_{\alpha,t}\rho_{\alpha, t}(x)\right\}.
\end{equation}
The objective function in (\ref{eqn:dual_ucb_alpha}) is not always convex in $\alpha$, though empirical evidence suggests that it is quasiconvex in $\alpha$, and could therefore be solved using the bisection method.

\subsection{Explicit Confidence Bounds}
\label{sec:explicit_confidence_bounds}

In this section, we focus on explicit confidence bounds. These are mainly useful for deriving worst-case regret bounds with explicit dependence on $T$ (see Sec. \ref{sec:regret}). From Eq. (\ref{eqn:dual_ucb_alpha}), for any fixed $\alpha > 0$, we have
\begin{equation}\label{eqn:dual_ucb_bound}
\mathrm{UCB}_{\mathcal{F}_t}(x) ~\leq~ \mu_{\alpha, t}(x) + \tfrac{1}{\sqrt{\alpha}}\widetilde{R}_{\alpha,t}\rho_{\alpha, t}(x).
\end{equation}
It turns out that this upper bound on $\mathrm{UCB}_{\mathcal{F}_t}(x)$ is equal to $\mathrm{UCB}_{\widetilde{\mathcal{F}}_t}(x)$, where $\widetilde{\mathcal{F}}_t$ is the set from Lem. \ref{cor:rkhs_ridge_conf_set}.
\begin{corollary}\label{cor:analytic_ucb}
For all $t \geq 1$ and $x \in \mathcal{X}$,
\begin{equation}\label{eqn:analytic_ucb}
\mathrm{UCB}_{\widetilde{\mathcal{F}}_t}(x) ~\leq~ \mu_{\alpha, t}(x) + \tfrac{1}{\sqrt{\alpha}}\widetilde{R}_{\alpha,t}\rho_{\alpha, t}(x).
\end{equation}
\end{corollary}\vspace{-1mm}
This follows from the proof of Thm.\ \ref{thm:dual_ucb}. In particular, one can show that the RHS of (\ref{eqn:analytic_ucb}) is the solution of the dual associated with $\max_{f \in \widetilde{\mathcal{F}}_t}\{f(x)\}$. Again, when $\widetilde{\mathcal{F}}_t$ has an interior point (as is usually the case), the inequality in (\ref{eqn:analytic_ucb}) is actually an equality.


\subsection{Confidence Bound Maximisation}

We now discuss how our upper confidence bounds can be maximised w.r.t.\ $x \in \mathcal{X}$, which is an important consideration for kernel bandits (see Sec.\ \ref{sec:kernelbandits}). If $\mathcal{X}$ has finite (and not too large) cardinality, then $\mathrm{UCB}_{\mathcal{F}_t}(x)$ (and upper bounds on it) can be maximised exactly.

If $\mathcal{X}$ is a continuous subset of $\mathbb{R}^d$, maximising $\mathrm{UCB}_{\mathcal{F}_t}(x)$ over $\mathcal{X}$ exactly is intractable in general, though other similar confidence bounds have the same limitation. There are at least two options for approximate maximisation. One option is to use gradient-based methods, as in \citep{flynn2023improved}. Gradients (w.r.t.\ $x$) of the confidence bounds in (\ref{eqn:ucb_cone_prog}), (\ref{eqn:dual_ucb}) and (\ref{eqn:dual_ucb_alpha}) can be computed using differentiable convex optimisation methods \citep{agrawal2019differentiable} or automatic implicit differentiation methods \citep{look2020differentiable, blondel2022efficient}. Gradient-based maximisation works well in practice, though there are no rigorous approximation guarantees as $\mathrm{UCB}_{\mathcal{F}_t}(x)$ is (in general) a non-concave function of $x$. A second option is to discretise $\mathcal{X}$ and optimise over the discretisation (see e.g.\ Sec.\ 3.1 of \citep{li2022gaussian}). If the kernel function $k(x, x^{\prime})$ is Lipschitz (as the RBF and Mat\'{e}rn kernels are), then the discretisation error can be controlled. However, the size of a sufficiently fine discretisation grows exponentially in the dimension $d$, so this quickly becomes impractical.

\section{KERNEL BANDITS}\label{sec:kernelbandits}

To demonstrate the utility of our confidence bounds, we apply them to the stochastic kernel bandit problem.

\subsection{Stochastic Kernel Bandits}

A learner plays a game over a sequence of $T$ rounds. In each round $t$, the learner observes an action set $\mathcal{X}_t \subseteq \mathcal{X}$ and must choose an action $x_t \in \mathcal{X}_t$. The learner then receives a reward $y_t = f^{\star}(x_t) + \epsilon_t$. The reward function $f^{\star}: \mathcal{X} \to \mathbb{R}$ is an unknown function in an RKHS. As before, we assume that $f^{\star}$ has bounded norm, i.e. $\norm{f^{\star}}_{\mathcal{H}} \leq B$, and that (in the noisy setting) $\epsilon_1, \epsilon_2, \dots, \epsilon_T$ are conditionally $\sigma$-sub-Gaussian.

The goal of the learner is to choose a sequence of actions that minimises the cumulative regret, which is the difference between the total expected reward of the learner and the optimal strategy. The regret in round $t$ is $r_t = f^{\star}(x_t^{\star}) - f^{\star}(x_t)$, where $x_t^{\star} = \argmax_{x \in \mathcal{X}_t}\{f^{\star}(x)\}$. After $T$ rounds, the cumulative regret is $r_{1:T} = \sum_{t=1}^{T}r_t$.

\subsection{Kernel Bandit Algorithms}

In Algorithm \ref{alg:kernel_ucb}, we present a recipe for a kernel bandit algorithm that uses confidence bounds to select actions.
\begin{algorithm}[H]
\caption{KernelUCB}
\label{alg:kernel_ucb}
\begin{algorithmic}
\STATE {\bfseries Input:} Upper confidence bounds $u_0, u_1, u_2, \dots$
\FOR{$t=1$ {\bfseries to} $T$}
\STATE Play $x_{t} = \argmax_{x \in \mathcal{X}_{t}}\{u_{t-1}(x)\}$
\STATE Observe $y_t = f^{\star}(x_t) + \epsilon_t$
\ENDFOR
\end{algorithmic}
\end{algorithm}
In each round $t$, we use all previous observations $\{(x_s, y_s)\}_{s=1}^{t-1}$ to construct the UCB $u_{t-1}$, and then select the action that maximises $u_{t-1}$. Kernel bandit algorithms of this form have appeared under different names, such as GP-UCB \citep{srinivas2010gaussian} and KernelUCB \citep{valko2013finite}, based on the confidence bounds that they use. We therefore use the names KernelUCB and GP-UCB synonymously. When we use the confidence bounds in Thm.\ 3.11 of \citep{abbasi2012online}, we call Algorithm \ref{alg:kernel_ucb} AY-GP-UCB. If we use the confidence bounds in Thm. 2 of \citep{chowdhury2017kernelized}, then Algorithm \ref{alg:kernel_ucb} is the Improved GP-UCB algorithm (IGP-UCB).

\subsection{Kernel CMM-UCB}
\label{sec:kernel_cmm_ucb}

For our first variant of KernelUCB, we use our exact confidence bound $\mathrm{UCB}_{\mathcal{F}_t}(x)$. In particular, we set $u_t(x)$ to be the numerical solution of (\ref{eqn:ucb_cone_prog}). The resulting algorithm is called Kernel Convex Martingale Mixture UCB (Kernel CMM-UCB).

\subsection{Kernel DMM-UCB}
\label{sec:kernel_dmm_ucb}

For our second variant of KernelUCB, we use the confidence bound in (\ref{eqn:dual_ucb_alpha}) that comes from the dual problem associated with $\max_{f \in \mathcal{F}_t}\{f(x)\}$. In the interest of computational efficiency, we replace $\min_{\alpha > 0}$ in (\ref{eqn:dual_ucb_alpha}) with $\min_{\alpha \in A}$, for a small grid $A$ of values of $\alpha$. In particular, we set
\begin{equation*}
u_t(x) = \min_{\alpha \in A}\left\{\mu_{\alpha, t}(x) + \tfrac{1}{\sqrt{\alpha}}\widetilde{R}_{\alpha,t}\rho_{\alpha, t}(x)\right\}.
\end{equation*}

We call the resulting algorithm Kernel Dual Martingale Mixture UCB (Kernel DMM-UCB). Because we only optimise $\alpha$ over the grid $A$, this upper confidence bound is looser than the exact confidence bound used by Kernel CMM-UCB. In return, the cost of computing this UCB is only a factor of $|A|$ more than the cost of computing the analytic UCB in Cor. \ref{cor:analytic_ucb}.

\subsection{Kernel AMM-UCB}
\label{sec:kernel_amm_ucb}

For our third variant of KernelUCB, we set $u_t(x)$ to be the analytic UCB from Cor. \ref{cor:analytic_ucb} with a fixed value of $\alpha$. The resulting algorithm is called Kernel Analytic Martingale Mixture UCB (Kernel AMM-UCB).

\section{THEORETICAL ANALYSIS}
\label{sec:theory}

In this section, we show that our analytic confidence bound from Cor. \ref{cor:analytic_ucb} is always tighter than the confidence bounds in Thm. 3.11 of \citep{abbasi2012online} and Thm. 2 of \citep{chowdhury2017kernelized}, which hold under the same conditions and assumptions. In addition, we establish cumulative regret bounds for our KernelUCB algorithms. We begin by formally stating the assumptions under which our analysis holds. We assume that $\mathcal{X}$ is a compact subset of $\mathbb{R}^d$. In the noisy setting, we assume that the noise variables are conditionally $\sigma$-sub-Gaussian, which means
\begin{equation}
\forall \lambda \in \mathbb{R}, \;\ \mathbb{E}\left[\mathrm{exp}(\lambda \epsilon_t)|\mathcal{D}_{t-1}\right] \leq \mathrm{exp}(\lambda^2\sigma^2/2).\label{eqn:subgaussian}
\end{equation}
For some $B, C > 0$, we assume that $\norm{f^{\star}}_{\mathcal{H}} \leq B$ and $\sup_{x}|f^{\star}(x)| \leq C$. If the kernel function is uniformly bounded by some $\kappa > 0$, i.e. $\sup_{x, x^{\prime}}|k(x, x^{\prime})| \leq \kappa$, then $\norm{f^{\star}}_{\mathcal{H}} \leq B$ implies that $C \leq \kappa B$.






\subsection{Tighter Confidence Bounds}
\label{sec:tighter_conf}

The confidence bounds in \citep{abbasi2012online} and \citep{chowdhury2017kernelized} are the same as our analytic UCB in (\ref{eqn:analytic_ucb}), except that the (scaled) radius $\widetilde{R}_{\alpha,t}/\sqrt{\alpha}$ is replaced with a different quantity. For any values of the regularisation parameters $\lambda > 0$ and $\eta > 0$ used in \citep{abbasi2012online} and \citep{chowdhury2017kernelized} respectively, we show that we can set $c$ and $\alpha$ such that $\widetilde{R}_{\alpha,t}/\sqrt{\alpha}$ is strictly less than these other radius quantities. This result is stated in Thm. \ref{thm:confidence_tightness} and proved in App. \ref{app:confidence_tightness}.

\begin{theorem}\label{thm:confidence_tightness}
For every $\delta \in (0, 1]$ and every value of the regularisation parameter $\lambda > 0$, the RHS of (\ref{eqn:analytic_ucb}) with $c=\sigma^2/\lambda$ and $\alpha = \lambda$ is strictly less than the UCB in Thm.\ 3.11 of \citep{abbasi2012online}.

For every $\delta \in (0, 1]$ and every value of the regularisation parameter $\eta > 0$, the RHS of (\ref{eqn:analytic_ucb}) with $c=\sigma^2/(1+\eta)$ and $\alpha = 1 + \eta$ is strictly less than the UCB in Thm. 2 of \citep{chowdhury2017kernelized}.
\end{theorem}

\subsection{Cumulative Regret Bounds}
\label{sec:regret}

We bound the cumulative regret in terms of two kernel-dependent quantities. The first is called the maximum information gain, and is defined as
\begin{equation*}
\gamma_{t}(\alpha) := \max_{x_1, \dots, x_t}\left\{\tfrac{1}{2}\ln\left(\det\left(\tfrac{1}{\alpha}\bs{K}_t + \bs{I}\right)\right)\right\}.
\end{equation*}

The second is defined as $\tau_t(\alpha) := |\mathcal{T}_t(\alpha)|$, where
\begin{equation}
\mathcal{T}_t(\alpha) = \{s \in [t]: \tfrac{1}{\alpha}\rho_{\alpha, s-1}^2(x_s) \geq 1\}.\label{eqn:tau_set}
\end{equation}

$\tau_t(\alpha)$ counts the number of rounds in which the scaled variance $\rho_{\alpha, s-1}^2(x_s)/\alpha$, also known as the elliptical potential, is large. We call $\tau_t(\alpha)$ the elliptical potential count. First, we upper bound our radius quantity $\widetilde{R}_{\alpha, t}$ in terms of the maximum information gain. We show that when $\alpha = \sigma^2/c$, the quadratic terms in $\widetilde{R}_{\alpha,t}$ cancel out (see Lemma \ref{lem:special_rad_bound}), and we are left with
\begin{align}
\widetilde{R}_{\sigma^2/c,t} &\leq \sqrt{2\sigma^2\gamma_t(\sigma^2/c) + \tfrac{\sigma^2B^2}{c} + 2\sigma^2\ln(\tfrac{1}{\delta})}.\label{eqn:rad_bound}
\end{align}

We now proceed to upper bound the cumulative regret of Kernel AMM-UCB with $\alpha = \sigma^2/c$. Using a standard optimism argument (see e.g. Prop. 1 in \citealt{russo2013eluder} and also Lemma \ref{lem:per_regret} in App. \ref{app:generic_regret}), the regret of any UCB algorithm at round $t$ is (with high probability) upper bounded by the difference between the upper and lower confidence bounds (at $x_t$) that it uses. For Kernel AMM-UCB, this bound is
\begin{equation}
r_t \leq \tfrac{2}{\sqrt{\sigma^2/c}}\widetilde{R}_{\sigma^2/c,t-1}\rho_{\sigma^2/c, t-1}(x_t).\label{eqn:per_round_regret}
\end{equation}\vspace{-4mm}

Because the confidence bounds used by Kernel CMM-UCB (with the same $c$) and Kernel DMM-UCB (with the same $c$ and with $\alpha = \sigma^2/c \in A$) are never looser than those used by Kernel AMM-UCB, these algorithms also satisfy Eq. (\ref{eqn:per_round_regret}). Next, as is now somewhat common in analyses of (generalised) linear bandit algorithms \citep{gales2022norm, lee2024improved, janz2024exploration}, we split $r_{1:T}$ into $\sum_{t \in \mathcal{T}_T(\sigma^2/c)}r_t$ and $\sum_{t \notin \mathcal{T}_T(\sigma^2/c)}r_t$. Using $\sup_{x}|f^{\star}(x)| \leq C$, the first part can be bounded by $2C\tau_{T}(\sigma^2/c)$. Eq. (\ref{eqn:rad_bound}), Eq. (\ref{eqn:per_round_regret}) and the Elliptical Potential Lemma (Lemma \ref{lem:elliptical}) are used to bound the second part in terms of $\gamma_{T}(\sigma^2/c)$, which yields the regret bound in Thm \ref{thm:regret_bound} (see App. \ref{app:generic_regret} for a full proof).
\begin{theorem}
For any covariance scale $c > 0$, with probability at least $1 - \delta$, for all $T \geq 1$, the cumulative regret of Kernel CMM-UCB, Kernel DMM-UCB (with a grid $A$ containing $\alpha = \sigma^2/c$) and Kernel AMM-UCB (with $\alpha = \sigma^2/c$) satisfies
\begin{equation*}
r_{1:T} \leq 2C\tau_T(\tfrac{\sigma^2}{c}) + \sigma\sqrt{24T\gamma_T(\tfrac{\sigma^2}{c})(\gamma_T(\tfrac{\sigma^2}{c}) + \tfrac{B^2}{2c} + \ln\tfrac{1}{\delta})}.
\end{equation*}
\label{thm:regret_bound}
\end{theorem}\vspace{-3mm}
Unless $\sigma$ is close enough to $0$, the dominant term in this regret bound is the second term. Its dependence on $T$ and the maximum information gain matches that of the best existing regret bound for KernelUCB \citep{whitehouse2023sublinear}, although the dependence on $\gamma_T$ is sub-optimal by a factor of $\sqrt{\gamma_T}$. However, due to the lower bound in \citep{lattimore2023lower}, this cannot be improved by using a better confidence sequence for $f^{\star}$.

We use bounds on $\gamma_T(\alpha)$ to derive new upper bounds $\tau_{T}(\alpha)$. Bounds on $\tau_T(\alpha)$ exist for the linear setting (e.g. \citealp{gales2022norm}), where $k(x, x^{\prime}) = x^{\top}x^{\prime}$, and $\gamma_T(\alpha)$ admits a relatively simple upper bound. However, for more general kernels, we must resort to more elaborate bounds on $\gamma_T(\alpha)$, which makes bounding $\tau_{T}(\alpha)$ rather complicated (cf. Lemma \ref{lem:ep_poly} and Lemma \ref{lem:ep_exp}). These bounds allow us to: a) make the dependence of Thm. (\ref{thm:regret_bound}) on $T$ explicit for specific kernels; b) provide sub-$\sqrt{T}$ regret bounds for the noiseless case (Sec. \ref{sec:noiseless}).

For the RBF kernel, and with $c \propto 1$, we show in Prop. \ref{pro:exp_regret} that the regret bound in Thm. \ref{thm:regret_bound} is $\mathcal{O}(\sqrt{T}(\ln(T))^{d+1})$. For the Mat\'{e}rn kernel with smoothness parameter $\nu$ and $c \propto T^{-\frac{d}{2d + 2\nu}}$, we show in Prop. \ref{pro:poly_regret} that the regret bound in Thm. \ref{thm:regret_bound} is $\mathcal{O}(T^{\frac{2d + \nu}{2d + 2\nu}}(\ln(T))^{\frac{2\nu}{d + 2\nu}})$. The same rates are obtained in \citep{whitehouse2023sublinear}, which are the best known rates for (vanilla) KernelUCB.


\subsection{The Noiseless Setting}
\label{sec:noiseless}

In a 2022 COLT open problem, \citet{vakili2022open} discussed the lack of order-optimal regret bounds for noiseless kernel bandits. Though the open problem has since been resolved \citep{salgia2024random}, we believe it is still not known whether KernelUCB can achieve sub-$\sqrt{T}$ regret bounds in the noiseless setting (see Sec. 4.1. in \citealp{vakili2022open}). We show that this is in fact the case. In the noiseless setting, the constraint $\|\bs{f}_t - \bs{y}_t\|_2 \leq R_t$ in Lem. \ref{lem:rkhs_conf_set} can be replaced with $\|\bs{f}_t - \bs{y}_t\|_2  = 0$. Since $R_t = 0$, the radius of the analytic confidence bound in Cor. \ref{cor:analytic_ucb} becomes
\begin{equation*}
\widetilde{R}_{\alpha, t}^2 = \alpha B^2 - \bs{y}_t^{\top}(\tfrac{1}{\alpha}\bs{K}_t + \bs{I})^{-1}\bs{y}_t \leq \alpha B^2
\end{equation*}
Following the proof of Thm. \ref{thm:regret_bound}, one can show that, for all $\alpha > 0$ and $T \geq 1$,
\begin{equation*}
r_{1:T} \leq 2C\tau_T(\alpha) + \sqrt{12\alpha B^2T\gamma_T(\alpha)}.
\end{equation*}

\begin{table*}
\caption{Cumulative regret of our KernelUCB-style algorithms as well as AY-GP-UCB, IGP-UCB and a random baseline after $T=1000$ rounds with $d=3$. We show the mean $\pm$ standard deviation over 10 repetitions.}
\label{tab:regret}
\centering
\resizebox{0.925\textwidth}{!}{\begin{tabular}{lcccccc}
\toprule
& \multicolumn{2}{c}{RBF Kernel} & \multicolumn{2}{c}{Mat\'{e}rn Kernel ($\nu=5/2$)} & \multicolumn{2}{c}{Mat\'{e}rn Kernel ($\nu=3/2$)}\\
\cmidrule(r){2-7}
 & $\ell = 0.5$ & $\ell = 0.2$ & $\ell = 0.5$ & $\ell = 0.2$ & $\ell = 0.5$ & $\ell = 0.2$\\
\midrule
Kernel DMM-UCB (Ours) & \textbf{32.2} $\pm$ 20.9 & \textbf{491.4} $\pm$ 117.1 & \textbf{129.5} $\pm$ 45.6 & \textbf{795.1} $\pm$ 206.0 & \textbf{195.6} $\pm$ 78.0 & \textbf{814.1} $\pm$ 344.4\\
Kernel AMM-UCB (Ours) & 88.8 $\pm$ 6.1 & 1206.2 $\pm$ 20.8 & 197.0 $\pm$ 24.4 & 1661.5 $\pm$ 90.1 & 316.1 $\pm$ 51.1 & 1741.2 $\pm$ 351.2\\
AY-GP-UCB	& 136.9 $\pm$ 12.7 & 1518.4 $\pm$ 38.9 & 331.7 $\pm$ 45.2 & 2382.4 $\pm$ 135.4 & 546.0 $\pm$ 70.0 & 2421.3 $\pm$ 568.5\\
IGP-UCB		& 314.1 $\pm$ 110.5 & 1433.0 $\pm$ 122.8 & 553.3 $\pm$ 67.5 & 1853.1 $\pm$ 105.7 & 655.6 $\pm$ 67.4 & 1707.5 $\pm$ 375.5\\
Random	& 4282.4 $\pm$ 1015.4 & 3872.4 $\pm$ 783.7 & 4264.7 $\pm$ 778.0 & 3677.5 $\pm$ 559.2 & 4175.1 $\pm$ 681.0 & 3442.0 $\pm$ 1080.4\\
\bottomrule
\end{tabular}}
\end{table*}
\begin{figure*}
\centering
\includegraphics[width=0.925\textwidth]{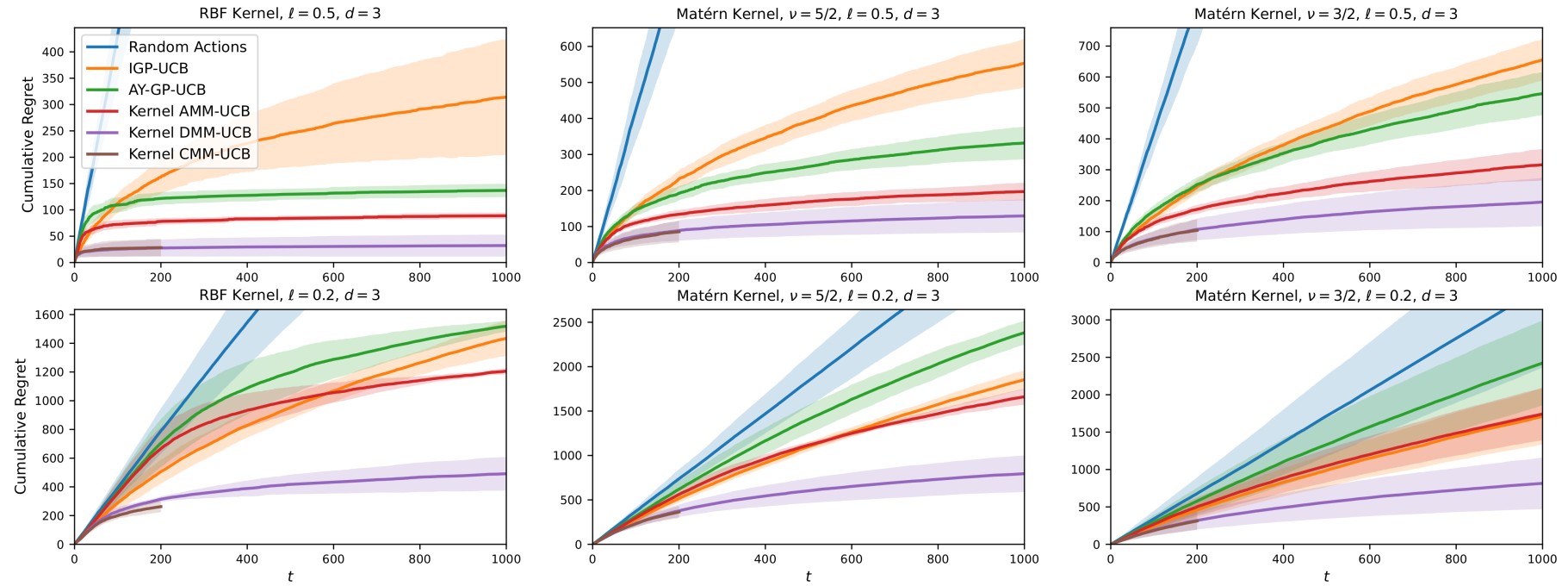}
\caption{Cumulative regret of our KernelUCB-style algorithms as well as AY-GP-UCB, IGP-UCB and a random baseline over $T=1000$ rounds for $d=3$ and various kernels (columns) and length scales (rows). We show the mean $\pm$ standard deviation over 10 repetitions.}
\label{fig:regret}
\end{figure*}

From here, in App. \ref{app:regret_noiseless}, we show that for the RBF kernel, if $\alpha \propto 1/T$, then this bound becomes $r_{1:T} = \mathcal{O}(\ln(T)^{d+1})$. For the Mat\'{e}rn kernel, if $\alpha \propto T^{-\frac{2\nu^2 + 2\nu d}{2\nu^2 + 2\nu d + d^2}}$, then the regret bound becomes $r_{1:T} = \mathcal{O}(T^{\frac{\nu d + d^2}{2\nu^2 + 2\nu d + d^2}}\ln(T))$. Whenever $\nu > d/\sqrt{2}$, this bound grows slower that $\sqrt{T}$.

\section{EXPERIMENTS}
\label{sec:experiments}

We aim to verify that when our confidence bounds replace existing ones, KernelUCB shows better empirical performance at a similar computational cost. We compare \emph{Kernel CMM-UCB}, \emph{Kernel DMM-UCB}, \emph{Kernel AMM-UCB}, \emph{AY-GP-UCB}, \emph{IGP-UCB} and \emph{Random}, which chooses actions uniformly at random.

We run each algorithm in several kernel bandit problems. To generate each reward function, we sample 20 inducing points $z_1, \dots, z_{20}$ uniformly at random from the hypercube $[0,1]^d$ and a 20-dimensional weight vector $\bs{w} \in \mathbb{R}^{20}$ from a standard normal distribution. The test function is $f^{\star}(x) = b\sum_{i=1}^{20}w_ik(x, z_i)$, with $b$ being set such that $\norm{f^{\star}}_{\mathcal{H}} = B$. The noise variables $\epsilon_1, \epsilon_2, \dots$ are drawn independently from a normal distribution with mean 0 and standard deviation $\sigma$. In each round $t$, the action set $\mathcal{X}_t$ consists of 100 $d$-dimensional vectors drawn uniformly at random from the hypercube $[0, 1]^d$. We present results for $d \in \{2, 3, 4\}$ and with the RBF and Mat\'{e}rn kernels with length scale $\ell \in \{0.2, 0.5\}$.

We run each algorithm for $T=1000$ rounds with $\delta = 0.01$ and the true values $B = 10$ and $\sigma = 0.1$. We run IGP-UCB with the recommend value of $\eta = 2/T$. For each of our methods, we set $c=1$ for the RBF kernel and $c = T^{-\frac{d}{2d + 2\nu}}$ for the Mat\'{e}rn kernel. Equivalently, for AY-GP-UCB, we set $\lambda = \sigma^2$ for the RBF kernel and $\lambda = \sigma^2T^{\frac{d}{2d + 2\nu}}$ for the Mat\'{e}rn kernel. We run Kernel AMM-UCB with $\alpha = \sigma^2/c$ and Kernel DMM-UCB with $A = \{0.1\sigma^2/c, 0.3\sigma^2/c, \sigma^2/c, 3\sigma^2/c, 10\sigma^2/c\}$. Due to prohibitive cost, we only run Kernel CMM-UCB for 200 rounds. All experiments were run on a single computer with an Intel i5-1145G7 CPU.

Fig. \ref{fig:regret} shows the cumulative regret of each algorithm with $d=3$ (results with $d \in \{2,4\}$ are shown in App. \ref{app:2d_and_4d}). The cumulative regret at $T=1000$ is displayed in Tab. \ref{tab:regret}. Our Kernel CMM-UCB algorithm consistently achieves the lowest cumulative regret over the first 200 steps, and is closely followed by Kernel DMM-UCB which achieves the lowest cumulative regret over 1000 steps. Fig. \ref{fig:time_per_step} (in App. \ref{app:time_per_step}) shows the wall-clock time per step for each method. IGP-UCB, AY-GP-UCB and Kernel AMM-UCB take approximately the same time per step. Since we used a grid $A$ of size 5, the time per step for Kernel DMM-UCB is a little under 5 times that of Kernel AMM-UCB. The time per step for Kernel CMM-UCB grows at the fastest rate in $t$. At $t=200$, the time per step for Kernel CMM-UCB is more than 15 seconds, whereas the time per step for every other method is below 0.02 seconds.

\section{CONCLUSION}
In this paper, we developed new confidence bounds for sequential kernel regression. In our theoretical analysis, we proved that our new confidence bounds are always tighter than existing confidence bounds, and we showed that our KernelUCB algorithms can have sub-$\sqrt{T}$ cumulative regret in the noiseless setting. Our experimental results suggest that our Kernel DMM-UCB algorithm is the preferred choice. It closely matches the empirical performance of our effective but expensive Kernel CMM-UCB algorithm, yet it has roughly the same cost as our Kernel AMM-UCB method.

Though we applied our confidence bounds only to the kernel bandit problem, and the KernelUCB algorithm in particular, they are a generic tool that could be used anywhere that existing confidence bounds are used. For instance, one could use our confidence bounds in the $\pi$-GP-UCB algorithm \citep{janz2020bandit}, which would give a better rate for the Mat\'{e}rn kernel. In addition, one could apply our confidence bounds to kernelised reinforcement learning \citep{yang2020function, liu2022provably, vakili2023kernelized} or adaptive control \citep{kakade2020information}.

Finally, we comment on some limitations of our work. We assume that the kernel as well as reasonably good upper bounds on $\sigma$ and $B$ are known in advance, which is often not the case in practice. While using a loose upper bound on $\sigma$ appears to have little impact, we found that using a loose bound on $B$ can harm empirical performance (see Fig. \ref{fig:loose_bsig}). In future work, we would like to address this limitation by using model selection methods to estimate $\norm{f^{\star}}_{\mathcal{H}}$ and choose the kernel.



\bibliography{main}

\begin{thebibliography}{}

\bibitem[Abbasi-Yadkori, 2012]{abbasi2012online}
Abbasi-Yadkori, Y. (2012).
\newblock {\em Online learning for linearly parametrized control problems}.
\newblock PhD thesis, University of Alberta.

\bibitem[Abbasi-Yadkori et~al., 2011]{abbasi2011improved}
Abbasi-Yadkori, Y., P{\'a}l, D., and Szepesv{\'a}ri, C. (2011).
\newblock Improved algorithms for linear stochastic bandits.
\newblock {\em Advances in neural information processing systems}, 24.

\bibitem[Abbasi-Yadkori et~al., 2012]{abbasi2012onlinetoconc}
Abbasi-Yadkori, Y., Pal, D., and Szepesvari, C. (2012).
\newblock Online-to-confidence-set conversions and application to sparse
  stochastic bandits.
\newblock In {\em Artificial Intelligence and Statistics}, pages 1--9. PMLR.

\bibitem[Agrawal et~al., 2019]{agrawal2019differentiable}
Agrawal, A., Amos, B., Barratt, S., Boyd, S., Diamond, S., and Kolter, J.~Z.
  (2019).
\newblock Differentiable convex optimization layers.
\newblock {\em Advances in neural information processing systems}, 32.

\bibitem[Auer, 2002]{auer2002using}
Auer, P. (2002).
\newblock Using confidence bounds for exploitation-exploration trade-offs.
\newblock {\em Journal of Machine Learning Research}, 3(Nov):397--422.

\bibitem[Belkin, 2018]{belkin2018approximation}
Belkin, M. (2018).
\newblock Approximation beats concentration? an approximation view on inference
  with smooth radial kernels.
\newblock In {\em Conference On Learning Theory}, pages 1348--1361. PMLR.

\bibitem[Berkenkamp et~al., 2017]{berkenkamp2017safe}
Berkenkamp, F., Turchetta, M., Schoellig, A., and Krause, A. (2017).
\newblock Safe model-based reinforcement learning with stability guarantees.
\newblock {\em Advances in neural information processing systems}, 30.

\bibitem[Blondel et~al., 2022]{blondel2022efficient}
Blondel, M., Berthet, Q., Cuturi, M., Frostig, R., Hoyer, S.,
  Llinares-L{\'o}pez, F., Pedregosa, F., and Vert, J.-P. (2022).
\newblock Efficient and modular implicit differentiation.
\newblock {\em Advances in neural information processing systems},
  35:5230--5242.

\bibitem[Boyd and Vandenberghe, 2004]{boyd2004convex}
Boyd, S. and Vandenberghe, L. (2004).
\newblock {\em Convex optimization}.
\newblock Cambridge university press.

\bibitem[Cesa-Bianchi and Lugosi, 2006]{cesa2006prediction}
Cesa-Bianchi, N. and Lugosi, G. (2006).
\newblock {\em Prediction, learning, and games}.
\newblock Cambridge university press.

\bibitem[Chowdhury and Gopalan, 2017]{chowdhury2017kernelized}
Chowdhury, S.~R. and Gopalan, A. (2017).
\newblock On kernelized multi-armed bandits.
\newblock In {\em International Conference on Machine Learning}, pages
  844--853. PMLR.

\bibitem[de~la Pe{\~n}a et~al., 2004]{pena2004self}
de~la Pe{\~n}a, V.~H., Klass, M.~J., and Leung~Lai, T. (2004).
\newblock Self-normalized processes: exponential inequalities, moment bounds
  and iterated logarithm laws.
\newblock {\em Annals of Probability}, 32:1902--1933.

\bibitem[de~la Pe{\~n}a et~al., 2009]{pena2009self}
de~la Pe{\~n}a, V.~H., Lai, T.~L., and Shao, Q.-M. (2009).
\newblock {\em Self-normalized processes: Limit theory and Statistical
  Applications}.
\newblock Springer.

\bibitem[Durand et~al., 2018]{durand2018streaming}
Durand, A., Maillard, O.-A., and Pineau, J. (2018).
\newblock Streaming kernel regression with provably adaptive mean, variance,
  and regularization.
\newblock {\em The Journal of Machine Learning Research}, 19(1):650--683.

\bibitem[Emmenegger et~al., 2023]{emmenegger2023likelihood}
Emmenegger, N., Mutn{\`y}, M., and Krause, A. (2023).
\newblock Likelihood ratio confidence sets for sequential decision making.
\newblock In {\em Advances in Neural Information Processing Systems (NeurIPS)}.

\bibitem[Flynn et~al., 2023]{flynn2023improved}
Flynn, H., Reeb, D., Kandemir, M., and Peters, J. (2023).
\newblock Improved algorithms for stochastic linear bandits using tail bounds
  for martingale mixtures.
\newblock In {\em Advances in Neural Information Processing Systems (NeurIPS)}.

\bibitem[Gales et~al., 2022]{gales2022norm}
Gales, S.~B., Sethuraman, S., and Jun, K.-S. (2022).
\newblock Norm-agnostic linear bandits.
\newblock In {\em International Conference on Artificial Intelligence and
  Statistics}, pages 73--91. PMLR.

\bibitem[Jamieson et~al., 2014]{jamieson2014lil}
Jamieson, K., Malloy, M., Nowak, R., and Bubeck, S. (2014).
\newblock lil’ucb: An optimal exploration algorithm for multi-armed bandits.
\newblock In {\em Conference on Learning Theory}, pages 423--439. PMLR.

\bibitem[Janz et~al., 2020]{janz2020bandit}
Janz, D., Burt, D., and Gonz{\'a}lez, J. (2020).
\newblock Bandit optimisation of functions in the {M}at{\'e}rn kernel {RKHS}.
\newblock In {\em International Conference on Artificial Intelligence and
  Statistics}, pages 2486--2495. PMLR.

\bibitem[Janz et~al., 2024]{janz2024exploration}
Janz, D., Liu, S., Ayoub, A., and Szepesv{\'a}ri, C. (2024).
\newblock Exploration via linearly perturbed loss minimisation.
\newblock In {\em International Conference on Artificial Intelligence and
  Statistics}, pages 721--729. PMLR.

\bibitem[Kakade et~al., 2020]{kakade2020information}
Kakade, S., Krishnamurthy, A., Lowrey, K., Ohnishi, M., and Sun, W. (2020).
\newblock Information theoretic regret bounds for online nonlinear control.
\newblock {\em Advances in Neural Information Processing Systems},
  33:15312--15325.

\bibitem[Kimeldorf and Wahba, 1971]{kimeldorf1971some}
Kimeldorf, G. and Wahba, G. (1971).
\newblock Some results on {T}chebycheffian spline functions.
\newblock {\em Journal of mathematical analysis and applications},
  33(1):82--95.

\bibitem[Lattimore, 2023]{lattimore2023lower}
Lattimore, T. (2023).
\newblock A lower bound for linear and kernel regression with adaptive
  covariates.
\newblock In {\em The Thirty Sixth Annual Conference on Learning Theory}, pages
  2095--2113. PMLR.

\bibitem[Lattimore and Szepesv{\'a}ri, 2020]{lattimore2020bandit}
Lattimore, T. and Szepesv{\'a}ri, C. (2020).
\newblock {\em Bandit algorithms}.
\newblock Cambridge University Press.

\bibitem[Lee et~al., 2024]{lee2024improved}
Lee, J., Yun, S.-Y., and Jun, K.-S. (2024).
\newblock Improved regret bounds of (multinomial) logistic bandits via
  regret-to-confidence-set conversion.
\newblock In {\em International Conference on Artificial Intelligence and
  Statistics}, pages 4474--4482. PMLR.

\bibitem[Li and Scarlett, 2022]{li2022gaussian}
Li, Z. and Scarlett, J. (2022).
\newblock Gaussian process bandit optimization with few batches.
\newblock In {\em International Conference on Artificial Intelligence and
  Statistics}, pages 92--107. PMLR.

\bibitem[Liu and Su, 2022]{liu2022provably}
Liu, S. and Su, H. (2022).
\newblock Provably efficient kernelized {Q}-learning.
\newblock {\em arXiv preprint arXiv:2204.10349}.

\bibitem[Look et~al., 2020]{look2020differentiable}
Look, A., Doneva, S., Kandemir, M., Gemulla, R., and Peters, J. (2020).
\newblock Differentiable implicit layers.
\newblock In {\em NeurIPS ML for Engineering Workshop}.

\bibitem[Neiswanger and Ramdas, 2021]{neiswanger2021uncertainty}
Neiswanger, W. and Ramdas, A. (2021).
\newblock Uncertainty quantification using martingales for misspecified
  {G}aussian processes.
\newblock In {\em Algorithmic Learning Theory}, pages 963--982. PMLR.

\bibitem[Nemirovski, 2004]{nemirovski2004interior}
Nemirovski, A. (2004).
\newblock Interior point polynomial time methods in convex programming.
\newblock {\em Lecture notes}, 42(16):3215--3224.

\bibitem[Russo and Van~Roy, 2013]{russo2013eluder}
Russo, D. and Van~Roy, B. (2013).
\newblock Eluder dimension and the sample complexity of optimistic exploration.
\newblock {\em Advances in Neural Information Processing Systems}, 26.

\bibitem[Salgia et~al., 2024]{salgia2024random}
Salgia, S., Vakili, S., and Zhao, Q. (2024).
\newblock Random exploration in bayesian optimization: Order-optimal regret and
  computational efficiency.
\newblock In {\em International Conference on Machine Learning}.

\bibitem[Santin and Schaback, 2016]{santin2016approximation}
Santin, G. and Schaback, R. (2016).
\newblock Approximation of eigenfunctions in kernel-based spaces.
\newblock {\em Advances in Computational Mathematics}, 42(4):973--993.

\bibitem[Scharnhorst et~al., 2022]{scharnhorst2022robust}
Scharnhorst, P., Maddalena, E.~T., Jiang, Y., and Jones, C.~N. (2022).
\newblock Robust uncertainty bounds in reproducing kernel hilbert spaces: A
  convex optimization approach.
\newblock {\em IEEE Transactions on Automatic Control}, 68(5):2848--2861.

\bibitem[Sch{\"o}lkopf et~al., 2001]{scholkopf2001generalized}
Sch{\"o}lkopf, B., Herbrich, R., and Smola, A.~J. (2001).
\newblock A generalized representer theorem.
\newblock In {\em International conference on computational learning theory},
  pages 416--426. Springer.

\bibitem[Srinivas et~al., 2010]{srinivas2010gaussian}
Srinivas, N., Krause, A., Kakade, S., and Seeger, M. (2010).
\newblock Gaussian process optimization in the bandit setting: No regret and
  experimental design.
\newblock In {\em Proc. International Conference on Machine Learning (ICML)}.

\bibitem[Sui et~al., 2015]{sui2015safe}
Sui, Y., Gotovos, A., Burdick, J., and Krause, A. (2015).
\newblock Safe exploration for optimization with {G}aussian processes.
\newblock In {\em International conference on machine learning}, pages
  997--1005. PMLR.

\bibitem[Sutton and Barto, 2018]{sutton2018reinforcement}
Sutton, R.~S. and Barto, A.~G. (2018).
\newblock {\em Reinforcement learning: An introduction}.
\newblock MIT press.

\bibitem[Vakili, 2022]{vakili2022open}
Vakili, S. (2022).
\newblock Open problem: Regret bounds for noise-free kernel-based bandits.
\newblock In {\em Conference on Learning Theory}, pages 5624--5629. PMLR.

\bibitem[Vakili et~al., 2021]{vakili2021information}
Vakili, S., Khezeli, K., and Picheny, V. (2021).
\newblock On information gain and regret bounds in {G}aussian process bandits.
\newblock In {\em International Conference on Artificial Intelligence and
  Statistics}, pages 82--90. PMLR.

\bibitem[Vakili and Olkhovskaya, 2023]{vakili2023kernelized}
Vakili, S. and Olkhovskaya, J. (2023).
\newblock Kernelized reinforcement learning with order optimal regret bounds.
\newblock In {\em Advances in Neural Information Processing Systems (NeurIPS)}.

\bibitem[Valko et~al., 2013]{valko2013finite}
Valko, M., Korda, N., Munos, R., Flaounas, I., and Cristianini, N. (2013).
\newblock {Finite-Time Analysis of Kernelised Contextual Bandits}.
\newblock In {\em Uncertainty in Artificial Intelligence}.

\bibitem[Ville, 1939]{ville1939etude}
Ville, J. (1939).
\newblock Etude critique de la notion de collectif.
\newblock {\em Bull. Amer. Math. Soc}, 45(11):824.

\bibitem[Whitehouse et~al., 2023]{whitehouse2023sublinear}
Whitehouse, J., Wu, Z.~S., and Ramdas, A. (2023).
\newblock On the sublinear regret of {GP-UCB}.
\newblock {\em Advances in Neural Information Processing Systems}.

\bibitem[Yang et~al., 2020]{yang2020function}
Yang, Z., Jin, C., Wang, Z., Wang, M., and Jordan, M.~I. (2020).
\newblock On function approximation in reinforcement learning: optimism in the
  face of large state spaces.
\newblock In {\em Proceedings of the 34th International Conference on Neural
  Information Processing Systems}, pages 13903--13916.

\bibitem[Zhang, 2006]{zhang2006schur}
Zhang, F. (2006).
\newblock {\em The {S}chur complement and its applications}, volume~4.
\newblock Springer Science \& Business Media.

\end{thebibliography}
\newpage
\appendix
\onecolumn
 
\section{Tail Bound Derivations}

\subsection{Derivation of Eq. (\ref{eqn:pre_radius_def})}
\label{app:pre_radius_derivation}

For convenience, we first re-state the setting in which the tail bound holds. $(\mathcal{D}_t|t \in \mathbb{N})$ is any filtration such that $x_t$ is $\mathcal{D}_{t-1}$-measurable and $y_t$ is $\mathcal{D}_t$-measurable (e.g. $\mathcal{D}_t = \sigma(x_1, y_1, \dots, x_{t}, y_{t}, x_{t+1})$). $(Z_t(g)| t \in \mathbb{N}, g \in \mathbb{R})$ is a stochastic process, indexed by a time $t$ and a real number $g$, which is adapted to a filtration $(\mathcal{D}_t|t \in \mathbb{N})$. $(g_t|t \in \mathbb{N})$ is a sequence of predictable guesses and $(\lambda_t|t \in \mathbb{N})$ is a sequence of predictable random variables. $(P_t| t \in \mathbb{N})$ is any sequence of mixture distributions that satisfies: (a) $P_t$ is a distribution over $\bs{g}_t \in \mathbb{R}^t$; (b) $P_t$ is $\mathcal{D}_{t-1}$-measurable; (c) $\int P_t(\bs{g}_t)dg_t=P_{t-1}(\bs{g}_{t-1})$ for all $t$. Define
\begin{equation}
M_t(\bs{g}_t, \bs{\lambda}_t) := \mathrm{exp}\left(\sum_{s=1}^{t}\lambda_s Z_s(g_s) - \psi_s(g_s, \lambda_s)\right), \quad \psi_t(g_t, \lambda_t) := \ln\left(\mathbb{E}\left[\mathrm{exp}(\lambda_t Z_t(g_t)) | \mathcal{D}_{t-1}\right]\right).\label{eqn:def_Mt_2}
\end{equation}

Due to Thm. 5.1 of \citep{flynn2023improved}, for any $\delta \in (0,1]$, with probability at least $1 - \delta$,
\begin{equation}
\forall t \geq 1: \quad \mathrm{ln}\left(\mathop{\mathbb{E}}_{\bs{g}_t \sim P_t}\left[M_t(\bs{g}_t, \bs{\lambda}_t)\right]\right) \leq \mathrm{ln}(1/\delta).\label{eqn:adaptive_tail_bound_2}
\end{equation}

We choose $Z_t(g_t) = (g_t - f^{\star}(x_t))\epsilon_t$. Using the $\sigma$-sub-Gaussian property of the noise variables, we have
\begin{equation*}
\psi_t(g_t, \lambda_t) = \ln\left(\mathbb{E}\left[\mathrm{exp}(\lambda_t(g_t - f^{\star}(x_t))\epsilon_t) | \mathcal{D}_{t-1}\right]\right) \leq \frac{\sigma^2\lambda_t^2(g_t - f^{\star}(x_t))^2}{2}.
\end{equation*}

Combining this with (\ref{eqn:adaptive_tail_bound_2}), we have
\begin{align}
\mathrm{ln}&\left(\mathop{\mathbb{E}}_{\bs{g}_t \sim P_t}\left[\mathrm{exp}\left(\sum_{s=1}^{t}\lambda(g_s - f^{\star}(x_s))\epsilon_s - \frac{\sigma^2\lambda_s^2(g_s - f^{\star}(x_s))^2}{2}\right)\right]\right)\label{eqn:supermart}\\
&\leq \mathrm{ln}\left(\mathop{\mathbb{E}}_{\bs{g}_t \sim P_t}\left[\mathrm{exp}\left(\sum_{s=1}^{t}\lambda Z_s(g_s) - \psi_s(g_s, \lambda)\right)\right]\right) = \mathrm{ln}\left(\mathop{\mathbb{E}}_{\bs{g}_t \sim P_t}\left[M_t(\bs{g}_t, \lambda)\right]\right) \leq \mathrm{ln}(1/\delta).\nonumber
\end{align}

Next, we rearrange the integrand on the LHS of (\ref{eqn:supermart}). For every $s \in [t]$, we have
\begin{align*}
(f^{\star}(x_s) - y_s)^2 - (g_s - y_s)^2 &= (\epsilon_s)^2 - (g_s - f^{\star}(x_s) - \epsilon_s)^2\\
&= (\epsilon_s)^2 - (g_s - f^{\star}(x_s))^2 + 2(g_s - f^{\star}(x_s))\epsilon_s - (\epsilon_s)^2\\
&= 2(g_s - f^{\star}(x_s))\epsilon_s - (g_s - f^{\star}(x_s))^2.
\end{align*}

This means that
\begin{align*}
\lambda_s(g_s - f^{\star}(x_s))\epsilon_s - \frac{\sigma^2\lambda_s^2(g_s - f^{\star}(x_s))^2}{2} = \frac{\lambda}{2}(f^{\star}(x_s) - y_s)^2 - \frac{\lambda}{2}(g_s - y_s)^2 + \frac{1}{2}(\lambda_s - \sigma^2\lambda^2)(g_s - f^{\star}(x_s))^2.
\end{align*}

Eq. (\ref{eqn:supermart}) can now be re-written as
\begin{equation}
\mathrm{ln}\left(\mathop{\mathbb{E}}_{\bs{g}_t \sim P_t}\left[\mathrm{exp}\left(\sum_{s=1}^{t}\frac{\lambda}{2}(f^{\star}(x_s) - y_s)^2 - \frac{\lambda_s}{2}(g_s - y_s)^2 + \frac{1}{2}(\lambda_s - \sigma^2\lambda_s^2)(g_s - f^{\star}(x_s))^2\right)\right]\right) \leq \ln(1/\delta).\label{eqn:pre_cancel}
\end{equation}

We set $\lambda_s \equiv 1/\sigma^2$. For this choice, we have $\lambda_s - \sigma^2\lambda_s^2 = 0$, and (\ref{eqn:pre_cancel}) becomes
\begin{equation}
\mathrm{ln}\left(\mathop{\mathbb{E}}_{\bs{g}_t \sim P_t}\left[\mathrm{exp}\left(\sum_{s=1}^{t}\frac{1}{2\sigma^2}(f^{\star}(x_s) - y_s)^2 - \frac{1}{2\sigma^2}(g_s - y_s)^2\right)\right]\right) \leq \ln(1/\delta).\label{eqn:pre_norm_notation}
\end{equation}

By defining $\bs{f}_{t}^{\star} := [f^{\star}(x_1), \dots, f^{\star}(x_t)]^{\top}$ and $\bs{y}_{t} := [y_1, \dots, y_t]^{\top}$, rearranging (\ref{eqn:pre_norm_notation}), and then writing the resulting inequality using squared norms rather than sums of squares, we obtain
\begin{equation*}
\Vert\bs{f}_{t}^{\star}-\bs{y}_{t}\Vert_2^2 \leq -2\sigma^2\ln\left(\mathop{\mathbb{E}}_{\bs{g}_{t} \sim P_t}\left[\mathrm{exp}\left(-\frac{1}{2\sigma^2}\Vert\bs{g}_{t}-\bs{y}_{t}\Vert_2^2\right)\right]\right) + 2\sigma^2\ln(1/\delta).
\end{equation*}

\subsection{Derivation of Eq. (\ref{eqn:radius_def})}
\label{app:radius_derivation}

Eq. (\ref{eqn:radius_def}) is derived from Eq. (\ref{eqn:pre_radius_def}) using the following Lemma, which is proved in App. B.2 of \citep{flynn2023improved}.

\begin{lemma}[\citep{flynn2023improved}]
For any $t$, any $\bs{\mu}_t$ and any positive semi-definite $\bs{T}_t$, we have
\begin{equation*}
\mathop{\mathbb{E}}_{\bs{g}_{t} \sim \mathcal{N}(\bs{\mu}, \bs{T}_t)}\left[\mathrm{exp}\left(-\frac{1}{2\sigma^2}\Vert\bs{g}_{t}-\bs{y}_{t}\Vert_2^2\right)\right] = \sqrt{\frac{1}{\det(\bs{I}+\bs{T}_t/\sigma^2)}}\exp\left(-\frac{1}{2\sigma^2}(\bs{\mu}_t - \bs{y}_{t})^{\top}\left(\bs{I}+\frac{\bs{T}_{t}}{\sigma^2}\right)^{-1}(\bs{\mu}_t - \bs{y}_{t})\right)
\end{equation*}
\label{lem:gaussian_int}
\end{lemma}

Combining Eq. (\ref{eqn:pre_radius_def}) and Lemma \ref{lem:gaussian_int}, we obtain
\begin{align}
\Vert\bs{f}_{t}^{\star}-\bs{y}_{t}\Vert_2^2 &\leq -2\sigma^2\ln\left(\mathop{\mathbb{E}}_{\bs{g}_{t} \sim \mathcal{N}(\bs{0}, c\bs{K}_t)}\left[\mathrm{exp}\left(-\frac{1}{2\sigma^2}\Vert\bs{g}_{t}-\bs{y}_{t}\Vert_2^2\right)\right]\right) + 2\sigma^2\ln(1/\delta).\label{eqn:general_radius}\\
&= \bs{y}_{t}^{\top}\left(\bs{I}+\frac{c\bs{K}_{t}}{\sigma^2}\right)^{-1}\bs{y}_{t} + \sigma^2\ln\left(\det\left(\bs{I}+\frac{c\bs{K}_{t}}{\sigma^2}\right)\right)+2\sigma^2\ln(1/\delta) = (R_t(\delta))^2.\nonumber
\end{align}

\section{Confidence Sequence and Confidence Bound Derivations}

In this section, we provide the full derivations of our confidence sequences and confidence bounds from Section \ref{sec:confidence_bounds}. Before proving our main results, we state and prove some useful lemmas. 

\subsection{Useful Lemmas}

Lemma \ref{lem:complete_square} allows us to express weighted sums of the quadratic constraints that define $\mathcal{F}_t$ (from Lemma \ref{lem:rkhs_conf_set}) as a single quadratic constraint.

\begin{lemma}
For any $\alpha > 0$ and any $f \in \mathcal{H}$,
\begin{equation*}
\Vert\bs{f}_t - \bs{y}_t\Vert_2^2 + \alpha\norm{f}_{\mathcal{H}}^2 = \norm{\left(\bs{V}_t + \alpha\bs{I}_{\mathcal{H}}\right)^{1/2}(f - \mu_{\alpha,t})}_{\mathcal{H}}^2 + \bs{y}_t^{\top}\bs{y}_t - \bs{y}_t^{\top}\Phi_t(\bs{V}_t + \alpha\bs{I}_{\mathcal{H}})^{-1}\Phi_t^{\top}\bs{y}_t.
\end{equation*}
\label{lem:complete_square}
\end{lemma}

\begin{proof}
$\Vert\bs{f}_t - \bs{y}_t\Vert_2^2 + \alpha\norm{f}_{\mathcal{H}}^2$ can be rewritten in the form $(f - b)^{\top}\bs{A}(f - b) + c$ by completing the square, where $\bs{A}$ is a self-adjoint (symmetric) linear operator, $b \in \mathcal{H}$ and $c \in \mathbb{R}$. Since $\bs{A}$ is self-adjoint, we have
\begin{equation*}
(f - b)^{\top}\bs{A}(f - b) + c = {f}^{\top}\bs{A}f -2b^{\top}\bs{A}f + b^{\top}\bs{A}b + c.
\end{equation*}

We also have
\begin{align*}
\Vert\bs{f}_t - \bs{y}_t\Vert_2^2 + \alpha\norm{f}_{\mathcal{H}}^2 &= (\Phi_tf - \bs{y}_t)^{\top}(\Phi_tf - \bs{y}_t) + \alpha {f}^{\top}f\\
&= {f}^{\top}(\bs{V}_t + \alpha\bs{I}_{\mathcal{H}})f - 2\bs{y}_t^{\top}\Phi_tf + \bs{y}_t^{\top}\bs{y}_t.
\end{align*}

By equating coefficients, we have
\begin{align*}
\bs{A} &= \bs{V}_t + \alpha\bs{I}_{\mathcal{H}},\\
b &= (\bs{V}_t + \alpha\bs{I}_{\mathcal{H}})^{-1}\Phi_t^{\top}\bs{y}_t = \Phi_t^{\top}(\bs{K}_t + \alpha\bs{I})^{-1}\bs{y}_t = \mu_{\alpha,t},\\
c &= \bs{y}_t^{\top}\bs{y}_t - \bs{y}_t^{\top}\Phi_t(\bs{V}_t + \alpha\bs{I}_{\mathcal{H}})^{-1}\Phi_t^{\top}\bs{y}_t.
\end{align*}
\end{proof}

Lemma \ref{lem:quad_opt} gives the maximum of certain concave quadratic functionals. We use it in the proof of Theorem \ref{thm:dual_ucb} to calculate the dual function associated with $\max_{f \in \mathcal{F}_t}\{f(x)\}$.

\begin{lemma}
For any functions $a, b \in \mathcal{H}$ and any self-adjoint (symmetric) positive-definite linear operator $\bs{A}$, we have
\begin{equation*}
\max_{f \in \mathcal{H}}\left\{a^{\top}f - (f - b)^{\top}\bs{A}(f - b)\right\} = a^{\top}b + \frac{1}{4}a^{\top}\bs{A}^{-1}a.
\end{equation*}
\label{lem:quad_opt}
\end{lemma}

\begin{proof}
Let $F(f) = a^{\top}f - (f - b)^{\top}\bs{A}(f - b)$. Since $F(f)$ is a concave quadratic functional, $F^{\prime}(\bar{f}) = 0$ is a sufficient condition for $\bar{f}$ to be a maximiser of $F(f)$, where $F^{\prime}(f)$ is the Fr\'{e}chet-derivative of $F$ at $f$. For any direction $g \in \mathcal{H}$ and any $s > 0$, we have
\begin{equation*}
\dd{}{}{s}F(f + sg) = a^{\top}g - 2sg^{\top}\bs{A}g - 2(f-b)^{\top}\bs{A}g.
\end{equation*}

Therefore, the directional derivative of $F$ (in the direction $g$) is
\begin{equation*}
\dd{}{}{s}F(f + sg)\bigg|_{s=0} = a^{\top}g - 2(f-b)^{\top}\bs{A}g.
\end{equation*}

This means that the Fr\'{e}chet-derivative of $F$ at $f$ is
\begin{equation*}
F^{\prime}(f) = a - 2\bs{A}(f-b).
\end{equation*}

There is a unique solution of $F^{\prime}(\bar{f}) = 0$, which is
\begin{equation*}
\bar{f} = b + \frac{1}{2}\bs{A}a.
\end{equation*}

The maximum is
\begin{equation*}
F(\bar{f}) = a^{\top}b + \frac{1}{4}a^{\top}\bs{A}^{-1}a.
\end{equation*}
\end{proof}

Lemma \ref{lem:kv_swap2} provides an alternative expression for $\rho_{\alpha, t}^2(x)$, which is already well-known.

\begin{lemma}
For any $\alpha > 0$,
\begin{equation*}
k(\cdot,x)^{\top}(\bs{V}_t + \alpha\bs{I}_{\mathcal{H}})^{-1}k(\cdot,x) = \frac{1}{\alpha}\left(k(x, x) - \bs{k}_t(x)^{\top}(\bs{K}_t + \alpha\bs{I})^{-1}\bs{k}_t(x)\right).
\end{equation*}
\label{lem:kv_swap2}
\end{lemma}

\begin{proof}
\begin{align*}
k(\cdot,x)^{\top}(\bs{V}_t + \alpha\bs{I}_{\mathcal{H}})^{-1}k(\cdot,x) &= \frac{1}{\alpha}k(\cdot,x)^{\top}(\bs{V}_t + \alpha\bs{I}_{\mathcal{H}})^{-1}(\alpha\bs{I}_{\mathcal{H}})k(\cdot,x)\\
&= \frac{1}{\alpha}k(\cdot,x)^{\top}(\bs{V}_t + \alpha\bs{I}_{\mathcal{H}})^{-1}(\Phi_t^{\top}\Phi_t + \alpha\bs{I}_{\mathcal{H}} - \Phi_t^{\top}\Phi_t)k(\cdot,x)\\
&= \frac{1}{\alpha}\left(k(\cdot, x)^{\top}\bs{I}_{\mathcal{H}}k(\cdot, x) - k(\cdot,x)^{\top}(\bs{V}_t + \alpha\bs{I}_{\mathcal{H}})^{-1}\Phi_t^{\top}\Phi_tk(\cdot,x)\right)\\
&= \frac{1}{\alpha}\left(k(\cdot, x)^{\top}\bs{I}_{\mathcal{H}}k(\cdot, x) - k(\cdot,x)^{\top}\Phi_t^{\top}(\bs{K}_t + \alpha\bs{I})^{-1}\Phi_tk(\cdot,x)\right)\\
&= \frac{1}{\alpha}\left(k(x, x) - \bs{k}_t(x)^{\top}(\bs{K}_t + \alpha\bs{I})^{-1}\bs{k}_t(x)\right).
\end{align*}
\end{proof}

\subsection{Proof of Corollary \ref{cor:rkhs_ridge_conf_set}}
\label{app:ridge_set}

\begin{proof}[Proof of Corollary \ref{cor:rkhs_ridge_conf_set}]
From $\norm{f^{\star}}_\mathcal{H} \leq B$ and Eq. (\ref{eqn:radius_def}), for any $\delta \in (0, 1]$, with probability at least $1 - \delta$, $f^{\star}$ satisfies
\begin{equation}
\forall t \geq 1, \alpha \geq 0: \quad \Vert\bs{f}_t^{\star} - \bs{y}_t\Vert_2^2 + \alpha\norm{f^{\star}}_{\mathcal{H}}^2 \leq R_{t}^2 + \alpha B^2.\label{eqn:sum_constraint}
\end{equation}

Using Lemma \ref{lem:complete_square}, the inequality in (\ref{eqn:sum_constraint}) can be rewritten as
\begin{align*}
\norm{\left(\bs{V}_t + \alpha\bs{I}_{\mathcal{H}}\right)^{1/2}(f^{\star} - \mu_{\alpha,t})}_{\mathcal{H}}^2 \leq R_{t}^2 + \alpha B^2 - \bs{y}_t^{\top}\bs{y}_t + \bs{y}_t^{\top}\Phi_t(\bs{V}_t + \alpha\bs{I}_{\mathcal{H}})^{-1}\Phi_t^{\top}\bs{y}_t.
\end{align*}

This means that $f^{\star} \in \widetilde{\mathcal{F}}_t$ for all $t \geq 1$ and $\alpha > 0$. From the definition of the confidence set $\mathcal{F}_t$ from Lemma \ref{lem:rkhs_conf_set}, every $f \in \mathcal{F}_t$ must satisfy $\Vert\bs{f}_t - \bs{y}_t\Vert_2^2 + \alpha\norm{f}_{\mathcal{H}}^2 \leq R_{t}^2 + \alpha B^2$, which means $f \in \widetilde{\mathcal{F}}_t$. Therefore, $\mathcal{F}_t \subseteq \widetilde{\mathcal{F}}_t$.
\end{proof}

\subsection{Proof of Theorem \ref{thm:ucb_representer}}
\label{app:representer_proof}

For convenience, we first restate some of the definitions from Section \ref{sec:implicit_conf}. $\bs{k}_{t+1}(x) = [k(x, x_1), \dots, k(x, x_t), k(x, x)]^{\top}$. $\bs{K}_{t, t+1}$ is the $t \times t+1$ kernel matrix with $i,j$\textsuperscript{th} element equal to $k(x_i, x_j)$ if $j < t+1$ and $k(x_i, x)$ otherwise. $\bs{L}_{t+1}$ is any matrix satisfying $\bs{L}_{t+1}^{\top}\bs{L}_{t+1} = \bs{K}_{t+1}$.

\begin{proof}[Proof of Theorem \ref{thm:ucb_representer}]
The optimisation problem $\max_{f \in \mathcal{F}_t}\{f(x)\}$ can be written as
\begin{equation}
\max_{f \in \mathcal{H}}\; f(x) \quad \text{s.t.} \quad \Vert\bs{f}_t - \bs{y}_t\Vert_2 \leq R_{t}, \quad \norm{f}_{\mathcal{H}} \leq B.\label{eqn:ucb_inf_cone_prog}
\end{equation}

First, we prove a statement that resembles a representer theorem \cite{kimeldorf1971some, scholkopf2001generalized}. We will show that the solution of (\ref{eqn:ucb_inf_cone_prog}) must be of the form $f(x) = \bs{k}_{t+1}(x)^{\top}\bs{w}$, for some weight vector $\bs{w} \in \mathbb{R}^{t+1}$.

All functions $f \in \mathcal{H}$ can be written in the form $f = f_{\parallel} + f_{\perp}$, where $f_{\parallel} = \bs{k}_{t+1}(\cdot)^{\top}\bs{w}$ is in the subspace of $\mathcal{H}$ spanned by $k(\cdot, x_1), \dots, k(\cdot, x_t), k(\cdot, x)$, and $f_{\perp}$ is orthogonal to this subspace. Since $f_{\perp}$ is orthogonal to the basis $k(\cdot, x_1), \dots, k(\cdot, x_t), k(\cdot, x)$, we have
\begin{equation*}
\norm{f}_{\mathcal{H}}^2 = \inner{f_{\parallel} + f_{\perp}}{f_{\parallel} + f_{\perp}}_{\mathcal{H}} = \inner{f_{\parallel}}{f_{\parallel}}_{\mathcal{H}} + 2\inner{f_{\parallel}}{f_{\perp}}_{\mathcal{H}} + \inner{f_{\perp}}{f_{\perp}}_{\mathcal{H}} = \norm{f_{\parallel}}_{\mathcal{H}}^2 + \norm{f_{\perp}}_{\mathcal{H}}^2,
\end{equation*}

This means that the RKHS norm of $f$ is minimised w.r.t. $f_{\perp}$ by choosing $f{\perp} \equiv 0$. Using the reproducing property of the kernel and the fact that $f_{\perp}$ is orthogonal to $k(\cdot, x_1), \dots, k(\cdot, x_t), k(\cdot, x)$, for any $z \in \{x_1, \dots, x_t, x\}$, we have
\begin{equation*}
f(z) = \inner{f}{k(\cdot, z)}_{\mathcal{H}} = \inner{f_{\parallel} + f_{\perp}}{k(\cdot, z)}_{\mathcal{H}} = \inner{f_{\parallel}}{k(\cdot, z)}_{\mathcal{H}} = f_{\perp}(z).
\end{equation*}

This means that the function values $f(x_1), \dots, f(x_t), f(x)$ are entirely determined by $f_{\parallel}$. The objective function is $f(x)$ and the constraint $\Vert\bs{f}_{t} - \bs{y}_{t}\Vert_2 \leq R_{t}$ depends only on $f(x_1), \dots, f(x_t)$. Therefore, any $f_{\perp} \neq 0$ will have no effect on the objective or this constraint, and will only increase the norm of $f$. This means that we must have $f_{\perp} \equiv 0$ at the solution of $\max_{f \in \mathcal{F}_t}\left\{f(x)\right\}$. We are now free to look for a solution of the form $f(x) = \bs{k}_{t+1}(x)^{\top}\bs{w}$, which means
\begin{align*}
\Vert\bs{f}_t - \bs{y}_t\Vert_2 &= \Vert\bs{K}_{t, t+1}\bs{w}- \bs{y}_t\Vert_2,\\
\norm{f}_{\mathcal{H}} &= \sqrt{\inner{\sum_{s=1}^{t}k(\cdot, x_s)w_s + k(\cdot,x)w_{t+1}}{\sum_{s^{\prime}=1}^{t}k(\cdot, x_{s^{\prime}})w_{s^{\prime}} + k(\cdot,x)w_{t+1}}_{\mathcal{H}}}\\
&= \sqrt{\bs{w}^{\top}\bs{K}_{t+1}\bs{w}} = \Vert\bs{L}_{t+1}\bs{w}\Vert_2.
\end{align*}

Substituting these expressions for the objective and constraints into (\ref{eqn:ucb_inf_cone_prog}) yields (\ref{eqn:ucb_cone_prog}).
\end{proof}

\subsection{Proof of Theorem \ref{thm:dual_ucb}}
\label{app:dual_ucb_proof}

The optimisation problem $\max_{f \in \mathcal{F}_t}\{f(x)\}$ can be written as
\begin{equation}
\max_{f \in \mathcal{H}}\; f(x) \quad \text{s.t.} \quad \Vert\bs{f}_t - \bs{y}_t\Vert_2^2 \leq R_{t}^2, \quad \norm{f}_{\mathcal{H}}^2 \leq B^2.\label{eqn:ucb_sum_square_prog}
\end{equation}

The Lagrangian for this problem is
\begin{equation*}
L(f, \eta_1, \eta_2) = f(x) + \eta_1(R_{t}^2 - \Vert\bs{f}_t - \bs{y}_t\Vert_2^2) + \eta_2(B^2 - \norm{f}_{\mathcal{H}}^2),
\end{equation*}

where $\eta_1, \eta_2$ are the Lagrange multipliers. The dual function is
\begin{equation*}
g(\eta_1, \eta_2) = \max_{f \in \mathcal{H}}\{L(f, \eta_1, \eta_2)\}.
\end{equation*}

By weak duality, the solution of (\ref{eqn:ucb_sum_square_prog}) is upper bounded by the solution of the dual problem, which is
\begin{equation*}
\min_{\eta_1 \geq 0, \eta_2 \geq 0}\{g(\eta_1, \eta_2)\}.
\end{equation*}

Since the Lagrangian $L(f, \eta_1, \eta_2)$ is linear in the Lagrange multipliers for every $f \in \mathcal{H}$, the dual function is the maximum over a set of linear functions, which is a convex function (see Eq. (3.7) in Sec. 3.2.3 of \citep{boyd2004convex}). This means that the dual problem is a convex program. We will now show that, for any $\eta_1 > 0$ and $\alpha >0$, there is a closed-form expression for the dual function $g(\eta_1, \alpha\eta_1)$. Using Lemma \ref{lem:complete_square}, the Lagrangian evaluated at the Lagrange multipliers $\eta_1$ and $\alpha\eta_1$ is
\begin{align*}
L(f, \eta_1, \alpha\eta_1) &= f(x) + \eta_1(R_{t}^2 - \Vert\bs{f}_t - \bs{y}_t\Vert_2^2 + \alpha B^2 - \alpha\norm{f}_{\mathcal{H}}^2)\\
&= f(x) + \eta_1\left(\widetilde{R}_{\alpha,t}^2 - \norm{\left(\bs{V}_t + \alpha\bs{I}_{\mathcal{H}}\right)^{1/2}(f - \mu_{\alpha,t})}_{\mathcal{H}}^2\right).
\end{align*}

Using Lemma \ref{lem:quad_opt} and Lemma \ref{lem:kv_swap2}, the dual function evaluated at $\eta_1$ and $\alpha\eta_1$ is
\begin{align*}
g(\eta_1, \alpha\eta_1) &= \max_{f \in \mathcal{H}}\{L(f, \eta_1, \alpha\eta_1)\}\\
&= \max_{f \in \mathcal{H}}\left\{k(\cdot,x)^{\top}f - \eta_1(f - \mu_{\alpha,t})^{\top}(\bs{V}_t + \alpha\bs{I}_{\mathcal{H}})(f - \mu_{\alpha,t})\right\} + \eta_1\widetilde{R}_{\alpha,t}^2\\
&= k(\cdot,x)^{\top}\mu_{\alpha,t} + \frac{1}{4\eta_1}k(\cdot,x)^{\top}(\bs{V}_t + \alpha\bs{I}_{\mathcal{H}})^{-1}k(\cdot,x) + \eta_1\widetilde{R}_{\alpha,t}^2\\
&= \bs{k}_t(x)^{\top}(\bs{K}_t + \alpha\bs{I})^{-1}\bs{y}_t + \frac{k(x, x) - \bs{k}_t(x)^{\top}(\bs{K}_t + \alpha\bs{I})^{-1}\bs{k}_t(x)}{4\eta_1\alpha} + \eta_1\widetilde{R}_{\alpha,t}^2.
\end{align*}

Therefore, the dual problem is
\begin{equation*}
\min_{\eta_1 > 0, \alpha > 0}\left\{\bs{k}_t(x)^{\top}(\bs{K}_t + \alpha\bs{I})^{-1}\bs{y}_t + \frac{k(x, x) - \bs{k}_t(x)^{\top}(\bs{K}_t + \alpha\bs{I})^{-1}\bs{k}_t(x)}{4\eta_1\alpha} + \eta_1\widetilde{R}_{\alpha,t}^2\right\}.
\end{equation*}

Using the substitution $\eta_2 = \alpha\eta_1$, we arrive at the dual problem stated in (\ref{eqn:dual_ucb}). By minimising this objective function w.r.t. $\eta_1$ for any fixed $\alpha$, we can also express the dual problem as
\begin{equation*}
\min_{\alpha > 0}\left\{\bs{k}_t(x)^{\top}(\bs{K}_t + \alpha\bs{I})^{-1}\bs{y}_t + \frac{\widetilde{R}_{\alpha,t}}{\sqrt{\alpha}}\sqrt{k(x, x) - \bs{k}_t(x)^{\top}(\bs{K}_t + \alpha\bs{I})^{-1}\bs{k}_t(x)}\right\}.
\end{equation*}

%
%

\section{Confidence Bound Tightness}

In this section, we prove Theorem \ref{thm:confidence_tightness}. First, we state and prove some useful lemmas.

\subsection{Useful Lemmas}

Lemma \ref{lem:quad_stuff} gives us alternative expressions for the quadratic terms in $\widetilde{R}_{\alpha,t}$, which we use to prove Lemma \ref{lem:special_rad_bound}.

\begin{lemma}\label{lem:quad_stuff}
For any $\alpha > 0$, we have
\begin{equation*}
\bs{y}_t^{\top}\bs{y}_t - \bs{y}_t^{\top}\Phi_t\left(\bs{V}_t + \alpha\bs{I}_{\mathcal{H}}\right)^{-1}\Phi_t^{\top}\bs{y}_t = \bs{y}_t^{\top}\left(\frac{1}{\alpha}\bs{K}_t + \bs{I}\right)^{-1}\bs{y}_t.
\end{equation*}
\end{lemma}

\begin{proof}
We start with the identity
\begin{equation*}
\Phi_t\left(\bs{V}_t + \alpha\bs{I}_{\mathcal{H}}\right) = \left(\bs{K}_t + \alpha\bs{I}\right)\Phi_t.
\end{equation*}

By post-multiplying both sides with $(\bs{V}_t + \alpha\bs{I}_{\mathcal{H}})^{-1}$ and pre-multiplying both sides with $(\bs{K}_t + \alpha\bs{I})^{-1}$, we obtain
\begin{equation}\label{eqn:inv_swap}
\left(\bs{K}_t + \alpha\bs{I}\right)^{-1}\Phi_t = \Phi_t\left(\bs{V}_t + \alpha\bs{I}_{\mathcal{H}}\right)^{-1}.
\end{equation}

Now, using (\ref{eqn:inv_swap}), we have
\begin{align*}
\bs{y}_t^{\top}\bs{y}_t - \bs{y}_t^{\top}\Phi_t\left(\bs{V}_t + \alpha\bs{I}_{\mathcal{H}}\right)^{-1}\Phi_t^{\top}\bs{y}_t &= \bs{y}_t^{\top}\bs{y}_t - \bs{y}_t^{\top}\left(\bs{K}_t + \alpha\bs{I}\right)^{-1}\bs{K}_t\bs{y}_t\\
&= \bs{y}_t^{\top}\bs{y}_t - \bs{y}_t^{\top}\left(\bs{K}_t + \alpha\bs{I}\right)^{-1}\left(\bs{K}_t + \alpha\bs{I} - \alpha\bs{I}\right)\bs{y}_t\\
&= \bs{y}_t^{\top}\bs{y}_t - \bs{y}_t^{\top}\bs{y}_t + \alpha\bs{y}_t^{\top}\left(\bs{K}_t + \alpha\bs{I}\right)^{-1}\bs{y}_t\\
&= \bs{y}_t^{\top}\left(\frac{1}{\alpha}\bs{K}_t + \bs{I}\right)^{-1}\bs{y}_t.
\end{align*}
\end{proof}

Lemma \ref{lem:special_rad_bound} gives us a simplified expression for $\widetilde{R}_{\alpha,t}$ when $\alpha = \sigma^2/c$. This allows us to compare $\widetilde{R}_{\alpha,t}$ with the equivalent radius quantities of the confidence bounds in \citep{abbasi2012online} and \citep{chowdhury2017kernelized}.

\begin{lemma}
For any $c > 0$, set $\alpha = \sigma^2/c$. We have
\begin{equation*}
\widetilde{R}_{\alpha,t} = \sqrt{\sigma^2\ln\left(\det\left(\tfrac{c}{\sigma^2}\bs{K}_t + \bs{I}\right)\right) + \tfrac{\sigma^2B^2}{c} + 2\sigma^2\ln(\tfrac{1}{\delta})}.
\end{equation*}
\label{lem:special_rad_bound}
\end{lemma}

\begin{proof}
From Lemma \ref{lem:quad_stuff}, with this $c$ and $\alpha$, the squared radius is
\begin{align}
\widetilde{R}_{\alpha,t}^2 &= \bs{y}_t^{\top}\left(\frac{c}{\sigma^2}\bs{K}_t + \bs{I}\right)^{-1}\bs{y}_t - \bs{y}_t^{\top}\bs{y}_t + \bs{y}_t^{\top}\Phi_t\left(\bs{V}_t + \frac{\sigma^2}{c}\bs{I}_{\mathcal{H}}\right)^{-1}\Phi_t^{\top}\bs{y}_t\label{eqn:special_radius}\\
&+ \sigma^2\ln\left(\det\left(\frac{c}{\sigma^2}\bs{K}_t + \bs{I}\right)\right) + 2\sigma^2\ln(1/\delta) + \frac{\sigma^2B^2}{c}\nonumber\\
&= \sigma^2\ln\left(\det\left(\frac{c}{\sigma^2}\bs{K}_t + \bs{I}\right)\right) + 2\sigma^2\ln(1/\delta) + \frac{\sigma^2B^2}{c}.\nonumber
\end{align}
\end{proof}

\subsection{Proof of Theorem \ref{thm:confidence_tightness}}
\label{app:confidence_tightness}

\begin{proof}[Proof of Theorem \ref{thm:confidence_tightness}]
For any $\delta \in (0, 1]$ and any $\lambda > 0$, the UCB from Thm. 3.11 of \citep{abbasi2012online} states that, with probability at least $1 - \delta$,
\begin{align*}
\forall t \geq 1: \quad f^{\star}(x) \leq \mu_{\lambda,t}(x) + \frac{\widetilde{R}_{\lambda,t}^{\mathrm{AY}}}{\sqrt{\lambda}}\rho_{\lambda,t}(x),
\end{align*}

where
\begin{equation*}
\widetilde{R}_{\lambda,t}^{\mathrm{AY}} =  \sigma\sqrt{\ln(\det(\tfrac{1}{\lambda}\bs{K}_t + \bs{I})) + 2\ln(1/\delta)} + \sqrt{\lambda}B.
\end{equation*}

We will now show that when we set the covariance scale to $c = \sigma^2/\lambda$ and set $\alpha = \sigma^2/c = \lambda$, the RHS of (\ref{eqn:analytic_ucb}) is strictly less than this UCB. Starting from the RHS of (\ref{eqn:analytic_ucb}) with $c = \sigma^2/\lambda$ and $\alpha = \sigma^2/c = \lambda$, and then using Lemma \ref{lem:special_rad_bound} and the inequality $\sqrt{a + b} < \sqrt{a} + \sqrt{b}$ (for $a, b > 0$), we have
\begin{align*}
\mu_{\lambda,t}(x) + \frac{\widetilde{R}_{\lambda,t}}{\sqrt{\lambda}}\rho_{\lambda,t}(x) &= \mu_{\lambda,t}(x) + \frac{\sqrt{\sigma^2\ln\left(\det\left(\frac{1}{\lambda}\bs{K}_t + \bs{I}\right)\right) + 2\sigma^2\ln(1/\delta) + \lambda B^2}}{\sqrt{\lambda}}\rho_{\lambda,t}(x)\\
&< \mu_{\lambda,t}(x) + \frac{\sigma\sqrt{\ln\left(\det\left(\frac{1}{\lambda}\bs{K}_t + \bs{I}\right)\right) + 2\ln(1/\delta)} + \sqrt{\lambda}B}{\sqrt{\lambda}}\rho_{\lambda,t}(x)\\
&= \mu_{\lambda,t}(x) + \frac{\widetilde{R}_{\lambda,t}^{\mathrm{AY}}}{\sqrt{\lambda}}\rho_{\lambda,t}(x).
\end{align*}

For any $\delta \in (0, 1]$ and any $\eta > 0$, the UCB from Thm. 2 of \citep{chowdhury2017kernelized} states that, with probability at least $1 - \delta$,
\begin{align*}
\forall t \geq 1: \quad f^{\star}(x) \leq \mu_{1+\eta,t}(x) + \widetilde{R}_{\eta,t}^{\mathrm{IGP}}\rho_{1+\eta,t}(x),
\end{align*}

where
\begin{equation*}
\widetilde{R}_{\eta,t}^{\mathrm{IGP}} = \sigma\sqrt{\ln(\det(\tfrac{1}{1+\eta}\bs{K}_t + \bs{I})) + t\eta + 2\ln(1/\delta)} + B.
\end{equation*}

If the number of rounds $T$ is known in advance, then one can set $\eta$ to the recommended value of $2/T$, and then the $t\eta$ term in $\widetilde{R}_{\eta,t}^{\mathrm{IGP}}$ is bounded by 2 for all $t \in [T]$. We compare confidence bounds for a general value of $\eta > 0$. Starting from the RHS of (\ref{eqn:analytic_ucb}) with $c = \sigma^2/(1 + \eta)$ and $\alpha = \sigma^2/c = (1 + \eta)$, we have
\begin{align*}
\mu_{1+\eta,t}(x) + \frac{\widetilde{R}_{1+\eta,t}}{\sqrt{1 + \eta}}\rho_{1+\eta,t}(x) &= \mu_{1+\eta,t}(x) + \frac{\sqrt{\sigma^2\ln\left(\det\left(\frac{1}{1 + \eta}\bs{K}_t + \bs{I}\right)\right) + 2\sigma^2\ln(1/\delta) + (1 + \eta) B^2}}{\sqrt{1 + \eta}}\rho_{1+\eta,t}(x)\\
&< \mu_{1+\eta,t}(x) + \left(\sigma\sqrt{\ln\left(\det\left(\frac{1}{1 + \eta}\bs{K}_t + \bs{I}\right)\right) + 2\ln(1/\delta)} + B\right)\rho_{1+\eta,t}(x)\\
&< \mu_{1+\eta,t}(x) + \widetilde{R}_{\eta,t}^{\mathrm{IGP}}\rho_{1+\eta,t}(x).
\end{align*}
\end{proof}

\subsection{Why Are Our Confidence Bounds Tighter?}

Here, we explain why our analytic confidence bounds in Cor. \ref{cor:analytic_ucb} are tighter than the confidence bounds in Thm. 3.11 of \citep{abbasi2012online}. The same reasoning explains why our analytic confidence bounds are tighter than those from Thm. 2 of \citep{chowdhury2017kernelized}. Both our analytic confidence bounds and the confidence bounds in \citep{abbasi2012online} are derived from a constraint of the form
\begin{equation*}
\norm{\left(\bs{V}_t + \alpha\bs{I}_{\mathcal{H}}\right)^{1/2}(f^{\star} - \mu_{\alpha, t})}_{\mathcal{H}} \leq R.
\end{equation*}

The difference between the confidence bounds is in the radius quantity $R$. The confidence bounds in \citep{abbasi2012online} use a self-normalised concentration inequality (Thm. 3.4 in \citealp{abbasi2012online}), which states that, with high probability,
\begin{equation*}
\norm{\left(\bs{V}_t + \alpha\bs{I}_{\mathcal{H}}\right)^{-1/2}\Phi_t\bs{\epsilon}_t}_{\mathcal{H}} \leq \sigma\sqrt{\ln\left(\det\left(\frac{1}{\alpha}\bs{K}_t + \bs{I}\right)\right) + 2\ln(1/\delta)}.
\end{equation*}

Using the definition of $\mu_{\alpha, t}$ and then the triangle inequality, it is shown in \citep{abbasi2012online} that
\begin{align*}
\norm{\left(\bs{V}_t + \alpha\bs{I}_{\mathcal{H}}\right)^{1/2}(f^{\star} - \mu_{\alpha, t})}_{\mathcal{H}} &= \norm{\left(\bs{V}_t + \alpha\bs{I}_{\mathcal{H}}\right)^{-1/2}(\Phi_t\bs{\epsilon}_t - \alpha f^{\star})}_{\mathcal{H}}\\
&\leq \norm{\left(\bs{V}_t + \alpha\bs{I}_{\mathcal{H}}\right)^{-1/2}\Phi_t\bs{\epsilon}_t}_{\mathcal{H}} + \alpha\norm{\left(\bs{V}_t + \alpha\bs{I}_{\mathcal{H}}\right)^{-1/2}f^{\star}}_{\mathcal{H}}\\
&\leq \sigma\sqrt{\ln\left(\det\left(\frac{1}{\alpha}\bs{K}_t + \bs{I}\right)\right) + 2\ln(1/\delta)} + \sqrt{\alpha}B.
\end{align*}

This gives us the expression for the radius in \citep{abbasi2012online}. The triangle inequality step causes the log determinant and the RKHS norm to appear under separate square roots. Instead of the self-normalised concentration inequality, we use a bound on $\|\bs{f}_t^{\star} - \bs{y}_t\|_2^2$ (Eq. (\ref{eqn:radius_def})), which can be combined with the constraint $\|f^{\star}\|_{\mathcal{H}} \leq B$ more directly via completing the square (Cor. \ref{cor:rkhs_ridge_conf_set}) or Lagrangian duality.

\section{Regret Analysis}

In this section, we prove the cumulative regret bound in Theorem \ref{thm:regret_bound}. In addition, we prove cumulative regret bounds with explicit dependence on $T$ for the noisy (Sec. \ref{app:eig_regret}) and noiseless (Sec. \ref{app:regret_noiseless}) settings. We use upper bounds on $\gamma_T(\sigma^2/c)$ from \citep{vakili2021information} and derive our own upper bounds on $\tau_T(\sigma^2/c)$. The bounds on both of these quantities apply when the kernel $k$ is a Mercer kernel, and their rates depend on how quickly the eigenvalues of the kernel decay to zero. We therefore begin by describing Mercer kernels, as well as the polynomial and exponential eigendecay conditions.

\subsection{Mercer Kernels and Eigendecay Conditions}

A positive-definite kernel $k$ on a set $\mathcal{X}$ is called a Mercer kernel if: a) $\mathcal{X}$ is a compact metric space (e.g. a compact subset of $\mathbb{R}^d$); b) the function $k: \mathcal{X} \times \mathcal{X} \to \mathbb{R}$ is continuous. Mercer's theorem provides a useful representation for Mercer kernels. Let $\rho$ be a non-degenerate measure on $\mathcal{X}$ and let $L^2(\mathcal{X}, \rho)$ denote the set of square integrable functions on $\mathcal{X}$, i.e.
\begin{equation*}
L^2(\mathcal{X}, \rho) := \left\{f:\mathcal{X} \to \mathbb{R} : \int_{\mathcal{X}}(f(x))^2\mathrm{d}\rho(x) < \infty\right\}.
\end{equation*} 

Define the linear operator $L_k : L^2(\mathcal{X}, \rho) \to L^2(\mathcal{X}, \rho)$ as
\begin{equation*}
L_k(f)(x) := \int_{\mathcal{X}}k(x, y)f(y)\mathrm{d}\rho(y).
\end{equation*}

\begin{theorem}[Mercer's Theorem]
If $k: \mathcal{X} \times \mathcal{X} \to \mathbb{R}$ is a Mercer kernel, then there exist non-negative eigenvalues $\xi_1 \geq \xi_2 \geq \cdots \geq 0$ and corresponding eigenfunctions $\phi_1, \phi_2, \dots$, such that
\begin{equation}
L_k(\phi_m) = \xi_m\phi_m, \quad \text{for all } m = 1, 2, \dots.
\end{equation}

In addition, the kernel function has the eigendecomposition
\begin{equation}
k(x, y) = \sum_{m=1}^{\infty}\xi_m\phi_m(x)\phi_m(y),
\end{equation}

where the convergence of the infinite series is absolute for each $x, y \in \mathcal{X}$ and uniform on $\mathcal{X} \times \mathcal{X}$.
\end{theorem}

The scaled eigenfunctions $\sqrt{\xi_1}\phi_1, \sqrt{\xi_2}\phi_2, \dots$ form a (countable) orthonormal basis of the RKHS $\mathcal{H}$. Following \citet{vakili2021information}, we make the following additional assumption about the kernel $k$: $\forall m \in \mathbb{N}, \forall x \in \mathcal{X}, |\phi_m(x)| \leq \psi$, for some $\psi > 0$. The rate at which the eigenvalues $\xi_1, \xi_2, \dots$ of the kernel decay to zero determines both the complexity of the corresponding RKHS and the hardness of regret minmisation in kernel bandits. The two most commonly studied eigendecay conditions are defined as follows.

\begin{definition}[Eigenvalue decay]
Consider the eigenvalues $\xi_1, \xi_2, \dots$ of a Mercer kernel $k(x,x^{\prime}) = \sum_{i=1}^{\infty}\xi_i\phi_i(x)\phi_i(x^{\prime})$. The polynomial eigenvalue decay condition is satisfied if there exist $C_p > 0$ and $\beta_p > 1$ such that
\begin{equation*}
\xi_i \leq C_pi^{-\beta_p}.
\end{equation*}

The exponential eigenvalue decay condition is satisfied if there exist $C_{e_1}, C_{e_2} > 0$ and $\beta_{e} \in (0, 1]$ such that
\begin{equation*}
\xi_i \leq C_{e_1} \exp(-C_{e_2}i^{\beta_e}).
\end{equation*}
\end{definition}

If a kernel is known to satisfy one of these eigendecay conditions, then the maximum information gain can be bounded. In particular, \citet{vakili2021information} proved the following bounds.

\begin{lemma}[Corollary 1 in \citep{vakili2021information}]
If the kernel $k$ satisfies the polynomial eigenvalue decay condition, then
\begin{equation}
\sup_{x_1, \dots, x_T}\left\{\frac{1}{2}\ln\det\left(\frac{1}{\alpha}\bs{K}_T + \bs{I}\right)\right\} \leq \ln(1 + \kappa T/\alpha) + \left(\frac{C_p\psi^2T}{\alpha}\right)^{1/\beta_p}\ln(1 + \kappa T/\alpha)^{1 - 1/\beta_p}.\label{eqn:poly_info_gain}
\end{equation}

If the kernel $k$ satisfies the exponential eigenvalue decay condition, then
\begin{equation}
\sup_{x_1, \dots, x_T}\left\{\frac{1}{2}\ln\det\left(\frac{1}{\alpha}\bs{K}_T + \bs{I}\right)\right\} \leq \left(\left(\frac{2}{C_{e_2}}(\ln(T) + C_{\beta})\right)^{\frac{1}{\beta_e}} + 1\right)\ln(1 + \kappa T/\alpha),\label{eqn:exp_info_gain}
\end{equation}

where $C_{\beta} = \ln(\frac{C_{e_1}\psi^2}{\alpha C_{e_2}})$ if $\beta_e = 1$, and $C_{\beta} = \ln(\frac{2C_{e_1}\psi^2}{\alpha \beta_eC_{e_2}}) + (1/\beta_e - 1)(\ln(\frac{2}{C_{e_2}}(1/\beta_e - 1)) - 1)$ if $\beta_e < 1$.
\label{lem:info_gain}
\end{lemma}

\subsection{Elliptical Potential Lemmas}

In this subsection, we recall the well-known Elliptical Potential Lemma (see e.g. Lemma 11.11 in \citealt{cesa2006prediction}, Lemma 5.3 in \citealt{srinivas2010gaussian} or Lemma 11 in \citealt{abbasi2011improved}), which states that the sum of the elliptical potentials $\sum_{t=1}^{T}\frac{1}{\alpha}\rho_{\alpha,t-1}^2(x_t)$ can be upper bounded by the maximum information gain. Then, for each eigendecay condition, we state and prove our bounds on the elliptical potential count $\tau_T(\alpha)$.

The version of the Elliptical Potential Lemma stated in Lemma \ref{lem:elliptical} is essentially the same as the version in Lemma 5 in \citep{whitehouse2023sublinear}, so we omit the proof. The only real difference is that we start with $\sum_{t=1}^{T}\min\left(1, \frac{1}{\alpha}\rho_{\alpha,t-1}^2(x_t)\right)$ instead of $\sum_{t=1}^{T}\frac{1}{\alpha}\rho_{\alpha,t-1}^2(x_t)$. This change means that Lemma \ref{lem:elliptical} holds for all $\alpha > 0$, rather than for $\alpha \geq \max(1, \kappa)$, where $\kappa$ is an upper bound on $\sup_{x, x^{\prime}}|k(x, x^{\prime})|$.

\begin{lemma}[Elliptical Potential Lemma]
For any $T \geq 1$ and any $\alpha > 0$,
\begin{equation*}
\sum_{t=1}^{T}\min\left(1, \frac{1}{\alpha}\rho_{\alpha,t-1}^2(x_t)\right) \leq \frac{3}{2}\ln\left(\det\left(\frac{1}{\alpha}\bs{K}_t + \bs{I}\right)\right).
\end{equation*}
\label{lem:elliptical}
\end{lemma}

Now, we turn our attention to the elliptical potential count $\tau_T(\alpha)$. In the linear/finite-dimensional setting, there exist bounds on $\tau_T(\alpha)$. \citet{lattimore2020bandit} (Exercise 19.3) and \citet{gales2022norm} (Lemma 7) show that if the kernel is $k(x, x^{\prime}) = x^{\top}x^{\prime}$, then $\tau_T(\alpha) = \mathcal{O}(d\ln(1/\alpha))$. We follow (as closely as we can) the proof used in \citep{gales2022norm}, which uses the following lemma.

\begin{lemma}[Lemma 8 from \citep{gales2022norm}]
Let $X, A, B \geq 0$. If $X \leq A\ln(1 + BX)$, then
\begin{equation*}
X \leq \inf_{\eta \in (0, 1)}\left\{\frac{A}{1 - \eta}\ln\left(\frac{A}{e\eta}\left(\frac{1}{X} + B\right)\right)\right\}.
\end{equation*}
\label{lem:gales}
\end{lemma}

By choosing $\eta = 1/e$, we get the following corollary, which comes from the proof of Lemma 7 in \citep{gales2022norm}.

\begin{corollary}
Let $X, A, B \geq 0$. If $X \leq A\ln(1 + BX)$, then
\begin{equation*}
X \leq \frac{eA}{e-1}\ln(1 + AB).
\end{equation*}
\label{cor:gales}
\end{corollary}

\begin{proof}
From Lemma \ref{lem:gales} with $\eta = 1/e$, we have
\begin{equation*}
X \leq \frac{eA}{e-1}\ln\left(A\left(\frac{1}{X} + B\right)\right).
\end{equation*}

On the one hand, if $1/X \leq 1/A$, then
\begin{equation*}
X \leq \frac{eA}{e-1}\ln(1 + AB).
\end{equation*}

On the other hand, if $1/X \geq 1/A$, then the inequality $X \leq A\ln(1 + BX)$ tells us that
\begin{equation*}
X \leq A\ln(1 + AB).
\end{equation*}
\end{proof}

We also use the following variant of Lemma \ref{lem:gales}. This inequality is stated just after Lemma 8 \citep{gales2022norm}. For completeness, we give a proof.

\begin{lemma}
Let $X, A, B \geq 0$. If $X \leq A\ln(BX)$, then
\begin{equation*}
X \leq \inf_{\eta \in (0, 1)}\left\{\frac{1}{1 - \eta}\left(A\ln\left(\frac{AB}{e\eta}\right)\right)\right\}.
\end{equation*}
\label{lem:gales_3}
\end{lemma}

\begin{proof}
\begin{align*}
X &\leq A\ln(BX)\\
&= A\ln\left(\frac{\eta X}{A}\right) + A\ln\left(\frac{AB}{\eta}\right)\\
&\leq \eta X - A\ln(e) + A\ln\left(\frac{AB}{\eta}\right)\\
&= \eta X + A\ln\left(\frac{AB}{e\eta}\right).
\end{align*}
\end{proof}

Finally, we require the following lemma, which gives an expression for $\det\left(\frac{1}{\alpha}\bs{K}_T + \bs{I}\right)$ whenever the RKHS $\mathcal{H}$ is separable.

\begin{lemma}[Lemma E.1 in \citep{abbasi2012online}]
If $\mathcal{H}$ is separable, then
\begin{equation*}
\det\left(\frac{1}{\alpha}\bs{K}_T + \bs{I}\right) = \prod_{t=1}^{T}(1 + \|k(\cdot, x_t)\|_{(\bs{V}_{t-1} + \alpha\bs{I}_{\mathcal{H}})^{-1}}^2).
\end{equation*}
\label{lem:e1_logdet}
\end{lemma}

Since a Hilbert space is separable if and only if it has a countable orthonormal basis, the RKHS of any Mercer kernel is separable. The next lemma gives an intermediate result that can be specialised to either of the eigendecay conditions.

\begin{lemma}[Elliptical potentials: You cannot have more than $\mathcal{O}(\gamma_T)$ big intervals]
Let $b > 0$, let $x_1, \dots, x_T$ be any sequence of elements in $\mathcal{X}$ and let $k(x,x^{\prime})$ be any positive-definite kernel. Let $\mathcal{T}_T(\alpha) = \{t \in [T]: \|k(\cdot, x_t)\|_{(\bs{V}_{t-1} + \alpha\bs{I}_{\mathcal{H}})^{-1}} \geq b\}$. For every $T \geq 1$,
\begin{equation}
|\mathcal{T}_T(\alpha)| \leq \frac{2\gamma_{|\mathcal{T}_T(\alpha)|}(\alpha)}{\ln(1 + b^2)} \leq \frac{2\gamma_{T}(\alpha)}{\ln(1 + b^2)}.\label{eqn:generic_epc}
\end{equation}
\label{lem:ep_info_gain}
\end{lemma}

Using Lemma \ref{lem:kv_swap2}, we have $\|k(\cdot, x_t)\|_{(\bs{V}_{t-1} + \alpha\bs{I}_{\mathcal{H}})^{-1}} = \frac{1}{\sqrt{\alpha}}\rho_{\alpha, t-1}(x_t)$, which means the set $\mathcal{T}_T(\alpha)$ defined here is the same as the one in Eq. (\ref{eqn:tau_set}), except here we allow $b$ to be any positive number, not just 1.

\begin{proof}
Define $\bs{G}_T = \sum_{t=1}^{T}\mathbb{I}\{t \in \mathcal{T}_T(\alpha)\}k(\cdot,x_t)k(\cdot,x_t)^{\top}$ and let $\bs{C}_T$ be the $|\mathcal{T}_T(\alpha)| \times |\mathcal{T}_T(\alpha)|$ kernel matrix constructed from the points $\{x_t\}_{t \in \mathcal{T}_T(\alpha)}$. Using Lemma \ref{lem:e1_logdet}, we have
\begin{align*}
\det\left(\frac{1}{\alpha}\bs{C}_T + \bs{I}\right) &= \prod_{t \in \mathcal{T}_T(\alpha)}(1 + \|k(\cdot, x_t)\|_{(\bs{G}_{t-1} + \alpha\bs{I}_{\mathcal{H}})^{-1}}^2)\\
&\geq \prod_{t \in \mathcal{T}_T(\alpha)}(1 + \|k(\cdot, x_t)\|_{(\bs{V}_{t-1} + \alpha\bs{I}_{\mathcal{H}})^{-1}}^2)\\
&\geq (1 + b^2)^{|\mathcal{T}_T(\alpha)|}.
\end{align*}

Taking logarithms of both sides finishes the proof.
\end{proof}

Relying on the second inequality in (\ref{eqn:generic_epc}) can give a rather loose dependence on $T$. From the first inequality in (\ref{eqn:generic_epc}), we can derive tighter bounds on $|\mathcal{T}_T(\alpha)|$ when the kernel satisfies either the polynomial or exponential eigendecay conditions. First, we consider polynomial eigendecay.

\begin{lemma}[Elliptical potentials: polynomial eigendecay]
Let $b > 0$ and let $x_1, \dots, x_T$ be any sequence of elements in $\mathcal{X}$. Let $k$ be a Mercer kernel that satisfies the polynomial eigendecay condition and let $\mathcal{T}_T(\alpha)$ be the set defined in Lemma \ref{lem:ep_info_gain}. For all $T \geq 1$ and all $\alpha > 0$,
\begin{align*}
|\mathcal{T}_T(\alpha)| \leq \max\bigg(&\frac{4e}{(e-1)\ln(1 + b^2)}\ln\left(1 + \frac{4\kappa}{\alpha\ln(1 + b^2)}\right),\\
&\left(\frac{4e}{(e-1)\ln(1 + b^2)}\right)^{\frac{\beta_p}{\beta_p - 1}}\left(\frac{C_p\psi^2}{\alpha}\right)^{\frac{1}{\beta_p - 1}}\ln\left(1 + \frac{\kappa
}{\alpha}\left(\frac{4}{\ln(1 + b^2)}\right)^{\frac{\beta_p}{\beta_p - 1}}\left(\frac{C_p\psi^2}{\alpha}\right)^{\frac{1}{\beta_p - 1}}\right)\bigg).
\end{align*}
\label{lem:ep_poly}
\end{lemma}

\begin{proof}
From Lemma \ref{lem:ep_info_gain}, we have $|\mathcal{T}_T(\alpha)| \leq \frac{2}{\ln(1 + b^2)}\gamma_{|\mathcal{T}_T(\alpha)|}(\alpha)$. The bound in Lemma \ref{lem:info_gain} tells us that
\begin{equation}
|\mathcal{T}_T(\alpha)| \leq  \frac{2}{\ln(1 + b^2)}\left(\left(\frac{C_p\psi^2|\mathcal{T}_T(\alpha)|}{\alpha}\right)^{1/\beta_p}\ln(1 + \kappa |\mathcal{T}_T(\alpha)|/\alpha)^{-1/\beta_p} + 1\right)\ln(1 + \kappa |\mathcal{T}_T(\alpha)|/\alpha).\label{eqn:poly_ep_pt_1}
\end{equation}

We consider two cases. If $\frac{C_p\psi^2|\mathcal{T}_T(\alpha)|}{\alpha \ln(1 + \kappa |\mathcal{T}_T(\alpha)|/\alpha)} \leq 1$, then
\begin{equation}
|\mathcal{T}_T(\alpha)| \leq \frac{4}{\ln(1 + b^2)}\ln(1 + \kappa |\mathcal{T}_T(\alpha)|/\alpha).\label{eqn:poly_ep_case_1_pt_1}
\end{equation}

Using Corollary \ref{cor:gales}, with $A = \frac{4}{\ln(1 + b^2)}$, $B = \kappa/\alpha$ and $X = |\mathcal{T}_T(\alpha)|$, we obtain
\begin{equation*}
|\mathcal{T}_T(\alpha)| \leq \frac{4e}{(e-1)\ln(1 + b^2)}\ln\left(1 + \frac{4\kappa}{\alpha\ln(1 + b^2)}\right).
\end{equation*}





If $\frac{C_p\psi^2|\mathcal{T}_T(\alpha)|}{\alpha \ln(1 + \kappa |\mathcal{T}_T(\alpha)|/\alpha)} \geq 1$, then from (\ref{eqn:poly_ep_pt_1}), we have
\begin{equation*}
|\mathcal{T}_T(\alpha)| \leq \frac{4}{\ln(1 + b^2)}\left(\frac{C_p\psi^2|\mathcal{T}_T(\alpha)|}{\alpha}\right)^{1/\beta_p}\ln(1 + \kappa |\mathcal{T}_T(\alpha)|/\alpha)^{1-1/\beta_p}.
\end{equation*}

After some algebra, we obtain
\begin{equation*}
|\mathcal{T}_T(\alpha)| \leq \left(\frac{4}{\ln(1 + b^2)}\right)^{\frac{\beta_p}{\beta_p - 1}}\left(\frac{C_p\psi^2}{\alpha}\right)^{\frac{1}{\beta_p - 1}}\ln(1 + \kappa |\mathcal{T}_T(\alpha)|/\alpha).
\end{equation*}

Using Corollary \ref{cor:gales}, with $A = \left(\frac{4}{\ln(1 + b^2)}\right)^{\beta_p/(\beta_p - 1)}\left(\frac{C_p\psi^2}{\alpha}\right)^{1/(\beta_p - 1)}$, $B = \kappa/\alpha$ and $X = |\mathcal{T}_T(\alpha)|$, we obtain
\begin{equation*}
|\mathcal{T}_T(\alpha)| \leq \left(\frac{4e}{(e-1)\ln(1 + b^2)}\right)^{\frac{\beta_p}{\beta_p - 1}}\left(\frac{C_p\psi^2}{\alpha}\right)^{\frac{1}{\beta_p - 1}}\ln\left(1 + \frac{\kappa
}{\alpha}\left(\frac{4}{\ln(1 + b^2)}\right)^{\frac{\beta_p}{\beta_p - 1}}\left(\frac{C_p\psi^2}{\alpha}\right)^{\frac{1}{\beta_p - 1}}\right).
\end{equation*}

Taking the maximum of the two inequalities gives an upper bound $|\mathcal{T}_T(\alpha)|$ that holds for all $T \geq 1$ and $\alpha > 0$.
\end{proof}

Next we consider exponential eigendecay.

\begin{lemma}[Elliptical potentials: exponential eigendecay]
Let $b > 0$ and let $x_1, \dots, x_T$ be any sequence of elements in $\mathcal{X}$. Let $k$ be a Mercer kernel that satisfies the exponential eigendecay condition and let $\mathcal{T}_T(\alpha)$ be the set defined in Lemma \ref{lem:ep_info_gain}. If $\beta_e = 1$, then for all $T \geq 1$,
\begin{align*}
|\mathcal{T}_T(\alpha)| &\leq \max(A, B, C, D, E),\\
\text{where} ~~A &= \frac{\alpha C_{e_2}}{C_{e_1}\psi^2},\\
B &= \frac{4e}{(e-1)\ln(1 + b^2)}\ln\left(1 + \frac{4\kappa}{\alpha\ln(1 + b^2)}\right),\\
C &= \frac{4}{\ln(1 + b^2)}\left(\frac{2}{C_{e_2}}\right)^{\frac{1}{\beta_e}}\ln(2)^{1 + \frac{1}{\beta_e}},\\
D &= \frac{4}{\ln(1 + b^2)}\left(\frac{2}{C_{e_2}}\right)^{\frac{1}{\beta_e}}\left(\frac{e(\beta_e+1)}{(e-1)\beta_e}\right)^{\frac{\beta_e+1}{\beta_e}}\ln\left(\left(\frac{8\kappa}{\alpha\ln(1 + b^2)}\right)^{\beta_e/(\beta_e+1)}\left(\frac{2}{C_{e_2}}\right)^{\frac{1}{\beta_e+1}}\frac{\beta_e+1}{\beta_e}\right)^{\frac{\beta_e+1}{\beta_e}},\\
E &= \frac{4}{\ln(1 + b^2)}\left(\frac{2}{C_{e_2}}\right)^{\frac{1}{\beta_e}}\left(\frac{e(\beta_e + 1)}{(e-1)\beta_e}\right)^{\frac{\beta_e + 1}{\beta_e}}\ln\left(\left(\frac{4C_{e_1}\psi^2}{\alpha C_{e_2}\ln(1 + b^2)}\right)^{\frac{\beta_e}{\beta_e + 1}}\left(\frac{2}{C_{e_2}}\right)^{\frac{1}{\beta_e + 1}}\frac{\beta_e + 1}{\beta_e}\right)^{\frac{\beta_e + 1}{\beta_e}}.
\end{align*}

If $\beta_e < 1$, then
\begin{align*}
|\mathcal{T}_T(\alpha)| &\leq \max(A, B, C, D, E, F),\\
\text{where} ~~A &= \frac{\alpha\beta_eC_{e_2}\exp(-\tilde{C}_{\beta})}{2C_{e_1}\psi^2},\\
B &= \frac{4e}{(e-1)\ln(1 + b^2)}\ln\left(1 + \frac{4\kappa}{\alpha\ln(1 + b^2)}\right),\\
C &= \frac{4}{\ln(1 + b^2)}\left(\frac{4}{C_{e_2}}\right)^{\frac{1}{\beta_e}}\ln(2)^{1 + \frac{1}{\beta_e}},\\
D &= \frac{4}{\ln(1 + b^2)}\left(\frac{4}{C_{e_2}}\right)^{\frac{1}{\beta_e}}\left(\frac{e(\beta_e+1)}{(e-1)\beta_e}\right)^{\frac{\beta_e+1}{\beta_e}}\ln\left(\left(\frac{8\kappa}{\alpha\ln(1 + b^2)}\right)^{\beta_e/(\beta_e+1)}\left(\frac{4}{C_{e_2}}\right)^{\frac{1}{\beta_e+1}}\frac{\beta_e+1}{\beta_e}\right)^{\frac{\beta_e+1}{\beta_e}},\\
E &= \frac{4}{\ln(1 + b^2)}\left(\frac{4}{C_{e_2}}\right)^{\frac{1}{\beta_e}}\left(\frac{e(\beta_e + 1)}{(e-1)\beta_e}\right)^{\frac{\beta_e + 1}{\beta_e}}\ln\left(\left(\frac{8C_{e_1}\psi^2}{\alpha \beta_eC_{e_2}\ln(1 + b^2)}\right)^{\frac{\beta_e}{\beta_e+1}}\left(\frac{4}{C_{e_2}}\right)^{\frac{1}{\beta_e+1}}\frac{\beta_e + 1}{\beta_e}\right)^{\frac{\beta_e + 1}{\beta_e}},\\
F &= \frac{2e}{(e-1)\ln(1 + b^2)}\left(\left(\frac{4\tilde{C}_{\beta}}{C_{e_2}}\right)^{\frac{1}{\beta_e}} + 1\right)\ln\left(1 + \frac{2\kappa}{\alpha\ln(1 + b^2)}\left(\left(\frac{4\tilde{C}_{\beta}}{C_{e_2}}\right)^{\frac{1}{\beta_e}} + 1\right)\right).
\end{align*}
\label{lem:ep_exp}
\end{lemma}

Component $A$ of both bounds can be very loose when $\alpha$ is large. For every $t$ and $x_t$, $\|k(\cdot, x_t)\|_{(\bs{V}_{t-1} + \alpha\bs{I}_{\mathcal{H}})^{-1}}^2 \leq \frac{1}{\alpha}k(x_t, x_t) \leq \kappa/\alpha$. Therefore, whenever $\alpha \geq \kappa/b^2$, we have $|\mathcal{T}_T(\alpha)| = 0$. However, since we consider situations in which $\alpha$ is constant or decreasing in $T$, this bound is adequate for our regret analysis.

\begin{proof}
From Lemma \ref{lem:ep_info_gain}, we have $|\mathcal{T}_T(\alpha)| \leq \frac{2}{\ln(1 + b^2)}\gamma_{|\mathcal{T}_T(\alpha)|}(\alpha)$. The bound in Lemma \ref{lem:info_gain} tells us that
\begin{equation}
|\mathcal{T}_T(\alpha)| \leq \frac{2}{\ln(1 + b^2)}\left(\left(\frac{2}{C_{e_2}}(\ln(|\mathcal{T}_T(\alpha)|) + C_{\beta})\right)^{\frac{1}{\beta_e}} + 1\right)\ln(1 + \kappa |\mathcal{T}_T(\alpha)|/\alpha).\label{eqn:exp_ep_pt_1}
\end{equation}

First, we consider the case when $\beta_e = 1$, which means $C_{\beta} = \ln(\frac{C_{e_1}\psi^2}{\alpha C_{e_2}})$, and (\ref{eqn:exp_ep_pt_1}) becomes
\begin{equation}
|\mathcal{T}_T(\alpha)| \leq \frac{2}{\ln(1 + b^2)}\left(\left(\frac{2}{C_{e_2}}\ln\left(\frac{C_{e_1}\psi^2|\mathcal{T}_T(\alpha)|}{\alpha C_{e_2}}\right)\right)^{\frac{1}{\beta_e}} + 1\right)\ln(1 + \kappa |\mathcal{T}_T(\alpha)|/\alpha).\label{eqn:exp_ep_pt_2}
\end{equation}

Suppose that $\ln\left(\frac{C_{e_1}\psi^2|\mathcal{T}_T(\alpha)|}{\alpha C_{e_2}}\right) \leq 0$. This implies the bound
\begin{equation}
|\mathcal{T}_T(\alpha)| \leq \frac{\alpha C_{e_2}}{C_{e_1}\psi^2}.\label{eqn:exp_beta_e1_pt1}
\end{equation}

For the rest of the proof for $\beta_e = 1$, suppose that $\ln\left(\frac{C_{e_1}\psi^2|\mathcal{T}_T(\alpha)|}{\alpha C_{e_2}}\right) \geq 0$. If $\frac{2}{C_{e_2}}\ln\left(\frac{C_{e_1}\psi^2|\mathcal{T}_T(\alpha)|}{\alpha C_{e_2}}\right) \leq 1$, then we obtain
\begin{equation*}
|\mathcal{T}_T(\alpha)| \leq \frac{4}{\ln(1 + b^2)}\ln(1 + \kappa |\mathcal{T}_T(\alpha)|/\alpha).
\end{equation*}

We know from the proof of Lemma \ref{lem:ep_poly} that this inequality implies 
\begin{equation}
|\mathcal{T}_T(\alpha)| \leq \frac{4e}{(e-1)\ln(1 + b^2)}\ln\left(1 + \frac{4\kappa}{\alpha\ln(1 + b^2)}\right).\label{eqn:exp_beta_e1_pt2}
\end{equation}

If $\frac{2}{C_{e_2}}\ln\left(\frac{C_{e_1}\psi^2|\mathcal{T}_T(\alpha)|}{\alpha C_{e_2}}\right) \geq 1$, then from (\ref{eqn:exp_ep_pt_2}), we obtain
\begin{equation}
|\mathcal{T}_T(\alpha)| \leq \frac{4}{\ln(1 + b^2)}\left(\frac{2}{C_{e_2}}\right)^{\frac{1}{\beta_e}}\ln\left(\frac{C_{e_1}\psi^2|\mathcal{T}_T(\alpha)|}{\alpha C_{e_2}}\right)^{\frac{1}{\beta_e}}\ln(1 + \kappa |\mathcal{T}_T(\alpha)|/\alpha).\label{eqn:exp_beta_e1_subcase}
\end{equation}

From here, if $\ln(\frac{C_{e_1}\psi^2|\mathcal{T}_T(\alpha)|}{\alpha C_{e_2}}) \leq \ln(1 + \kappa |\mathcal{T}_T(\alpha)|/\alpha)$, then we must have $\ln(\frac{C_{e_1}\psi^2|\mathcal{T}_T(\alpha)|}{\alpha C_{e_2}})^{1/\beta_e} \leq \ln(1 + \kappa |\mathcal{T}_T(\alpha)|/\alpha)^{1/\beta_e}$, since $1/\beta_e > 0$ and we are in the case where $\ln(\frac{C_{e_1}\psi^2|\mathcal{T}_T(\alpha)|}{\alpha C_{e_2}}) \geq 0$. In this case, we have
\begin{equation}
|\mathcal{T}_T(\alpha)| \leq \frac{4}{\ln(1 + b^2)}\left(\frac{2}{C_{e_2}}\right)^{\frac{1}{\beta_e}}\ln(1 + \kappa |\mathcal{T}_T(\alpha)|/\alpha)^{1 + \frac{1}{\beta_e}}.\label{eqn:exp_ep_repeat_1}
\end{equation}

We must have either $\kappa|\mathcal{T}_T(\alpha)|/\alpha \leq 1$ or $\kappa|\mathcal{T}_T(\alpha)|/\alpha \geq 1$. If $\kappa|\mathcal{T}_T(\alpha)|/\alpha \leq 1$, then we get
\begin{equation}
|\mathcal{T}_T(\alpha)| \leq \frac{4}{\ln(1 + b^2)}\left(\frac{2}{C_{e_2}}\right)^{\frac{1}{\beta_e}}\ln(2)^{1 + \frac{1}{\beta_e}}.\label{eqn:exp_beta_e1_pt3}
\end{equation}

Otherwise, if $\kappa|\mathcal{T}_T(\alpha)|/\alpha \geq 1$, then
\begin{equation*}
|\mathcal{T}_T(\alpha)| \leq \frac{4}{\ln(1 + b^2)}\left(\frac{2}{C_{e_2}}\right)^{\frac{1}{\beta_e}}\ln(2\kappa |\mathcal{T}_T(\alpha)|/\alpha)^{1 + \frac{1}{\beta_e}}.
\end{equation*}

After some algebra, we obtain
\begin{equation*}
|\mathcal{T}_T(\alpha)|^{\beta_e/(\beta_e+1)} \leq \left(\frac{4}{\ln(1 + b^2)}\right)^{\beta_e/(\beta_e+1)}\left(\frac{2}{C_{e_2}}\right)^{\frac{1}{\beta_e+1}}\frac{\beta_e+1}{\beta_e}\ln\left(\left(\frac{2\kappa}{\alpha}\right)^{\beta_e/(\beta_e+1)}|\mathcal{T}_T(\alpha)|^{\beta_e/(\beta_e+1)}\right).
\end{equation*}

Now, using Lemma \ref{lem:gales_3} with $A = (\frac{4}{\ln(1 + b^2)})^{\beta_e/(\beta_e+1)}(\frac{2}{C_{e_2}})^{\frac{1}{\beta_e+1}}\frac{\beta_e+1}{\beta_e}$, $B = (\frac{2\kappa}{\alpha})^{\beta_e/(\beta_e+1)}$, $X = |\mathcal{T}_T(\alpha)|^{\beta_e/(\beta_e+1)}$ and $\eta = 1/e$, we get
\begin{equation*}
|\mathcal{T}_T(\alpha)|^{\beta_e/(\beta_e+1)} \leq \left(\frac{4}{\ln(1 + b^2)}\right)^{\beta_e/(\beta_e+1)}\left(\frac{2}{C_{e_2}}\right)^{\frac{1}{\beta_e+1}}\frac{e(\beta_e+1)}{(e-1)\beta_e}\ln\left(\left(\frac{8\kappa}{\alpha\ln(1 + b^2)}\right)^{\beta_e/(\beta_e+1)}\left(\frac{2}{C_{e_2}}\right)^{\frac{1}{\beta_e+1}}\frac{\beta_e+1}{\beta_e}\right),
\end{equation*}

which means
\begin{equation}
|\mathcal{T}_T(\alpha)| \leq \frac{4}{\ln(1 + b^2)}\left(\frac{2}{C_{e_2}}\right)^{\frac{1}{\beta_e}}\left(\frac{e(\beta_e+1)}{(e-1)\beta_e}\right)^{\frac{\beta_e+1}{\beta_e}}\ln\left(\left(\frac{8\kappa}{\alpha\ln(1 + b^2)}\right)^{\beta_e/(\beta_e+1)}\left(\frac{2}{C_{e_2}}\right)^{\frac{1}{\beta_e+1}}\frac{\beta_e+1}{\beta_e}\right)^{\frac{\beta_e+1}{\beta_e}}.\label{eqn:exp_beta_e1_pt4}
\end{equation}

Now, we return to (\ref{eqn:exp_beta_e1_subcase}). We have considered the case when $\ln(\frac{C_{e_1}\psi^2|\mathcal{T}_T(\alpha)|}{\alpha C_{e_2}}) \leq \ln(1 + \kappa |\mathcal{T}_T(\alpha)|/\alpha)$. If instead, $\ln(\frac{C_{e_1}\psi^2|\mathcal{T}_T(\alpha)|}{\alpha C_{e_2}}) \geq \ln(1 + \kappa |\mathcal{T}_T(\alpha)|/\alpha)$, then from (\ref{eqn:exp_beta_e1_subcase}), we have
\begin{equation*}
|\mathcal{T}_T(\alpha)| \leq \frac{4}{\ln(1 + b^2)}\left(\frac{2}{C_{e_2}}\right)^{\frac{1}{\beta_e}}\ln\left(\frac{C_{e_1}\psi^2|\mathcal{T}_T(\alpha)|}{\alpha C_{e_2}}\right)^{1 + \frac{1}{\beta_e}}.
\end{equation*}

After some algebra, this becomes
\begin{equation*}
|\mathcal{T}_T(\alpha)|^{\frac{\beta_e}{\beta_e + 1}} \leq \left(\frac{4}{\ln(1 + b^2)}\right)^{\frac{\beta_e}{\beta_e + 1}}\left(\frac{2}{C_{e_2}}\right)^{\frac{1}{\beta_e + 1}}\frac{\beta_e + 1}{\beta_e}\ln\left(\left(\frac{C_{e_1}\psi^2}{\alpha C_{e_2}}\right)^{\frac{\beta_e}{\beta_e + 1}}|\mathcal{T}_T(\alpha)|^{\frac{\beta_e}{\beta_e + 1}}\right).
\end{equation*}

Using Lemma \ref{lem:gales_3} with $A = (\frac{4}{\ln(1 + b^2)})^{\beta_e/(\beta_e+1)}(\frac{2}{C_{e_2}})^{\frac{1}{\beta_e+1}}\frac{\beta_e+1}{\beta_e}$, $B = (\frac{C_{e_1}\psi^2}{\alpha C_{e_2}})^{\beta_e/(\beta_e+1)}$, $X = |\mathcal{T}_T(\alpha)|^{\beta_e/(\beta_e+1)}$ and $\eta = 1/e$, we get
\begin{equation*}
|\mathcal{T}_T(\alpha)|^{\frac{\beta_e}{\beta_e + 1}} \leq \left(\frac{4}{\ln(1 + b^2)}\right)^{\frac{\beta_e}{\beta_e + 1}}\left(\frac{2}{C_{e_2}}\right)^{\frac{1}{\beta_e + 1}}\frac{e(\beta_e + 1)}{(e-1)\beta_e}\ln\left(\left(\frac{4C_{e_1}\psi^2}{\alpha C_{e_2}\ln(1 + b^2)}\right)^{\frac{\beta_e}{\beta_e + 1}}\left(\frac{2}{C_{e_2}}\right)^{\frac{1}{\beta_e + 1}}\frac{\beta_e + 1}{\beta_e}\right),
\end{equation*}

which gives
\begin{equation}
|\mathcal{T}_T(\alpha)| \leq \frac{4}{\ln(1 + b^2)}\left(\frac{2}{C_{e_2}}\right)^{\frac{1}{\beta_e}}\left(\frac{e(\beta_e + 1)}{(e-1)\beta_e}\right)^{\frac{\beta_e + 1}{\beta_e}}\ln\left(\left(\frac{4C_{e_1}\psi^2}{\alpha C_{e_2}\ln(1 + b^2)}\right)^{\frac{\beta_e}{\beta_e + 1}}\left(\frac{2}{C_{e_2}}\right)^{\frac{1}{\beta_e + 1}}\frac{\beta_e + 1}{\beta_e}\right)^{\frac{\beta_e + 1}{\beta_e}}.\label{eqn:exp_beta_e1_pt5}
\end{equation}

By taking the maximum of the bounds in (\ref{eqn:exp_beta_e1_pt1}), (\ref{eqn:exp_beta_e1_pt2}), (\ref{eqn:exp_beta_e1_pt3}), (\ref{eqn:exp_beta_e1_pt4}) and (\ref{eqn:exp_beta_e1_pt5}), we obtain the claimed result for $\beta_e = 1$. Now, we consider the case where $\beta_e < 1$, which means $C_{\beta} = \ln(\frac{2C_{e_1}\psi^2}{\alpha \beta_eC_{e_2}}) + \tilde{C}_{\beta}$, where $\tilde{C}_{\beta} = (1/\beta_e - 1)(\ln(\frac{2}{C_{e_2}}(1/\beta_e - 1)) - 1)$, and (\ref{eqn:exp_ep_pt_1}) becomes
\begin{equation}
|\mathcal{T}_T(\alpha)| \leq \frac{2}{\ln(1 + b^2)}\left(\left(\frac{2}{C_{e_2}}\left(\ln\left(\frac{2C_{e_1}\psi^2|\mathcal{T}_T(\alpha)|}{\alpha \beta_eC_{e_2}}\right) + \tilde{C}_{\beta}\right)\right)^{\frac{1}{\beta_e}} + 1\right)\ln(1 + \kappa |\mathcal{T}_T(\alpha)|/\alpha).\label{eqn:exp_ep_pt_3}
\end{equation}

Note that if $\ln(\frac{2C_{e_1}\psi^2|\mathcal{T}_T(\alpha)|}{\alpha \beta_eC_{e_2}}) + \tilde{C}_{\beta} \leq 0$, then we have the bound
\begin{equation}
|\mathcal{T}_T(\alpha)| \leq \frac{\alpha\beta_eC_{e_2}\exp(-\tilde{C}_{\beta})}{2C_{e_1}\psi^2}.\label{eqn:exp_beta_l1_pt1}
\end{equation}

For the rest of the proof, we consider the case when $\ln(\frac{2C_{e_1}\psi^2|\mathcal{T}_T(\alpha)|}{\alpha \beta_eC_{e_2}}) + \tilde{C}_{\beta} \geq 0$, which means $(\ln(\frac{2C_{e_1}\psi^2|\mathcal{T}_T(\alpha)|}{\alpha \beta_eC_{e_2}}) + \tilde{C}_{\beta})^{\frac{1}{\beta_e}}$ is a nonnegative real. If $\tilde{C}_{\beta_e} \leq \ln(\frac{2C_{e_1}\psi^2|\mathcal{T}_T(\alpha)|}{\alpha \beta_eC_{e_2}})$, then since $\ln(\frac{2C_{e_1}\psi^2|\mathcal{T}_T(\alpha)|}{\alpha \beta_eC_{e_2}}) + \tilde{C}_{\beta} \geq 0$, $\ln(\frac{2C_{e_1}\psi^2|\mathcal{T}_T(\alpha)|}{\alpha \beta_eC_{e_2}})$ must be nonnegative and we must have $(\ln(\frac{2C_{e_1}\psi^2|\mathcal{T}_T(\alpha)|}{\alpha \beta_eC_{e_2}}) + \tilde{C}_{\beta})^{1/\beta_e} \leq (2\ln(\frac{2C_{e_1}\psi^2|\mathcal{T}_T(\alpha)|}{\alpha \beta_eC_{e_2}}))^{1/\beta_e}$. From (\ref{eqn:exp_ep_pt_3}), we then have
\begin{equation}
|\mathcal{T}_T(\alpha)| \leq \frac{2}{\ln(1 + b^2)}\left(\left(\frac{4}{C_{e_2}}\ln\left(\frac{2C_{e_1}\psi^2|\mathcal{T}_T(\alpha)|}{\alpha \beta_eC_{e_2}}\right)\right)^{\frac{1}{\beta_e}} + 1\right)\ln(1 + \kappa |\mathcal{T}_T(\alpha)|/\alpha).\label{eqn:exp_ep_pt_4}
\end{equation}

If $\frac{4}{C_{e_2}}\ln(\frac{2C_{e_1}\psi^2|\mathcal{T}_T(\alpha)|}{\alpha \beta_eC_{e_2}}) \leq 1$, then we have
\begin{equation*}
|\mathcal{T}_T(\alpha)| \leq \frac{4}{\ln(1 + b^2)}\ln(1 + \kappa |\mathcal{T}_T(\alpha)|/\alpha).
\end{equation*}

Corollary \ref{cor:gales} with $A = \frac{4}{\ln(1 + b^2)}$, $B = \kappa/\alpha$, $X = |\mathcal{T}_T(\alpha)|$ then tells us that
\begin{equation}
|\mathcal{T}_T(\alpha)| \leq \frac{4e}{(e-1)\ln(1 + b^2)}\ln\left(1 + \frac{4\kappa}{\alpha\ln(1 + b^2)}\right).\label{eqn:exp_beta_l1_pt2}
\end{equation}

Returning to (\ref{eqn:exp_ep_pt_4}), if $\frac{4}{C_{e_2}}\ln(\frac{2C_{e_1}\psi^2|\mathcal{T}_T(\alpha)|}{\alpha \beta_eC_{e_2}}) \geq 1$, then we have
\begin{equation}
|\mathcal{T}_T(\alpha)| \leq \frac{4}{\ln(1 + b^2)}\left(\frac{4}{C_{e_2}}\ln\left(\frac{2C_{e_1}\psi^2|\mathcal{T}_T(\alpha)|}{\alpha \beta_eC_{e_2}}\right)\right)^{\frac{1}{\beta_e}}\ln(1 + \kappa |\mathcal{T}_T(\alpha)|/\alpha).\label{eqn:exp_ep_pt_5}
\end{equation}

From here, if $\ln(\frac{2C_{e_1}\psi^2|\mathcal{T}_T(\alpha)|}{\alpha \beta_eC_{e_2}}) \leq \ln(1 + \kappa |\mathcal{T}_T(\alpha)|/\alpha)$, then we must have $\ln(\frac{2C_{e_1}\psi^2|\mathcal{T}_T(\alpha)|}{\alpha \beta_eC_{e_2}})^{1/\beta_e} \leq \ln(1 + \kappa |\mathcal{T}_T(\alpha)|/\alpha)^{1/\beta_e}$, since $1/\beta_e > 0$ and $\ln(\frac{2C_{e_1}\psi^2|\mathcal{T}_T(\alpha)|}{\alpha \beta_eC_{e_2}})$ is nonnegative. In this case, we have
\begin{equation*}
|\mathcal{T}_T(\alpha)| \leq \frac{4}{\ln(1 + b^2)}\left(\frac{4}{C_{e_2}}\right)^{\frac{1}{\beta_e}}\ln(1 + \kappa |\mathcal{T}_T(\alpha)|/\alpha)^{1 + \frac{1}{\beta_e}}.
\end{equation*}

Note that this inequality is the same as (\ref{eqn:exp_ep_repeat_1}), except a 2 is replaced by a 4. As before, this inequality implies that 
\begin{align}
|\mathcal{T}_T(\alpha)| \leq \max\bigg(&\frac{4}{\ln(1 + b^2)}\left(\frac{4}{C_{e_2}}\right)^{\frac{1}{\beta_e}}\ln(2)^{1 + \frac{1}{\beta_e}},\label{eqn:exp_beta_l1_pt3}\\
&\frac{4}{\ln(1 + b^2)}\left(\frac{4}{C_{e_2}}\right)^{\frac{1}{\beta_e}}\left(\frac{e(\beta_e+1)}{(e-1)\beta_e}\right)^{\frac{\beta_e+1}{\beta_e}}\ln\left(\left(\frac{8\kappa}{\alpha\ln(1 + b^2)}\right)^{\beta_e/(\beta_e+1)}\left(\frac{4}{C_{e_2}}\right)^{\frac{1}{\beta_e+1}}\frac{\beta_e+1}{\beta_e}\right)^{\frac{\beta_e+1}{\beta_e}}\bigg).\nonumber
\end{align}

Returning to (\ref{eqn:exp_ep_pt_5}), if $\ln(\frac{2C_{e_1}\psi^2|\mathcal{T}_T(\alpha)|}{\alpha \beta_eC_{e_2}}) \geq \ln(1 + \kappa |\mathcal{T}_T(\alpha)|/\alpha)$, then
\begin{equation}
|\mathcal{T}_T(\alpha)| \leq \frac{4}{\ln(1 + b^2)}\left(\frac{4}{C_{e_2}}\right)^{\frac{1}{\beta_e}}\ln\left(\frac{2C_{e_1}\psi^2|\mathcal{T}_T(\alpha)|}{\alpha \beta_eC_{e_2}}\right)^{1 + \frac{1}{\beta_e}}.\label{eqn:exp_ep_repeat_2}
\end{equation}

After some algebra, we obtain
\begin{equation*}
|\mathcal{T}_T(\alpha)|^{\frac{\beta_e}{\beta_e+1}} \leq \left(\frac{4}{\ln(1 + b^2)}\right)^{\frac{\beta_e}{\beta_e+1}}\left(\frac{4}{C_{e_2}}\right)^{\frac{1}{\beta_e+1}}\frac{\beta_e + 1}{\beta_e}\ln\left(\left(\frac{2C_{e_1}\psi^2}{\alpha \beta_eC_{e_2}}\right)^{\frac{\beta_e}{\beta_e+1}}|\mathcal{T}_T(\alpha)|^{\frac{\beta_e}{\beta_e+1}}\right).
\end{equation*}

Using Lemma \ref{lem:gales_3} with $A = (\frac{4}{\ln(1 + b^2)})^{\frac{\beta_e}{\beta_e+1}}(\frac{4}{C_{e_2}})^{\frac{1}{\beta_e+1}}\frac{\beta_e + 1}{\beta_e}$, $B = (\frac{2C_{e_1}\psi^2}{\alpha \beta_eC_{e_2}})^{\frac{\beta_e}{\beta_e+1}}$, $X = |\mathcal{T}_T(\alpha)|^{\beta_e/(\beta_e+1)}$ and $\eta = 1/e$, we get
\begin{equation*}
|\mathcal{T}_T(\alpha)|^{\frac{\beta_e}{\beta_e + 1}} \leq \left(\frac{4}{\ln(1 + b^2)}\right)^{\frac{\beta_e}{\beta_e+1}}\left(\frac{4}{C_{e_2}}\right)^{\frac{1}{\beta_e+1}}\frac{e(\beta_e + 1)}{(e-1)\beta_e}\ln\left(\left(\frac{8C_{e_1}\psi^2}{\alpha \beta_eC_{e_2}\ln(1 + b^2)}\right)^{\frac{\beta_e}{\beta_e+1}}\left(\frac{4}{C_{e_2}}\right)^{\frac{1}{\beta_e+1}}\frac{\beta_e + 1}{\beta_e}\right),
\end{equation*}

which gives
\begin{equation}
|\mathcal{T}_T(\alpha)| \leq \frac{4}{\ln(1 + b^2)}\left(\frac{4}{C_{e_2}}\right)^{\frac{1}{\beta_e}}\left(\frac{e(\beta_e + 1)}{(e-1)\beta_e}\right)^{\frac{\beta_e + 1}{\beta_e}}\ln\left(\left(\frac{8C_{e_1}\psi^2}{\alpha \beta_eC_{e_2}\ln(1 + b^2)}\right)^{\frac{\beta_e}{\beta_e+1}}\left(\frac{4}{C_{e_2}}\right)^{\frac{1}{\beta_e+1}}\frac{\beta_e + 1}{\beta_e}\right)^{\frac{\beta_e + 1}{\beta_e}}.\label{eqn:exp_beta_l1_pt4}
\end{equation}

Returning to (\ref{eqn:exp_ep_pt_3}), if $\ln(\frac{2C_{e_1}\psi^2|\mathcal{T}_T(\alpha)|}{\alpha \beta_eC_{e_2}}) \leq \tilde{C}_{\beta}$, then since $\ln(\frac{2C_{e_1}\psi^2|\mathcal{T}_T(\alpha)|}{\alpha \beta_eC_{e_2}}) + \tilde{C}_{\beta} \geq 0$, $\tilde{C}_{\beta}$ must be nonnegative. From (\ref{eqn:exp_ep_pt_3}), we have
\begin{equation*}
|\mathcal{T}_T(\alpha)| \leq \frac{2}{\ln(1 + b^2)}\left(\left(\frac{4\tilde{C}_{\beta}}{C_{e_2}}\right)^{\frac{1}{\beta_e}} + 1\right)\ln(1 + \kappa |\mathcal{T}_T(\alpha)|/\alpha)
\end{equation*}

Using Corollary \ref{cor:gales}, with $A = \frac{2}{\ln(1 + b^2)}((\frac{4\tilde{C}_{\beta}}{C_{e_2}})^{\frac{1}{\beta_e}} + 1)$, $B = \kappa/\alpha$, $X = |\mathcal{T}_T(\alpha)|$, we obtain
\begin{equation}
|\mathcal{T}_T(\alpha)| \leq \frac{2e}{(e-1)\ln(1 + b^2)}\left(\left(\frac{4\tilde{C}_{\beta}}{C_{e_2}}\right)^{\frac{1}{\beta_e}} + 1\right)\ln\left(1 + \frac{2\kappa}{\alpha\ln(1 + b^2)}\left(\left(\frac{4\tilde{C}_{\beta}}{C_{e_2}}\right)^{\frac{1}{\beta_e}} + 1\right)\right).\label{eqn:exp_beta_l1_pt5}
\end{equation}

Finally, by taking the maximum of the bounds in (\ref{eqn:exp_beta_l1_pt1}), (\ref{eqn:exp_beta_l1_pt2}), (\ref{eqn:exp_beta_l1_pt3}), (\ref{eqn:exp_beta_l1_pt4}) and (\ref{eqn:exp_beta_l1_pt5}), we obtain the claimed result for $\beta_e < 1$.
\end{proof}

\subsection{Generic Regret Bound}
\label{app:generic_regret}

Before proving the generic regret bound in Thm. \ref{thm:regret_bound}, we state a useful bound on the regret in a given round $t$. The proof is standard in the analysis of optimistic bandit algorithms (e.g. Prop. 1 in \citealp{russo2013eluder}), but we include it here for completeness.

\begin{lemma}[Per-round regret bound]\label{lem:per_regret}
For any $c > 0$ and any $\delta \in (0, 1]$, let $x_1, x_2, \dots$ be the actions chosen by Kernel CMM-UCB, Kernel DMM-UCB (with any grid $A$ that contains $\alpha$) or Kernel AMM-UCB (with any fixed $\alpha$). With probability at least $1 - \delta$,
\begin{equation*}
\forall t \geq 1, \qquad r_t \leq \frac{2\widetilde{R}_{\alpha,t-1}}{\sqrt{\alpha}}\rho_{\alpha,t-1}(x_t).
\end{equation*}
\end{lemma}

\begin{proof}
We prove the result for Kernel CMM-UCB. The same argument applies for Kernel DMM-UCB and Kernel AMM-UCB. Using Lemma \ref{lem:rkhs_conf_set}, for any $c > 0$, with probability at least $1 - \delta$, we have
\begin{equation*}
\forall t \geq 1, \forall x \in \mathcal{X}, \qquad \mathrm{LCB}_{\mathcal{F}_t}(x) \leq f^{\star}(x) \leq \mathrm{UCB}_{\mathcal{F}_t}(x).
\end{equation*}

Since $x_1, x_2, \dots$ are the actions chosen by Kernel CMM-UCB, we have $x_t = \argmax_{x \in \mathcal{X}_t}\{\mathrm{UCB}_{\mathcal{F}_{t-1}}(x)\}$. Therefore,
\begin{align*}
r_t &= f^{\star}(x_t^{\star}) - f^{\star}(x_t)\\
&\leq \mathrm{UCB}_{\mathcal{F}_{t-1}}(x_t^{\star}) - \mathrm{LCB}_{\mathcal{F}_{t-1}}(x_t)\\
&\leq \mathrm{UCB}_{\mathcal{F}_{t-1}}(x_t) - \mathrm{LCB}_{\mathcal{F}_{t-1}}(x_t)\\
&\leq \frac{2\widetilde{R}_{\alpha,t-1}}{\sqrt{\alpha}}\rho_{\alpha, t-1}(x_t).
\end{align*}

The final inequality uses (\ref{eqn:dual_ucb_bound}) and the analogous lower bound for $\mathrm{LCB}_{\mathcal{F}_{t}}(x)$.
\end{proof}

We can now prove the generic regret bound in Thm. \ref{thm:regret_bound}.

\begin{proof}[Proof of Thm. \ref{thm:regret_bound}]
Using Lemma \ref{lem:per_regret} with $\alpha = \sigma^2/c$, with probability at least $1 - \delta$, we have
\begin{equation*}
\forall t \geq 1, \qquad r_t \leq \frac{2\widetilde{R}_{\sigma^2/c,t-1}}{\sqrt{\sigma^2/c}}\rho_{\sigma^2/c,t-1}(x_t).
\end{equation*}

From our boundedness assumptions, we also have $r_t \leq 2C$. We define the set $\mathcal{T}_T(\alpha)$ as
\begin{equation*}
\mathcal{T}_T(\alpha) = \left\{t \in [T]: \frac{1}{\alpha}\rho_{\alpha,t-1}^2(x_t) \geq 1\right\}.
\end{equation*}

Recall that $\tau_{T}(\alpha) = |\mathcal{T}_T(\alpha)|$. From Lemma \ref{lem:kv_swap2}, we have $\frac{1}{\alpha}\rho_{\alpha, t-1}^2(x_t) = \|k(\cdot, x_t)\|_{(\bs{V}_{t-1} + \alpha\bs{I}_{\mathcal{H}})^{-1}}^2$, which means $\mathcal{T}_T(\alpha)$ is the same as the set defined in Lemma \ref{lem:ep_info_gain}, with $b = 1$. By decomposing the cumulative regret, and then using Cauchy-Schwarz, we obtain
\begin{align*}
\sum_{t=1}^{T}r_t &= \sum_{t \in \mathcal{T}_T}r_t + \sum_{t \notin \mathcal{T}_T}r_t\\
&\leq 2C\tau_T(\tfrac{\sigma^2}{c}) + \sqrt{T{\textstyle\sum_{t \notin \mathcal{T}_T}}r_t^2}\\
&\leq 2C\tau_T(\tfrac{\sigma^2}{c}) + \sqrt{T{\textstyle\sum_{t \notin \mathcal{T}_T}}4\widetilde{R}_{\sigma^2/c,t-1}^2\frac{c}{\sigma^2}\rho_{\sigma^2/c,t-1}^2(x_t)}\\
&\leq 2C\tau_T(\tfrac{\sigma^2}{c}) + \sqrt{4\max_{s \in \{0, 1, \dots, T\}}\widetilde{R}_{\sigma^2/c,s}^2T\sum_{t=1}^{T}\min\left(1, \frac{c}{\sigma^2}\rho_{\sigma^2/c,t-1}^2(x_t)\right)}.
\end{align*}

Using the bound on the squared radius and Lemma \ref{lem:elliptical}, we have
\begin{equation}
\sum_{t=1}^{T}r_t \leq 2C\tau_T\left(\frac{\sigma^2}{c}\right) + \sigma\sqrt{24T\gamma_T\left(\frac{\sigma^2}{c}\right)\left(\gamma_T\left(\frac{\sigma^2}{c}\right) + \frac{B^2}{2c} + \ln\left(\frac{1}{\delta}\right)\right)}\label{eqn:regret_bound}
\end{equation}
\end{proof}

\subsection{Regret Bounds By Eigendecay Rate}
\label{app:eig_regret}

Here, we determine regret bounds for kernels satisfying the polynomial or exponential eigendecay conditions. We focus on the dependence of the regret on $T$ (and $c$, since we choose this based on $T$). First, suppose that the kernel $k$ satisfies the polynomial eigenvalue decay condition.

\begin{proposition}[Regret Bound for Polynomial Decay]
Suppose that $k$ satisfies the polynomial eigendecay condition, and let $c \propto T^{-\frac{1}{\beta_p + 1}}$. For any $T \geq 1$ and any $\delta \in (0, 1]$, with probability at least $1 - \delta$, the cumulative regret of Kernel CMM-UCB, Kernel DMM-UCB (using a grid $A$ containing $\alpha = \sigma^2/c$) and Kernel AMM-UCB (using $\alpha = \sigma^2/c$) satisfies
\begin{equation*}
r_{1:T} = \mathcal{O}\left(T^{\frac{\beta_p + 3}{2\beta_p + 2}}\ln\left(T\right)^{1-\frac{1}{\beta_p}}\right).
\end{equation*}
\label{pro:poly_regret}
\end{proposition}

\begin{proof}
From the bounds in Lemma \ref{lem:info_gain} and Lemma \ref{lem:ep_poly}, we have
\begin{equation*}
\gamma_{T}\left(\tfrac{\sigma^2}{c}\right) = \mathcal{O}\left(\left(cT\right)^{\frac{1}{\beta_p}}\ln\left(cT\right)^{1 - \frac{1}{\beta_p}}\right), ~~\text{and} ~~\tau_{T}\left(\tfrac{\sigma^2}{c}\right) = \mathcal{O}\left(c^{\frac{1}{\beta_p - 1}}\ln\left(c^{\frac{\beta_p}{\beta_p - 1}}\right)\right).
\end{equation*}

If we plug these into the regret bound in Eq. (\ref{eqn:regret_bound}), we get
\begin{equation*}
\sum_{t=1}^{T}r_t = \mathcal{O}\left(c^{\frac{1}{\beta_p - 1}}\ln\left(c^{\frac{\beta_p}{\beta_p - 1}}\right) + T^{\frac{1}{2}+\frac{1}{\beta_p}}c^{\frac{1}{\beta_p}}\ln(cT)^{1-\frac{1}{\beta_p}} + T^{\frac{1}{2} + \frac{1}{2\beta_p}}c^{\frac{1}{2\beta_p}-\frac{1}{2}}\ln(cT)^{\frac{1}{2}-\frac{1}{2\beta_p}}\right).
\end{equation*}

If we set $c \propto T^q$, for some $q$ to be chosen later, this becomes
\begin{equation*}
\sum_{t=1}^{T}r_t = \mathcal{O}\left(T^{\frac{q}{\beta_p - 1}}\ln\left(T^{\frac{q\beta_p}{\beta_p - 1}}\right) + T^{\frac{1}{2}+\frac{1}{\beta_p}+\frac{q}{\beta_p}}\ln\left(T^{1+q}\right)^{1-\frac{1}{\beta_p}} + T^{\frac{1}{2} + \frac{1}{2\beta_p}+\frac{q}{2\beta_p}-\frac{q}{2}}\ln\left(T^{1+q}\right)^{\frac{1}{2}-\frac{1}{2\beta_p}}\right).
\end{equation*}

If we choose $q = -\frac{1}{\beta_p + 1}$, then
\begin{equation*}
\frac{1}{2}+\frac{1}{\beta_p}+\frac{q}{\beta_p} = \frac{1}{2} + \frac{1}{2\beta_p}+\frac{q}{2\beta_p}-\frac{q}{2} = \frac{\beta_p + 3}{2\beta_p + 2}.
\end{equation*}

Therefore, with this choice of $q$, the regret bound is
\begin{equation*}
\sum_{t=1}^{T}r_t = \mathcal{O}\left(T^{\frac{\beta_p + 3}{2\beta_p + 2}}\ln\left(T\right)^{1-\frac{1}{\beta_p}}\right).
\end{equation*}
\end{proof}

This matches the regret bound in Thm. 2 in \citep{whitehouse2023sublinear}, which we believe to be the best known regret bound for (vanilla) KernelUCB/GP-UCB under polynomial eigendecay. It is known that the Mat\'{e}rn kernel, with smoothness parameter $\nu > 1/2$, satisfies the polynomial eigendecay condition with $\beta_p = \frac{2\nu + d}{d}$ \citep{santin2016approximation}. Therefore, if we choose $c \propto T^{-\frac{d}{2\nu + 2d}}$, the the regret bound is
\begin{equation*}
\sum_{t=1}^{T}r_t = \mathcal{O}\left(T^{\frac{\nu + 2d}{2\nu + 2d}}\ln\left(T\right)^{\frac{2\nu}{2\nu + d}}\right).
\end{equation*}

Now, suppose that the kernel $k$ satisfies the exponential eigendecay condition.

\begin{proposition}[Regret Bound for Exponential Eigendecay]
Suppose that $k$ satisfies the exponential eigendecay condition, and let $c \propto 1$. For any $T \geq 1$ and any $\delta \in (0, 1]$, with probability at least $1 - \delta$, the cumulative regret of Kernel CMM-UCB, Kernel DMM-UCB (using a grid $A$ containing $\alpha = \sigma^2/c$) and Kernel AMM-UCB (using $\alpha = \sigma^2/c$) satisfies
\begin{equation*}
r_{1:T} = \mathcal{O}\left(\sqrt{T}\ln\left(T\right)^{1 + \frac{1}{\beta_e}}\right).
\end{equation*}
\label{pro:exp_regret}
\end{proposition}

\begin{proof}
From the bounds in Lemma \ref{lem:info_gain} and Lemma \ref{lem:ep_exp}, we have
\begin{equation*}
\gamma_{T}\left(\tfrac{\sigma^2}{c}\right) = \mathcal{O}\left(\ln\left(cT\right)^{1 + \frac{1}{\beta_e}}\right), ~~\text{and} ~~\tau_{T}\left(\tfrac{\sigma^2}{c}\right) = \mathcal{O}\left(\max\left(\tfrac{1}{c}, \ln\left(c\right)^{1 + \frac{1}{\beta_e}}\right)\right).
\end{equation*}

If we plug these into Eq. (\ref{eqn:regret_bound}), we get
\begin{equation*}
\sum_{t=1}^{T}r_t = \mathcal{O}\left(\max\left(\tfrac{1}{c}, \ln\left(c\right)^{1 + \frac{1}{\beta_e}}\right) + \sqrt{T}\ln\left(cT\right)^{1 + \frac{1}{\beta_e}} + \sqrt{\tfrac{T}{c}\ln\left(cT\right)^{1 + \frac{1}{\beta_e}}}\right).
\end{equation*}

If we set $c \propto 1$, this becomes
\begin{equation*}
\sum_{t=1}^{T}r_t = \mathcal{O}\left(\sqrt{T}\ln\left(T\right)^{1 + \frac{1}{\beta_e}}\right).
\end{equation*}
\end{proof}

It is known that the RBF kernel satisfies the exponential eigendecay condition with $\beta_e = 1/d$ \citep{belkin2018approximation}. This means that, when $c \propto 1$, the regret bound is
\begin{equation*}
\sum_{t=1}^{T}r_t = \mathcal{O}\left(\sqrt{T}\ln\left(T\right)^{d+1}\right).
\end{equation*}

Note that one could also use the regret bound in (\ref{eqn:regret_bound}) to determine how $c$ should depend on $B$. For exponential eigendecay, one can show that the choice $c \propto B^2$ leads to a regret bound which is polylogarithmic in $B$ in the dominant $\sqrt{T}$ term. One can also consider the dependence of the regret on $\sigma$. For exponential eigendecay, when $\sigma \propto 1/\sqrt{T}$, the regret bound becomes polylogarithmic in $T$. Therefore, KernelUCB with our confidence bounds can still enjoy $\mathrm{polylog}(T)$ regret even when the problem is only approximately noiseless.

\subsection{Regret Bounds for the Noiseless Setting}
\label{app:regret_noiseless}

In Sec. \ref{sec:noiseless}, we showed that for all $\alpha > 0$ and $T \geq 1$,
\begin{equation}
r_{1:T} \leq 2C\tau_T(\alpha) + \sqrt{12\alpha B^2T\gamma_T(\alpha)}.\label{eqn:noiseless_regret}
\end{equation}

First, suppose that the kernel $k$ satisfies the polynomial eigendecay condition. In the noiseless setting, we can prove the following regret bound.

\begin{proposition}[Regret Bound for Polynomial Decay in the Noiseless Setting]
Suppose that $k$ satisfies the polynomial eigendecay condition, and let $\alpha \propto T^{\frac{1 - \beta_p^2}{\beta_p^2 + 1}}$. For any $T \geq 1$, the cumulative regret of Kernel CMM-UCB, Kernel DMM-UCB (using a grid $A$ containing this $\alpha$) and Kernel AMM-UCB (using this $\alpha$) satisfies
\begin{equation*}
r_{1:T} = \mathcal{O}\left(T^{\frac{\beta_p + 1}{\beta_p^2 + 1}}\ln\left(T\right)\right).
\end{equation*}
\label{pro:poly_regret_noiseless}
\end{proposition}

This is always sublinear in $T$, since $\beta_p > 1$, and has a growth rate slower than $\sqrt{T}$ whenever $\beta_p > 1 + \sqrt{2}$.

\begin{proof}
We're interested in the case where $\alpha \propto T^{q}$ for some $q < 0$. In this regime, the bounds in Lemma \ref{lem:info_gain} and Lemma \ref{lem:ep_poly} tell us that
\begin{equation*}
\tau_T(\alpha) = \mathcal{O}\left(\alpha^{-\frac{1}{\beta_p - 1}}\ln\left(\alpha^{-\frac{\beta_p}{\beta_p - 1}}\right)\right), \quad \gamma_T(\alpha) = \mathcal{O}\left(\left(\tfrac{T}{\alpha}\right)^{\frac{1}{\beta_p}}\ln\left(\tfrac{T}{\alpha}\right)^{1 - \frac{1}{\beta_p}}\right).
\end{equation*}

This means that Eq. (\ref{eqn:noiseless_regret}) becomes
\begin{equation*}
r_{1:T} = \mathcal{O}\left(\alpha^{-\frac{1}{\beta_p - 1}}\ln\left(\alpha^{-\frac{\beta_p}{\beta_p - 1}}\right) + T^{\frac{1}{2}+\frac{1}{2\beta_p}}\alpha^{\frac{1}{2}-\frac{1}{2\beta_p}}\ln\left(\tfrac{T}{\alpha}\right)^{\frac{1}{2} - \frac{1}{2\beta_p}}\right).
\end{equation*}

If $\alpha \propto T^q$, for some $q < 0$ to be chosen later, then
\begin{equation*}
r_{1:T} = \mathcal{O}\left(T^{-\frac{q}{\beta_p - 1}}\ln\left(T^{-\frac{\beta_pq}{\beta_p - 1}}\right) + T^{\frac{1}{2}+\frac{1}{2\beta_p}+\frac{q}{2}-\frac{q}{2\beta_p}}\ln\left(T^{1-q}\right)^{\frac{1}{2} - \frac{1}{2\beta_p}}\right).
\end{equation*}

We want to choose $q$ such that the maximum of the powers of $T$ is minimised. We notice that, since $\beta_p > 1$, $-\frac{q}{\beta_p - 1}$ is monotonically decreasing in $q$ and $\frac{1}{2}+\frac{1}{2\beta_p}+\frac{q}{2}-\frac{q}{2\beta_p}$ is monotonically increasing in $q$. If we choose $q = \frac{1 - \beta_p^2}{\beta_p^2 + 1}$, then
\begin{equation*}
-\frac{q}{\beta_p - 1} = \frac{1}{2}+\frac{1}{2\beta_p}+\frac{q}{2}-\frac{q}{2\beta_p} = \frac{\beta_p + 1}{\beta_p^2 + 1}.
\end{equation*}

Therefore, with this choice of $q$, we have
\begin{equation*}
r_{1:T} = \mathcal{O}\left(T^{\frac{\beta_p + 1}{\beta_p^2 + 1}}\ln\left(T\right)\right).
\end{equation*}
\end{proof}

Consider the the Mat\'{e}rn kernel, where $\beta_p = \frac{2\nu + d}{d}$. If we choose $\alpha \propto T^{-\frac{2\nu^2 + 2\nu d}{2\nu^2 + 2\nu d + d^2}}$, then the regret bound for the noiseless case is
\begin{equation*}
r_{1:T} = \mathcal{O}\left(T^{\frac{\nu d + d^2}{2\nu^2 + 2\nu d + d^2}}\ln\left(T\right)\right).
\end{equation*}

Whenever $\nu > d/\sqrt{2}$, the growth rate is slower than $\sqrt{T}$. Now, suppose that the kernel satisfies the exponential eigendecay condition. In the noiseless setting, we can prove the following regret bound.

\begin{proposition}[Regret Bound for Exponential Decay in the Noiseless Setting]
Suppose that $k$ satisfies the exponential eigendecay condition, and let $\alpha \propto 1$. For any $T \geq 1$, the cumulative regret of Kernel CMM-UCB, Kernel DMM-UCB (using a grid $A$ containing this $\alpha$) and Kernel AMM-UCB (using this $\alpha$) satisfies
\begin{equation*}
r_{1:T} = \mathcal{O}\left(\ln\left(T\right)^{1 + \frac{1}{\beta_e}}\right).
\end{equation*}
\label{pro:exp_regret_noiseless}
\end{proposition}

\begin{proof}
We're interested in the case where $\alpha \propto T^{q}$ for some $-1 \leq q \leq 0$. In this regime, the bounds in Lemma \ref{lem:info_gain} and Lemma \ref{lem:ep_exp} tell us that
\begin{equation*}
\tau_T(\alpha) = \mathcal{O}\left(\ln\left(\tfrac{1}{\alpha}\right)^{1 + \frac{1}{\beta_e}}\right), \quad \gamma_T(\alpha) = \mathcal{O}\left(\ln\left(\tfrac{T}{\alpha}\right)^{1 + \frac{1}{\beta_e}}\right).
\end{equation*}

This means that Eq. (\ref{eqn:noiseless_regret}) becomes
\begin{equation*}
r_{1:T} = \mathcal{O}\left(\ln\left(\tfrac{1}{\alpha}\right)^{1 + \frac{1}{\beta_e}} + T^{\frac{1}{2}}\alpha^{\frac{1}{2}}\ln\left(\tfrac{T}{\alpha}\right)^{\frac{1}{2} + \frac{1}{2\beta_e}}\right).
\end{equation*}

If we choose $\alpha = 1/T$, then
\begin{equation*}
r_{1:T} = \mathcal{O}\left(\ln\left(T\right)^{1 + \frac{1}{\beta_e}}\right).
\end{equation*}
\end{proof}

For the RBF kernel, where $\beta_e = 1/d$, we get
\begin{equation*}
r_{1:T} = \mathcal{O}\left(\ln\left(T\right)^{d + 1}\right).
\end{equation*}

\section{Additional Practical Details}

In this section, we comment on some practical details to aid with reproducibility.

\subsection{Sample CVXPY Code For Solving (\ref{eqn:ucb_cone_prog})}
\label{app:cvxpy_code}

In Fig. \ref{fig:cvxpy}, we provide some sample CVXPY code for computing the solution of the second order cone program in (\ref{eqn:ucb_cone_prog}). \texttt{\_w} is the optimisation variable $\bs{w} \in \mathbb{R}^{t+1}$; \texttt{k\_t1} is $\bs{k}_{t+1}(x)$; \texttt{k\_tt1} is the $t \times (t+1)$ kernel matrix $\bs{K}_{t,t+1}$; \texttt{y\_t} is the reward vector $\bs{y}_t$; \texttt{R\_t} is the radius $R_t$; \texttt{B} is the norm bound $B$; \texttt{l\_t1} is the right Cholesky factor $\bs{L}_{t+1}$ of the $(t+1) \times (t+1)$ kernel matrix $\bs{K}_{t+1}$ (so \texttt{l\_t1} is an upper triangular matrix).

\begin{figure}[H]
\begin{python}
import cvxpy as cp

_w = cp.Variable(t+1)
obj = k_t1.T @ _w
cons = [cp.norm(k_tt1 @ _w - y_t) <= R_t,
       cp.norm(l_t1 @ _w) <= B]
prob = cp.Problem(cp.Maximize(obj), cons)
ucb = prob.solve()
\end{python}
\caption{Python code for solving (\ref{eqn:ucb_cone_prog}) with CVXPY.}
\label{fig:cvxpy}
\end{figure}

The variable \texttt{ucb} will now be equal to the numerical solution of $\max_{f \in \mathcal{F}_t}\{f(x)\}$. In practice, we find that it is helpful to add a small multiple of the identity matrix to the kernel matrix $\bs{K}_{t+1}$ before computing the Cholesky decomposition. The reason is that, due to numerical error, small positive eigenvalues of $\bs{K}_{t+1}$ can appear to be negative. One could avoid the necessity of computing the Cholesky decomposition of $\bs{K}_{t+1}$ by expressing the second constraint as $\bs{w}^{\top}\bs{K}_{t+1}\bs{w} \leq B^2$ (which is still convex in $\bs{w}$). However, the conic form of (\ref{eqn:ucb_cone_prog}) (with norm constraints) is favourable for the conic solvers used by CVXPY.

\subsection{Recursive Updates}

For a more efficient implementation of our Kernel CMM-UCB, Kernel DMM-UCB and Kernel AMM-UCB algorithms, the inverse $(\bs{K}_t + \alpha\bs{I})^{-1}$, the log determinant $\ln(\det(\frac{1}{\alpha}\bs{K}_t + \bs{I}))$ and the right Cholesky factor $\bs{L}_t$ can be updated using the following recursive formulas. These update rules can be derived using properties of the Schur complement \citep{zhang2006schur}, and have been used before in other kernel bandit algorithms (see e.g. App. F of \citep{chowdhury2017kernelized}).

For the inverse, we have
\begin{align*}
K_{22} &= \frac{1}{k(x_{t+1}, x_{t+1}) + \alpha - \bs{k}_t(x_{t+1})^{\top}(\bs{K}_{t} + \alpha\bs{I})^{-1}\bs{k}_t(x_{t+1})},\\
K_{11} &= (\bs{K}_t + \alpha\bs{I})^{-1} + K_{22}\bs{k}_t(x_{t+1})^{\top}(\bs{K}_{t} + \alpha\bs{I})^{-1}\bs{k}_t(x_{t+1}),\\
K_{12} &= -K_{22}(\bs{K}_{t} + \alpha\bs{I})^{-1}\bs{k}_t(x_{t+1}),\\
K_{21} &= -K_{22}\bs{k}_t(x_{t+1})^{\top}(\bs{K}_{t} + \alpha\bs{I})^{-1},\\
(\bs{K}_{t+1} + \alpha\bs{I})^{-1} &= \left[\begin{array}{c|c} K_{11} & K_{12}\\ \hline K_{21} & K_{22} \end{array}\right].
\end{align*}

For the log determinant, we have
\begin{equation*}
\ln\left(\det\left(\frac{1}{\alpha}\bs{K}_{t+1} + \bs{I}\right)\right) = \ln\left(\det\left(\frac{1}{\alpha}\bs{K}_t + \bs{I}\right)\right) + \ln\left(1 + \frac{k(x_{t+1}, x_{t+1}) - \bs{k}_t(x_{t+1})^{\top}(\bs{K}_{t} + \alpha\bs{I})^{-1}\bs{k}_t(x_{t+1})}{\alpha}\right).
\end{equation*}

For the right Cholesky factor, we have
\begin{equation*}
\bs{L}_{t+1} = \left[\begin{array}{c|c} \bs{L}_t & \bs{k}_{t}(x_{t+1})^{\top}\bs{L}_t^{-1}\\ \hline \bs{0} & \sqrt{k(x_{t+1}, x_{t+1}) - \Vert\bs{k}_{t}(x_{t+1})^{\top}\bs{L}_t^{-1}\Vert_2^2} \end{array}\right].
\end{equation*}

Note that the vector $\bs{k}_{t}(x_{t+1})^{\top}\bs{L}_t^{-1}$ is the transpose of $\bs{v}$, where $\bs{v}$ is the solution of $\bs{L}_t^{\top}\bs{v} = \bs{k}_t(x_{t+1})$. $\bs{v}$ can be computed using numerical methods for solving lower triangular systems (see e.g. \texttt{linalg.triangular\_solve} in the SciPy library), so it is not necessary to invert $\bs{L}_t$.

\section{Additional Experimental Results}

In this section, we present some additional experimental results.

\subsection{Time Per Step}
\label{app:time_per_step}

We run each method in a synthetic kernel bandit problem, as described in Sec.\ \ref{sec:experiments}. We use the Mat\'{e}rn kernel with smoothness $\nu = 5/2$ and length-scale $\ell = 0.5$.

\begin{figure}[H]
\centering
\includegraphics[width=0.8\textwidth]{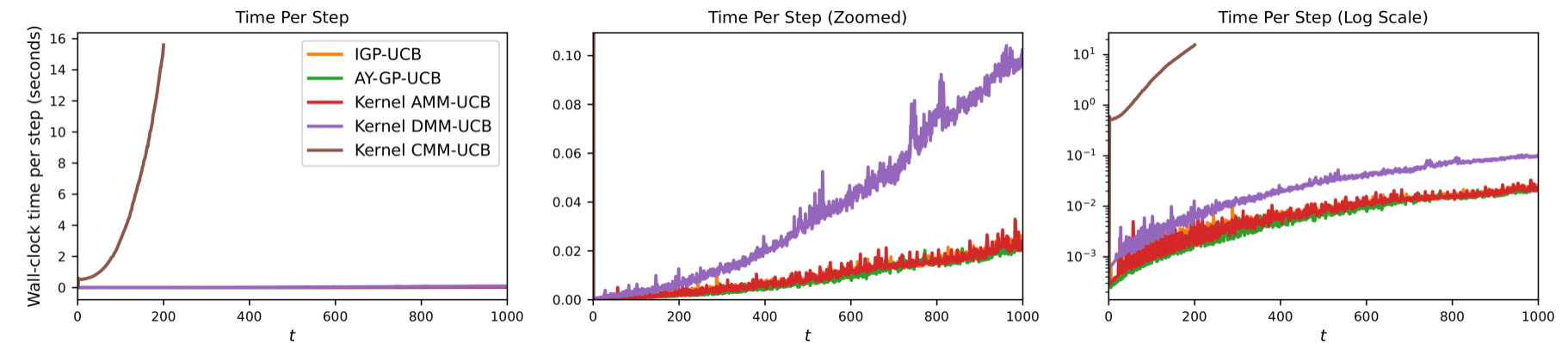}
\caption{Wall-clock time per step for our kernel UCB-style algorithms as well as AY-GP-UCB, IGP-UCB over $T=1000$ rounds with: (left) original $y$-axis; (middle) zoomed-in $y$-axis; (right) log scale $y$-axis. We show the mean over 10 repetitions.}
\label{fig:time_per_step}
\end{figure}

We find that Kernel DMM-UCB and Kernel AMM-UCB have a similar computational cost to AY-GP-UCB and IGP-UCB, whereas Kernel CMM-UCB has a much greater computational cost.

\subsection{2-Dimensional and 4-Dimensional Actions}
\label{app:2d_and_4d}

We display the cumulative regret of each method in the synthetic kernel bandit problems described in Sec.\ \ref{sec:experiments}. Here, the action set $\mathcal{X}_t$ in each round $t$ consists of 100 random vectors in the 2-dimensional hypercube $[0,1]^2$ (Fig. \ref{fig:regret_2d}) or the 4-dimensional hypercube $[0,1]^4$ (Fig. \ref{fig:regret_4d}).

\begin{figure}[H]
\centering
\includegraphics[width=0.8\textwidth]{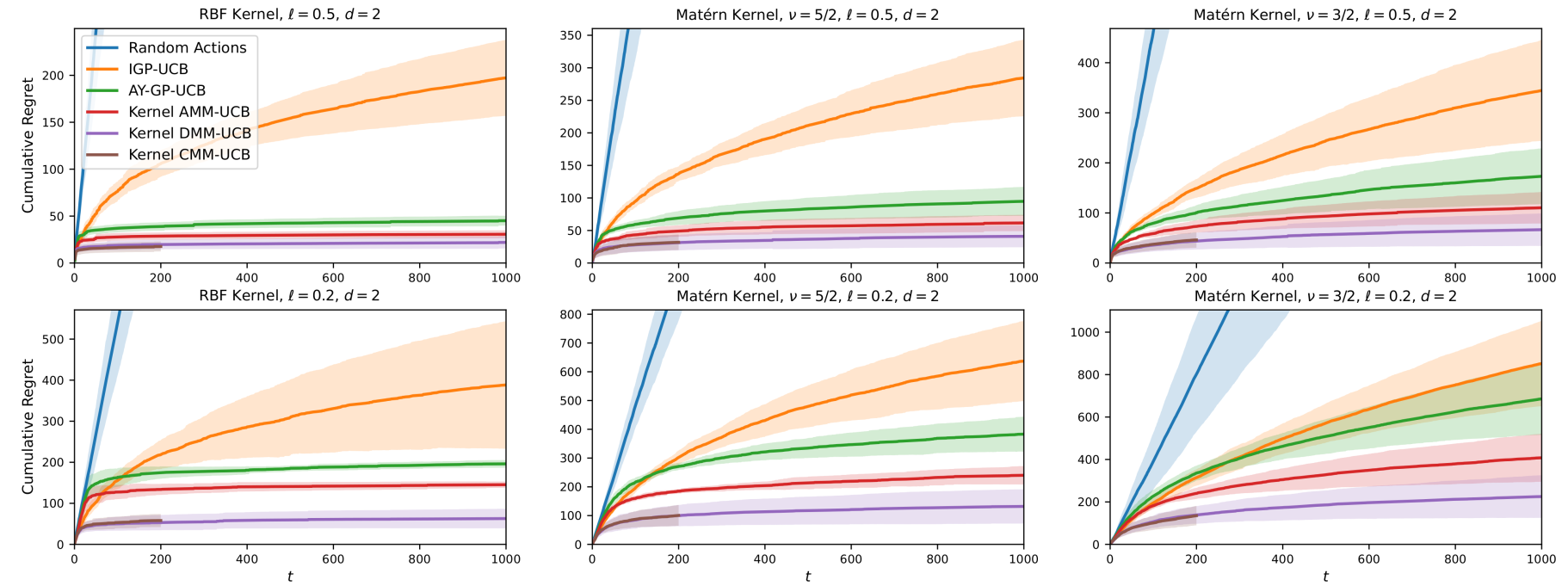}
\caption{Cumulative regret of our KernelUCB-style algorithms as well as AY-GP-UCB, IGP-UCB and a random baseline over $T=1000$ rounds for $d=2$ and various kernels (columns) and length scales (rows). We show the mean $\pm$ standard deviation over 10 repetitions.}
\label{fig:regret_2d}
\end{figure}

In Fig. \ref{fig:regret_2d}, we observe that our Kernel CMM-UCB and Kernel DMM-UCB algorithms achieve the lowest cumulative regret, followed by Kernel AMM-UCB and the AY-GP-UCB and IGP-UCB. All methods performed noticeably better with 2-dimensional action sets than with 3-dimensional action sets (see Fig. \ref{fig:regret}).

\begin{figure}[H]
\centering
\includegraphics[width=0.8\textwidth]{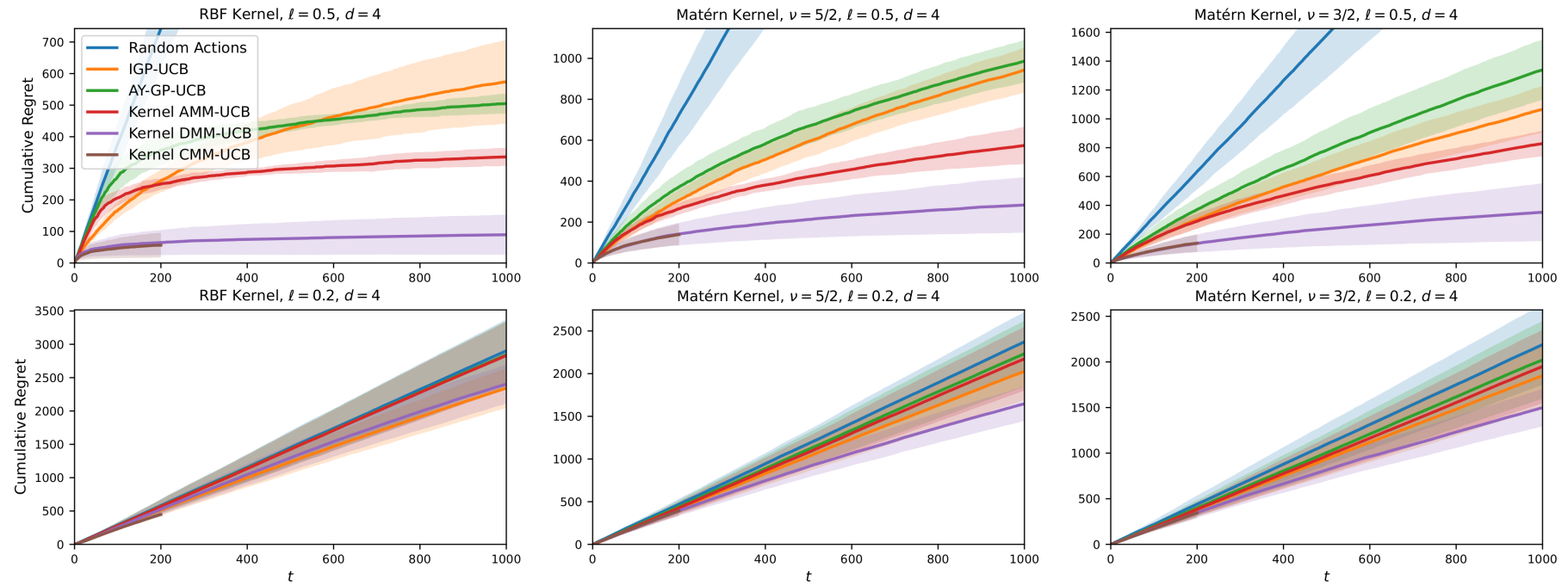}
\caption{Cumulative regret of our KernelUCB-style algorithms as well as AY-GP-UCB, IGP-UCB and a random baseline over $T=1000$ rounds for $d=4$ and various kernels (columns) and length scales (rows). We show the mean $\pm$ standard deviation over 10 repetitions.}
\label{fig:regret_4d}
\end{figure}

In Fig. \ref{fig:regret_4d}, we observe a similar drop in the performance of all methods when moving up to 4-dimensional action sets. In the top row of Fig. \ref{fig:regret_4d}, in which the kernel length-scale is $\ell = 0.5$, we observe that each method still achieves much lower cumulative regret than the random baseline, and Kernel CMM-UCB and Kernel DMM-UCB algorithms still achieve the lower cumulative regret. However, in bottom row of Fig. \ref{fig:regret_4d}, we see that no method performs much better than the random baseline.

\subsection{Loose Upper Bounds on $B$ and $\sigma$}

\begin{figure}[H]
\centering
\includegraphics[width=0.8\textwidth]{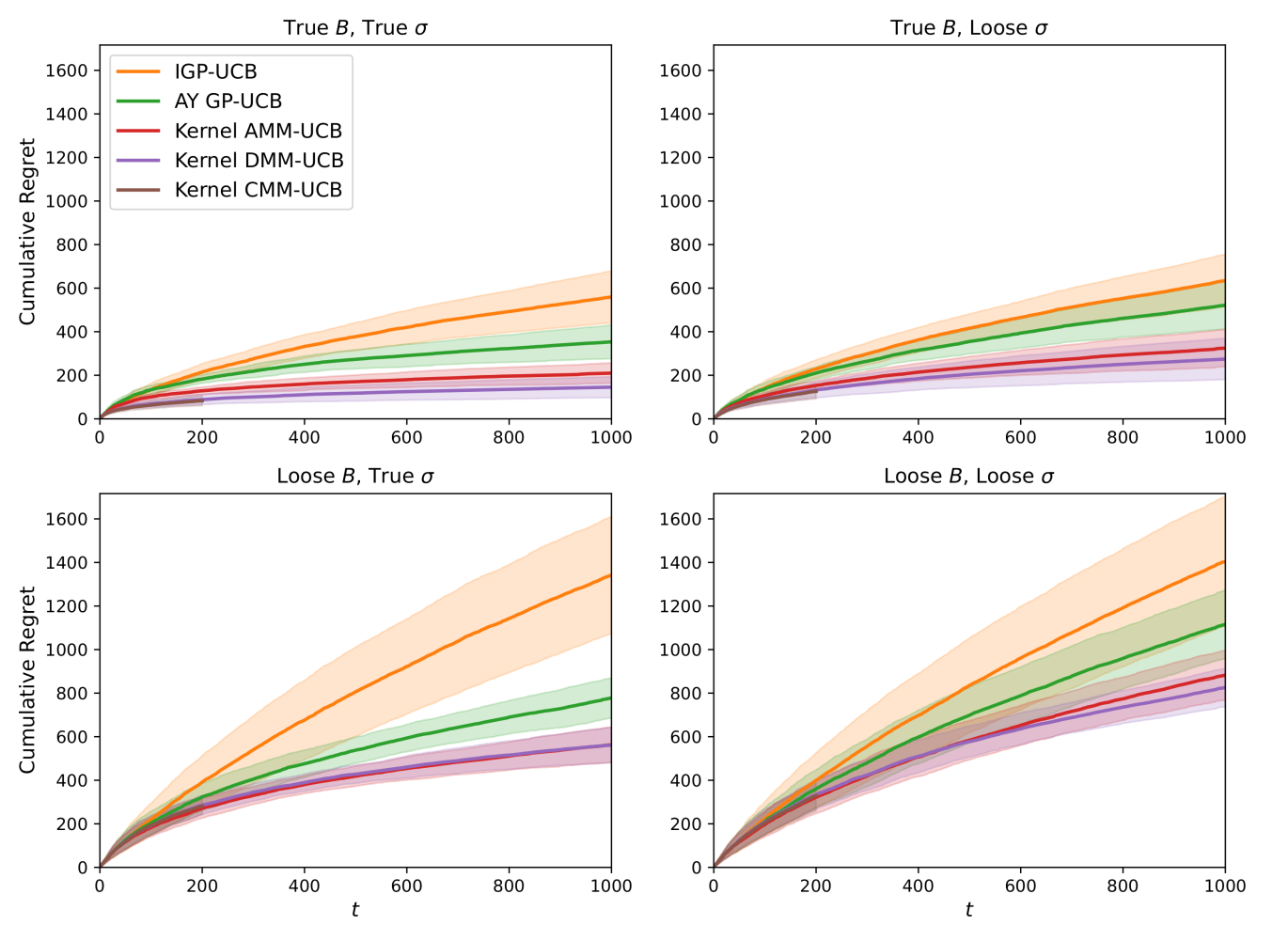}
\caption{Cumulative regret of our kernel UCB-style algorithms as well as AY-GP-UCB, IGP-UCB and a random baseline after $T=1000$ rounds for various kernels (columns) and length scales (rows). The action sets are subsets of $[0,1]^4$. We show the mean $\pm$ standard deviation over 10 repetitions.}
\label{fig:loose_bsig}
\end{figure}

To run each kernel bandit algorithm, we only need to know upper bounds on the sub-Gaussian parameter $\sigma$ and the bound $B$ on the RKHS norm of the reward function $f^{\star}$. We now investigate what happens when the known bounds on one or both of these quantities are loose.

We run each algorithm in the same synthetic kernel bandit problem as described in Sec.\ \ref{sec:experiments}, where the kernel is a Mat\'{e}rn kernel with smoothness $\nu = 5/2$ and length-scale $\ell = 0.5$, and the actions are 3-dimensional vectors in $[0,1]^3$. As before, the noise variable are $\sigma$-sub-Gaussian with $\sigma = 0.1$ and $\norm{f^{\star}} = 10$. We run each algorithm with all four pairs of values of $(\sigma, B)$, where $\sigma \in \{0.1, 0.2\}$ and $B \in \{10, 20\}$.

Fig. \ref{fig:loose_bsig} shows the cumulative regret of each method and with each combination of correct or loose $\sigma$ and $B$. We observe that using a loose upper bound on $\sigma$ has a relatively small effect on each method, whereas a loose bound on $B$ can cause the cumulative regret of each method to grow considerably.
\end{document}